\documentclass{article}
\usepackage[margin=0.97in]{geometry}
\usepackage{xspace}
\usepackage{amsmath}
\usepackage{graphicx}
\usepackage{framed}
\usepackage{bbm}
\usepackage{algorithm}
\usepackage{amsthm}
\usepackage[algo2e, vlined, ruled, linesnumbered]{algorithm2e}
\usepackage{algorithmic}
\usepackage{natbib}
\usepackage{mathrsfs}
\usepackage{amssymb}
\usepackage{bm}
\usepackage[colorlinks=true, linkcolor=blue, citecolor=blue, breaklinks=true]{hyperref}
\usepackage{makecell}
\usepackage{tabulary}
\usepackage{xcolor}
\usepackage{prettyref}
\usepackage{mathtools}
\usepackage{amsmath}
\usepackage{multirow}
\usepackage{nicefrac}
\usepackage{amssymb}
\usepackage{pifont}
\definecolor{Green}{rgb}{0.13, 0.65, 0.3}
\usepackage{colortbl}
\usepackage[]{color-edits}
\addauthor{HL}{red}
\newcommand{\fillcell}{\cellcolor{blue!25}}
\newcommand{\thline}{\Xhline{\arrayrulewidth}}
\DeclareMathOperator*{\argmin}{argmin} 
\DeclareMathOperator*{\argmax}{argmax} 

\newcommand{\const}{c}

\newcommand{\barpi}{\overline{\pi}}

\newcommand{\hattheta}{\widehat{\theta}}
\newcommand{\Reg}{\text{\rm Reg}}
\newcommand{\hatReg}{\widehat{\text{\rm Reg}}}
\newcommand{\hatphi}{\widehat{\phi}}
\newcommand{\Regret}{\textsc{D-Reg}\xspace}

\newcommand{\RegL}{\Reg_L^\star\xspace}
\newcommand{\RegDev}{\Reg_\dev^\star\xspace}

\newcommand{\masterucrl}{{\small\textsf{\textup{MASTER-UCRL}}}\xspace}

\newcommand{\avgreg}{\rho}

\newcommand{\calV}{\mathcal{V}}
\newcommand{\calA}{\mathcal{A}}

\newcommand{\calE}{\mathcal{E}}

\newcommand{\calS}{\mathcal{S}}

\newcommand{\calI}{\mathcal{I}}

\newcommand{\calJ}{\mathcal{J}}
\newcommand{\calP}{\mathcal{P}}
\newcommand{\calD}{\mathcal{D}}

\newcommand{\conf}{\textit{conf}}
\newcommand{\nmax}{\widehat{n}\xspace}
\newcommand{\tildepi}{\widetilde{\pi}}
\newcommand{\hatp}{\widehat{p}}

\newcommand{\tilder}{\widetilde{r}}
\newcommand{\tildep}{\widetilde{p}}
\newcommand{\tildeJ}{\widetilde{J}}

\newcommand{\tildeh}{\widetilde{h}}

\newcommand{\barr}{\overline{r}}
\newcommand{\barp}{\overline{p}}
\newcommand{\spn}{\text{sp}}

\newcommand{\disc}{\textit{disc}}
\newcommand{\calR}{\mathcal{R}}
\newcommand{\hatcalR}{\widehat{\calR}}

\newcommand{\myComment}[1]{\null\hfill\scalebox{0.9}{\text{\color{black}$\triangleleft$ \textsf{#1}}}}
\newcommand{\dupp}{\overline{D}}

\newcommand{\otil}{\widetilde{\order}}

\newcommand{\acw}{{\small\textsf{\textup{UCRL-ACW}}}\xspace}
\newcommand{\ucrlcw}{{\small\textsf{\textup{UCRL-CW}}}\xspace}
\newcommand{\mucrl}{{\small\textsf{\textup{MUCRL}}}\xspace}

\newcommand{\E}{\mathbb{E}}
\newcommand{\order}{\mathcal{O}}

\newcommand{\hatr}{\widehat{r}}

\newcommand{\norm}[1]{\left\|#1\right\|}

\newcommand{\dev}{\Delta}
\newcommand{\hatdev}{\widehat{\Delta}}

\newcommand{\calX}{\mathcal{X}}

\newcommand{\one}{\mathbbm{1}}
\newcommand{\inner}[1]{\langle#1\rangle}

\newcommand{\tildef}{\widetilde{f}}
\newcommand{\tildeg}{\widetilde{g}}

\newcommand{\alg}{{\small\textsf{\textup{ALG}}}\xspace}
\newcommand{\inst}{\textit{alg}\xspace} 

\newcommand{\multialg}{{\small\textsf{\textup{MALG}}}\xspace}
\newcommand{\regbound}{C}
\newcommand{\multireg}{\widehat{C}}
\newcommand{\multiavg}{\widehat{\rho}}
\newcommand{\gap}{\alpha}
\newcommand{\multigap}{\widehat{\alpha}}
\newcommand{\masteralg}{{\small\textsf{\textup{MASTER}}}\xspace}
\newcommand{\term}{\textbf{term}}
\newcommand{\ucrl}{{\small\textsf{\textup{UCRL}}}\xspace}

\newtheorem{theorem}{Theorem}
\newtheorem{assumption}{Assumption}
\newcommand{\testone}{\textbf{Test~1}\xspace}
\newcommand{\testtwo}{\textbf{Test~2}\xspace}
\newcommand{\highlight}[1]{{\color{blue}#1}}
\newtheorem{lemma}[theorem]{Lemma}

\newtheorem{definition}[theorem]{Definition}

\usepackage{tikz}
\newcommand*\circled[1]{\tikz[baseline=(char.base)]{
    \node[shape=circle,draw,inner sep=2pt] (char) {#1};}}

\newenvironment{exampl}[1][htb]
  {
   
   \begin{algorithm2e}[#1]%
   }{\end{algorithm2e}}

\newenvironment{procedures}[1][htb]
  {
   \begin{algorithm2e}[#1]%
   }{\end{algorithm2e}}

\newcommand{\nonl}{\renewcommand{\nl}{\let\nl}}

\usepackage{prettyref}
\newcommand{\pref}[1]{\prettyref{#1}}

\newcommand{\savehyperref}[2]{\texorpdfstring{\hyperref[#1]{#2}}{#2}}
\newrefformat{eq}{\savehyperref{#1}{Eq.~\textup{(\ref*{#1})}}}
\newrefformat{eqn}{\savehyperref{#1}{Equation~\ref*{#1}}}
\newrefformat{lem}{\savehyperref{#1}{Lemma~\ref*{#1}}}
\newrefformat{lemma}{\savehyperref{#1}{Lemma~\ref*{#1}}}
\newrefformat{def}{\savehyperref{#1}{Definition~\ref*{#1}}}
\newrefformat{line}{\savehyperref{#1}{Line~\ref*{#1}}}
\newrefformat{thm}{\savehyperref{#1}{Theorem~\ref*{#1}}}
\newrefformat{corr}{\savehyperref{#1}{Corollary~\ref*{#1}}}
\newrefformat{cor}{\savehyperref{#1}{Corollary~\ref*{#1}}}
\newrefformat{sec}{\savehyperref{#1}{Section~\ref*{#1}}}
\newrefformat{subsec}{\savehyperref{#1}{Section~\ref*{#1}}}
\newrefformat{app}{\savehyperref{#1}{Appendix~\ref*{#1}}}
\newrefformat{assum}{\savehyperref{#1}{Assumption~\ref*{#1}}}
\newrefformat{ex}{\savehyperref{#1}{Algorithm~\ref*{#1}}}
\newrefformat{fig}{\savehyperref{#1}{Figure~\ref*{#1}}}
\newrefformat{alg}{\savehyperref{#1}{Algorithm~\ref*{#1}}}
\newrefformat{rem}{\savehyperref{#1}{Remark~\ref*{#1}}}
\newrefformat{conj}{\savehyperref{#1}{Conjecture~\ref*{#1}}}
\newrefformat{prop}{\savehyperref{#1}{Proposition~\ref*{#1}}}
\newrefformat{proto}{\savehyperref{#1}{Protocol~\ref*{#1}}}
\newrefformat{prob}{\savehyperref{#1}{Problem~\ref*{#1}}}
\newrefformat{claim}{\savehyperref{#1}{Claim~\ref*{#1}}}
\newrefformat{que}{\savehyperref{#1}{Question~\ref*{#1}}}
\newrefformat{op}{\savehyperref{#1}{Open Problem~\ref*{#1}}}
\newrefformat{fn}{\savehyperref{#1}{Footnote~\ref*{#1}}}
\newrefformat{tab}{\savehyperref{#1}{Table~\ref*{#1}}}
\newrefformat{fig}{\savehyperref{#1}{Figure~\ref*{#1}}}
\newrefformat{proc}{\savehyperref{#1}{Procedure~\ref*{#1}}}

\usepackage{caption}

\DeclareCaptionFormat{myformat}{#3}
\title{Non-stationary Reinforcement Learning without Prior Knowledge: \\
An Optimal Black-box Approach}  
\usepackage{times} 

\begin{document}
\DontPrintSemicolon 
\allowdisplaybreaks 

\author{%
	Chen-Yu Wei\\
	University of Southern California \\
	\texttt{chenyu.wei@usc.edu}\\
	\and
	Haipeng Luo\\
	University of Southern California \\
	\texttt{haipengl@usc.edu}
}
\date{}
\maketitle
\sloppy
\begin{abstract}%
We propose a black-box reduction that turns a certain reinforcement learning algorithm with optimal regret in a (near-)stationary environment into another algorithm with optimal dynamic regret in a non-stationary environment, importantly without any prior knowledge on the degree of non-stationarity.
By plugging different algorithms into our black-box,
we provide a list of examples showing that our approach not only recovers recent results for (contextual) multi-armed bandits achieved by very specialized algorithms, but also significantly improves the state of the art for (generalized) linear bandits, episodic MDPs, and infinite-horizon MDPs in various ways.
Specifically, in most cases our algorithm achieves the optimal dynamic regret $\otil(\min\{\sqrt{LT}, \dev^{\nicefrac{1}{3}}T^{\nicefrac{2}{3}}\})$ where $T$ is the number of rounds and $L$ and $\dev$ are the number and amount of changes of the world respectively, while previous works only obtain suboptimal bounds and/or require the knowledge of $L$ and $\dev$.
\end{abstract}




\section{Introduction}\label{sec: intro}
Most existing works on reinforcement learning consider a stationary environment and aim to find or be comparable to an optimal policy (known as having low \emph{static regret}).
In many applications, however, the environment is far from being stationary.
In these cases, it is much more meaningful to minimize \emph{dynamic regret}, the gap between the total reward of the optimal {\it sequence} of policies and that of the learner.
Indeed, there is a surge of studies on this topic recently~\citep{jaksch2010near, gajane2018sliding, li2019online, ortner2020variational, cheung2020reinforcement, fei2020dynamic, domingues2020kernel, mao2020nearoptimal, zhou2020nonstationary, touati2020efficient}.

One common issue of all these works, however, is that their algorithms crucially rely on having some prior knowledge on the \emph{degree of non-stationarity} of the world, such as how much or how many times the distribution changes, which is often unavailable in practice. \citet{cheung2020reinforcement} develop a Bandit-over-Reinforcement-Learning (BoRL) framework to relax this assumption, but it introduces extra overhead and leads to suboptimal regret.  Indeed, as discussed in their work, there are multiple aspects (which they call \emph{endogeneity}, \emph{exogeneity}, \emph{uncertainty}, and \emph{bandit feedback}) combined in non-stationary reinforcement learning that make the problem highly challenging.

For bandit problems, the special case of reinforcement learning, the works of~\citet{auer2019adaptively} and~\citet{chen2019new} are the first to achieve near-optimal dynamic regret without any prior knowledge on the degree of non-stationarity. 
The same technique has later been adopted by \cite{chen2020combinatorial} for the case of combinatorial semi-bandits. 
Their algorithms maintain a distribution over arms (or policies/super-arms in the contextual/combinatorial case~\citep{chen2019new, chen2020combinatorial}) with properly controlled variance for all reward estimators. This approach is generally incompatible with standard reinforcement learning algorithms, which are usually built upon the \emph{optimism in the face of uncertainty} principle and do not maintain a distribution over policies (see also \citep{lykouris2019corruption, wang2020long} for related discussions). Another drawback is that their algorithms are very specialized to their problems, and it is highly unclear whether the ideas can be extended to other problems.  

In this work, we address all these issues and make significant progress in this direction. 
Specifically, we propose a {\it general} approach that is applicable to various reinforcement learning settings (including bandits, episodic MDPs, infinite-horizon MDPs, etc.) and achieves {\it optimal dynamic regret without any prior knowledge} on the degree of non-stationarity.
Our approach, called \masteralg, is a {\it black-box reduction} that turns any algorithm with optimal performance in a (near-)stationary environment and additionally some mild requirements into another algorithm with optimal dynamic regret in a non-stationary environment,
again, without the need of any prior knowledge.
For example, all existing UCB-based algorithms satisfy the conditions of our reduction and are readily to be plugged into our black-box.

\renewcommand{\arraystretch}{1.15}
\begin{table}[t]
    \centering
    \caption{A summary of our results and comparisons with the state-of-the-art. 
    Our algorithms are named in the form of ``\masteralg + $X$'' where $X$ is the base algorithm used in our reduction.
    Here, $\RegL = \sqrt{LT}$ and $\RegDev = \dev^{\nicefrac{1}{3}}T^{\nicefrac{2}{3}} + \sqrt{T}$, where $T$ is the number of rounds and $L$ and $\dev$ are the number and amount of changes of the world respectively (dependence on other parameters is omitted).
    $D_{\max}$ is the maximum diameters of the MDPs. 
    }
    \label{tab:compare}
     \vspace*{5pt}
    \begin{tabular}{|c|c|c|c|}
        \thline
        Setting  & Algorithm & Regret in $\otil(\cdot)$  &  \makecell{Required \\ knowledge}  \\
       \thline
        \multirow{2}{*}{Multi-armed bandits} & \makecell{\citep{auer2019adaptively}}  & $\RegL$  &    \\
        \cline{2-4}
         & \fillcell\text{\masteralg + UCB1}  & $\min\{\RegL, \RegDev\}$  &     \\
        \thline
        \multirow{3}{*}{Contextual bandits} &  \makecell{\citep{chen2019new}}   &   \multirow{3}{*}{$\min\{\RegL, \RegDev\}$}   &   \\
        \cline{2-2}
         & \fillcell\text{\masteralg + ILTCB}   &   &     \\
         \cline{2-2}  
        & \fillcell\text{\masteralg + FALCON}  &  &    \\
        \thline
        \multirow{3}{*}{Linear bandits}  & \makecell{\citep{cheung2018hedging}}   &  $\RegDev$  &  $\dev$   \\  
        \cline{2-4}
          &  \makecell{\citep{cheung2018hedging}}   &  $\RegDev  + T^{\nicefrac{3}{4}}$  &      \\
        \cline{2-4}
         & \fillcell \text{\masteralg + OFUL}  &     $\min\{\RegL, \RegDev\}$   &   \\
        \thline
        \multirow{3}{*}{Generalized linear bandits}  & \makecell{\citep{russac2020algorithms}}   &  $L^{\nicefrac{1}{3}}T^{\nicefrac{2}{3}}$  &  $L$   \\  
        \cline{2-4}
          &  \makecell{\citep{faury2021regret}}   &  $\dev^{\nicefrac{1}{5}}T^{\nicefrac{4}{5}}  + T^{\nicefrac{3}{4}}$  &      \\
        \cline{2-4}
         & \fillcell \text{\masteralg + GLM-UCB}  &     $\min\{\RegL, \RegDev\}$   &   \\
        \thline
        
        \multirow{2}{*}{\makecell{Episodic MDPs\\ (tabular case)}}  &  \makecell{\citep{mao2020nearoptimal}}  &  $\RegDev$   &  $\dev$    \\
        \cline{2-4}
        & \fillcell\text{\masteralg + Q-UCB}  & $\min\{\RegL, \RegDev\}$ &  \\
        \thline
        \multirow{2}{*}{\makecell{Episodic MDPs\\ (linear case)}}  & \makecell{\citep{touati2020efficient}} &  $\dev^{\nicefrac{1}{4}}T^{\nicefrac{3}{4}} + \sqrt{T}$   &  $\dev$   \\
        \cline{2-4}
        & \fillcell\text{\masteralg + LSVI-UCB}  & $\min\{\RegL, \RegDev\}$ &  \\
        \thline
        \multirow{6}{*}{\makecell{Infinite-horizon \\ communicating MDPs \\ (tabular case)}}  &  \makecell{ \citep{gajane2018sliding}}   &  $L^{\nicefrac{1}{3}}T^{\nicefrac{2}{3}}$    &  $L$   \\
        \cline{2-4}
        & \makecell{\citep{cheung2020reinforcement}}  & $\dev^{\nicefrac{1}{4}}T^{\nicefrac{3}{4}} + \sqrt{T}$   &  $\dev$   \\
        \cline{2-4}
        & \makecell{\citep{cheung2020reinforcement}}  & $\dev^{\nicefrac{1}{4}}T^{\nicefrac{3}{4}} + T^{\nicefrac{3}{4}}$   &     \\
        \cline{2-4}
        &\fillcell \multirow{1}{*}{\text{\masteralg + UCRL}} &  \multirow{1}{*}{ $\RegL$ or $\RegDev$} & $L$ or $\dev$   \\
        \cline{2-4}
        & \fillcell\text{\masteralg + UCRL} & $\min\{\RegL, \RegDev\}$ &  \makecell{$D_{\max}$} \\
        \cline{2-4}
        &\fillcell\text{\masteralg + UCRL + BoRL} & \makecell{$\min\{\RegL, \RegDev\} +T^{\nicefrac{3}{4}}$} &  \\
        \thline
    \end{tabular}
\end{table}

\paragraph{Applications and comparisons}
To showcase the versatility of our approach, we provide a list of examples by considering different settings and applying our reduction with different base algorithms.
These examples, summarized in \pref{tab:compare}, recover the results of \cite{auer2019adaptively} and \cite{chen2019new} for (contextual) multi-armed bandits, and more importantly, improve the best known results for (generalized) linear bandits, episodic MDPs, and infinite-horizon MDPs in various ways.
More specifically, let $L$ and $\dev$ be the number and amount of changes of the environment respectively (see \pref{sec:setting} for formal definition).
For all settings except infinite-horizon MDPs, ignoring other parameters, our algorithms achieve dynamic regret $\min\{\RegL, \RegDev\}$ without knowing $L$ and $\dev$, where $\RegL = \sqrt{LT}$, $\RegDev = \dev^{\nicefrac{1}{3}}T^{\nicefrac{2}{3}} + \sqrt{T}$, and $T$ is the number of rounds.
These bounds are known to be optimal {\it even when $L$ and $\dev$ are known},
and they improve over~\citep{cheung2018hedging, cheung2019learning, russac2019weighted, kim2019randomized, pmlr-v108-zhao20a, zhao2021nonstationary} for linear bandits, \citep{russac2020algorithms, faury2021regret} for generalized linear bandits, \citep{mao2020nearoptimal} for episodic tabular MDPs,
and \citep{touati2020efficient, zhou2020nonstationary} for episodic linear MDPs.
For infinite-horizon MDPs, we achieve the same optimal regret when the maximum diameter of the MDPs is known, or when $L$ and $\dev$ are known, improving over the best existing results by~\citep{gajane2018sliding} and~\citep{cheung2020reinforcement}.
When none of them is known, we can still adopt the BoRL technique~\citep{cheung2020reinforcement} with the price of paying extra $T^{3/4}$ regret,
which is suboptimal but still outperforms best known results. 
  

In particular, we emphasize that achieving dynamic regret $\RegL$ beyond (contextual) multi-armed bandits is one notable breakthrough we make.
Indeed, \emph{even when $L$ is known}, previous approaches based on restarting after a fixed period, a sliding window with a fixed size, or discounting with a fixed discount factor, all lead to a suboptimal bound of $\otil(L^{\nicefrac{1}{3}}T^{\nicefrac{2}{3}})$ at best~\citep{gajane2018sliding}.
Since this bound is subsumed by $\RegDev$, related discussions are also often omitted in previous works.

For non-stationary linear bandits, although several existing works \citep{russac2019weighted, kim2019randomized, pmlr-v108-zhao20a}
claim that their algorithms achieve the bound $\RegDev$ (when $\Delta$ is known), there is in fact a technical flaw in all of them, as explained and corrected recently in~\citep{zhao2021nonstationary, touati2020efficient}. After the correction, their bounds all deteriorate  to $\dev^{\nicefrac{1}{4}}T^{\nicefrac{3}{4}}+\sqrt{T}$, which is no longer near-optimal. Recently, \cite{cheung2018hedging} sidesteps this difficulty by leveraging adversarial linear bandit algorithms, and achieves the tight bound of $\RegDev$ when $\dev$ is known. 
On the other hand, our approach is based on stochastic linear bandit algorithms; however, it not only sidesteps the difficulty met in previous works, but also avoids the requirement of knowing $\dev$. When dealing with other linear-structured problems including generalized linear bandits and linear MDPs, our bounds $\RegDev$ is new \emph{even when $\Delta$ is known}. Previous results \citep{russac2020algorithms, faury2021regret, touati2020efficient, zhou2020nonstationary} cannot achieve the optimal bound due to the same technical difficulty mentioned above.   

\paragraph{High-level ideas}
The high-level idea of our reduction is to schedule multiple instances of the base algorithm with different durations in a carefully-designed randomized scheme, which facilitates non-stationarity detection with little overhead. 
A related and well-known approach for non-stationary environments is to maintain multiple instances of a base algorithm with different parameter tunings or different starting points and to learn the best of them via another ``expert'' algorithm,
which can be very successful when learning with full information~\citep{hazan2007adaptive, luo2015achieving, daniely2015strongly, jun2017improved} but is suboptimal and has many limitations when learning with partial information~\citep{luo2018efficient,cheung2019learning,cheung2020reinforcement}.
Our approach is different as we do not try to learn the best instance; instead, we always follow the decision suggested by the instance with the currently shortest scheduled duration, and also only update this instance after receiving feedback from the environment.
The is because base algorithms with shorter duration are responsible for detecting larger distribution changes,
and always following the shortest one ensure that it is not blocked by the longer ones and thus every scale of distribution change is detected in a timely manner.

Another related approach is \emph{regret balancing}, developed recently  for model selection in bandit problems~\citep{abbasi2020regret, pacchiano2020regret}. The idea is also to run multiple base algorithms in parallel, each with a putative regret upper bound. The learner executes one of them in each round which incurs the least regret so far, and also constantly compares the performance among base algorithms, eliminating those whose putative regret bounds are violated. While our algorithm resembles regret balancing in some aspects, the way it chooses the base algorithm in each round is clearly quite different, which is also crucial for our problem. 

\paragraph{Other related work}
There are also a series of works on learning MDPs with adversarial rewards and a {\it fixed} transition \citep{even2009online, neu2010online,  arora2012deterministic, neu2012adversarial, dekel2013better, neu2013online, zimin2013online, dick2014online, rosenberg2019online, cai2020provably, jin2020learning, shani2020optimistic, rosenberg2020adversarial, lee2020bias, jin2020simultaneously, chen2020minimax, lancewicki2020learning}. 
These models can potentially handle non-stationarity in the reward function but not the transition kernel (in fact, most of these works also only consider static regret).
\citet{lykouris2019corruption} investigate an episodic MDP setting where an adversary can corrupt both the reward and the transition for up to $L'$ episodes, and achieve dynamic regret $\otil(\min\{L'\sqrt{T}, L'/\textit{gap}\})$ without knowing $L'$, where $\textit{gap}$ is the minimal suboptimality gap and could be arbitrarily small. Since corruption of up to $L'$ episodes implies that the world changes at most $L=2L'$ times, our result improves theirs from $\otil(L'\sqrt{T})$ to $\otil(\sqrt{L'T})$ when $1/\textit{gap} > \sqrt{T}$. 
On the other hand, it is possible that $L$ is much smaller than $L'$ (e.g. $L=\Theta(1)$ while $L' = \Theta(T)$), in which case our results are also significantly better.


\section{Problem Setting, Main Results, and High-level Ideas}



Throughout the paper, we fix a probability parameter $\delta$ of order $1/\text{poly}(T)$, and write $h_1(x)=\otil(h_2(x))$ or $h_2(x)=\widetilde{\Omega}(h_1(x))$ if $h_1(x)=\order\left(\text{poly}(\log(T/\delta)) h_2(x)\right)$. We say ``with high probability, $h_1=\otil(h_2(x))$'' if ``with probability $1-\delta$, $h_1=\otil(h_2(x))$''. 
For an integer $n$, we denote the set $\{1,2,\ldots, n\}$ by $[n]$;
and for integers $s$ and $e$, we denote the set $\{s,s+1,\ldots, e\}$ by $[s,e]$. 

\subsection{Problem setting}\label{sec:setting}
We consider the following general reinforcement learning (RL) framework that covers a wide range of problems.
Ahead of time, the learner is given a policy set $\Pi$,
and the environment decides $T$ reward functions $f_1, \ldots, f_T: \Pi \to [0,1]$ unknown to the learner.
Then, in each round $t = 1, \ldots, T$, the learner chooses a policy $\pi_t\in\Pi$ and receives a noisy reward $R_t \in [0,1]$ whose mean is $f_t(\pi_t)$.\footnote{%
The range $[0,1]$ is only for simplicity. Our results can be directly extended to the case with sub-Gaussian noise.
}
The dynamic regret of the learner is defined as 
$
    \Regret = \sum_{t=1}^T \left(f^\star_t - R_t\right)
$,
where $f^\star_t = \max_{\pi\in\Pi} f_t(\pi)$ is the expected reward of the optimal policy for round $t$.

Many heavily-studied problems fall into this framework.
For example, in the classic multi-armed bandit problem~\citep{lai1985asymptotically},
it suffices to treat each arm as a policy;
for finite-horizon episodic RL (e.g.~\citep{jin2018q}),
each state-to-action mapping is considered as a policy, and $f_t(\pi)$ is the expected reward of executing $\pi$ in the $t$-th episode's MDP with some transition kernel and some reward function.
See more examples in \pref{app: verify example}.
Note that our framework ignores many details of the actual problem we are trying to solve (e.g. not even mentioning the MDPs for RL).
This is because our results only rely on certain guarantees provided by a base algorithm, making these details irrelevant to our presentation.
There is also some technicality to fit the infinite-horizon RL problem into our framework, which we will discuss in detail in \pref{sec: average-reward}.




\paragraph{Non-stationarity measure}
A natural way to measure the distribution drift between rounds $t$ and $t+1$ is to see how much the expected reward of any policy could change, that is, $\max_{\pi\in\Pi} |f_{t}(\pi) - f_{t+1}(\pi)|$.
However, to make our results more general, we take a sligtly more abstract way to define non-stationarity whose exact form  eventually depends on what guarantees the base algorithm can provide for a concrete problem.
To this end, we define the following.
\begin{definition}\label{def: deviation}
$\dev: [T]\to \mathbb{R}$ is a non-stationarity measure if it satisfies $\dev(t) \geq \max_{\pi\in\Pi} |f_{t}(\pi) - f_{t+1}(\pi)|$ for all $t$.
Define for any interval $\calI = [s,e]$, $\dev_{\calI}=\sum_{\tau=s}^{e-1} \dev(\tau)$ (note $\dev_{[s,s]} = 0$) and $L_\calI=1+\sum_{\tau=s}^{e-1}\one[\dev(\tau)\neq 0]$.
With slight abuse of notation, we write $\dev=\dev_{[1,T]}$ and $L=L_{[1,T]}$. 
\end{definition}

\paragraph{Base algorithm and requirements}
As mentioned, our approach takes a base algorithm that tackles the problem when the environment is (near-)stationary, and turns it into another algorithm that can deal with non-stationary environments.
Throughout the paper, we denote the base algorithm by \alg and assumes that it satisfies the following mild requirements when run alone.

\begin{assumption}
\label{assum:assump2}
 \alg outputs an auxiliary quantity $\tildef_t\in[0, 1]$ at the beginning of each round $t$.  
There exist a non-stationarity measure $\dev$ and a non-increasing function $\avgreg: [T] \to \mathbb{R}$ such that running \alg satisfies the following: for all $t\in[T]$, as long as $\dev_{[1,t]}\leq \avgreg(t)$, without knowing $\dev_{[1,t]}$ \alg ensures with probability at least $1-\frac{\delta}{T}$:  
    \begin{align}
        &\tildef_t \geq \min_{\tau\in [1,t]} f_\tau^\star - \dev_{[1,t]} 
        \qquad\text{and}\qquad
        \frac{1}{t}\sum_{\tau=1}^t \left(\tildef_\tau - R_\tau\right) \leq \avgreg(t) + \dev_{[1,t]}.   \label{eq: general condition 2}
    \end{align}
    Furthermore, we assume that $\avgreg(t)\geq \frac{1}{\sqrt{t}}$ and $\regbound(t)= t\avgreg(t)$ is a non-decreasing function. 
\end{assumption}

We unpack the meaning of this assumption and explain why this is a mild requirement via a few remarks below, followed by examples of existing algorithms that do satisfy our assumption.

First, consider choosing $\dev(t) = \max_{\pi\in\Pi} |f_{t}(\pi) - f_{t+1}(\pi)|$ and see what the assumption means for a stationary environment with $f_t= f$ and $\dev(t) = 0$ for all $t$.
In this case, \pref{eq: general condition 2} simply becomes
$\tildef_t \geq \max_{\pi \in \Pi} f(\pi)$ and $\sum_{\tau=1}^t \left(\tildef_\tau - R_\tau\right) \leq \regbound(t)$,
which are standard properties of \emph{Upper-Confidence-Bound (UCB)-based} algorithms, where $\tildef_t$ is an \emph{optimistic} estimator of the optimal reward and $\regbound(t)$ is the regret bound usually of order $\sqrt{t}$.
In fact, even for non-UCB-based algorithms that do not explicitly maintain optimistic estimators,
by looking into their analysis, it is still possible to extract a quantity $\tildef_t$ satisfying these two properties (see our example for contextual bandits in \pref{app: verify example}).
We also note that this requirement for the special stationary case is in fact all we need to achieve our claimed regret bound $\RegL$. 

Second, to simultaneously achieve the regret bound $\RegDev$ as well, we require {\it slightly more} from the base algorithm: in a {\it near-stationary} environment with $\dev_{[1,t]}\leq \avgreg(t)$, the two aforementioned properties still hold approximately with degradation $\dev_{[1,t]}$ (that is, \pref{eq: general condition 2}).\footnote{%
We use $\min_{\tau\in [1,t]} f_\tau^\star$ instead of the more natural one $f_t^\star$ since the former is weaker and the difference between these two is at most $\dev_{[1,t]}$ anyway.
} 
We call this a near-stationary environment because $\dev_{[1,t]}$ can be of order $\Theta(t)$ in a highly non-stationary environment, while here we restrict it to be at most
$\avgreg(t)$, which is non-increasing in $t$ (and in fact of order $1/\sqrt{t}$ in all our examples).
To the best of our knowledge, all UCB-based algorithms satisfy \pref{assum:assump2} with some suitable choice of $\dev$. The fact that we only require \pref{eq: general condition 2} to hold for near-stationary environments is the key to bypassing the technical difficulty of getting the optimal bound $\RegDev$ met in \citep{russac2019weighted, pmlr-v108-zhao20a, russac2020algorithms, faury2021regret, touati2020efficient, zhou2020nonstationary} for linear bandits, generalized linear bandits, and linear MDPs, as mentioned in \pref{sec: intro}.


Finally, noting that $\avgreg(t)$ and $\regbound(t)$ represent an average and an cumulative regret bound respectively, the monotonicity requirement on them is more than natural.
The requirement $\avgreg(t)\geq \frac{1}{\sqrt{t}}$ is also usually unavoidable without further structures in the problem.
Note that while we write $\avgreg$ and $\regbound$ as a function of $t$ only, they can depend on $\log(1/\delta)$, $\log T$,  the complexity of $\Pi$, and other problem-dependent parameters such as the number of states/actions of an MDP.

Following the order in \pref{tab:compare}, we now give a list of existing algorithms that satisfy \pref{assum:assump2} in different problem settings with proper non-stationarity measure $\dev$ and regret bound $\regbound$.
We defer the concrete form of $\tildef_t$ (which requires introducing other notations) and all the proofs to \pref{app: verify example}. 
\begin{itemize}
\item UCB1~\citep{auer2002finite}: $\regbound(t)=\otil(\sqrt{At} + A)$ and $\dev(t)=\Theta(\|r_t - r_{t+1}\|_\infty)$, where $A$ is the number of arms, and $r_t$ is the expected reward vector at time $t$.

\item  ILTCB~\citep[short for ILOVETOCONBANDITS]{agarwal2014taming}: 
$\regbound(t)=\otil(\sqrt{At\log|\Pi|}+A\log|\Pi|)$ and $\dev(t)=\Theta\left( \int_r \int_x  |\calD_t(x,r) - \calD_{t+1}(x,r)| \mathrm{d}x\mathrm{d}r \right)$, where $A$ is the number of actions and $\calD_t$ is the joint distribution of the context-reward pair $(x,r)$ at time $t$.

\item FALCON~\citep{simchi2020bypassing}: $\regbound(t)=\otil(\sqrt{At\log|\Phi|} + A\log|\Phi|)$ and $\dev(t)=\Theta(\sqrt{A}\max_{x,a} |\phi_t^\star(x,a) - \phi_{t+1}^\star(x,a)| + \int_x |\calD_t(x) - \calD_{t+1}(x)|\mathrm{d}x)$, where $A$ is the number of actions, $\Phi$ is the set of regressors (each of which maps a context-action pair to a predicted reward), $\phi^\star_t \in \Phi$ is the true regressor at time $t$, and $\calD_t$ is the distribution of contexts at time $t$. 

\item OFUL~\citep{abbasi2011improved}: $\regbound(t)=\otil\left(d\sqrt{t}\right)$ and $\dev(t)=\widetilde{\Theta}(d\|\theta_t-\theta_{t+1}\|_2)$, where $d$ is the feature dimension and $\theta_t \in\mathbb{R}^d$ parameterizes the linear reward function at time $t$.

\item GLM-UCB~\citep{filippi2010parametric}: $\regbound(t)=\otil\left(\frac{k_\mu d}{c_\mu}\sqrt{t}\right)$ and $\dev(t)=\widetilde{\Theta}\left(\frac{k_\mu^2 d}{c_\mu}\|\theta_t-\theta_{t+1}\|_2\right)$, where $d$ is the feature dimension, $\theta_t \in\mathbb{R}^d$ parameterizes the linear reward function at time $t$, and $k_\mu, c_\mu$ are the upper and lower bounds of the gradient of the link function. 

\item Q-UCB~\citep[short for Q-learning UCB-H]{jin2018q}:\footnote{%
For ease of comparison, here, the reward range is changed from $[0,1]$ to the more common $[0,H]$. \label{fn:scaling1}
}  
$\regbound(t)=\otil(\sqrt{H^5SAt} + H^3SA)$ and $\dev(t)=\Theta(H\sum_{h=1}^H \max_{s,a}|r_h^{t}(s,a)-r_{h}^{t+1}(s,a)| + H^2\sum_{h=1}^H \max_{s,a}\|p_h^{t}(\cdot|s,a)-p_{h}^{t+1}(\cdot|s,a)\|_1)$, where $H$, $S$ and $A$ are the numbers of layers, states, and actions of the MDP respectively, and
$p_h^{t}$ and $r_h^{t}$ are the transition and reward functions for layer $h$ of episode $t$.

\item LSVI-UCB~\citep{jin2020provably}:\footnote{%
Same as \pref{fn:scaling1}.
}  $\regbound(t)=\otil(\sqrt{d^3 H^4 t})$ and $\dev(t)=\widetilde{\Theta}(dH\sum_{h=1}^H \|\theta_h^{t}-\theta_{h}^{t+1}\|_2 + dH^2\sum_{h=1}^H \|\mu_h^{t}-\mu_{h}^{t+1}\|_{F})$, where $d$ is the feature dimension, $H$ is the number of layers, and $\theta_h^{t}$ and $\mu_h^{t}$ are the parameters of the linear MDP for layer $h$ of episode $t$. 
\end{itemize}

\subsection{Main results}\label{sec:main_results}

Our main result is that, with an algorithm satisfying \pref{assum:assump2} at hand, our proposed black-box reduction, \masteralg (\pref{alg: final adaptive alg}), ensures the following dynamic regret bound.

\begin{theorem}\label{thm: regret bound}
     If \pref{assum:assump2} holds with $\regbound(t)= c_1t^p + c_2$ for some $p\in[\frac{1}{2}, 1)$ and $c_1, c_2 >0$, then \masteralg (\pref{alg: final adaptive alg}), without knowing $L$ and $\dev$, guarantees with high probability:
     \begin{align*}
      \Regret = \otil\left( \min\left\{\left(c_1 + \frac{c_2}{c_1}\right)\sqrt{LT}, \;  \left(c_1^{\nicefrac{2}{3}} + c_2c_1^{-\nicefrac{4}{3}}\right)\dev^{\nicefrac{1}{3}}T^{\nicefrac{2}{3}} +  \left(c_1 + \frac{c_2}{c_1}\right)\sqrt{T} \right\} \right)
    \end{align*}
    when $p=\frac{1}{2}$, and 
$
         \Regret=  \otil\left(\min\left\{c_1 L^{1-p} T^{p}, \;\left(c_1\dev^{1-p}T\right)^{\frac{1}{2-p}} + c_1T^{p}\right\}\right)
$
    when $p>\frac{1}{2}$ (omitting some lower-order terms). 
\end{theorem}

For ease of presentation, in this theorem we assume that $\regbound(\cdot)$ takes a certain form that is common in the literature and holds for all our examples with $p=\frac{1}{2}$.
Applying this theorem to all the examples discussed earlier, we achieve all the optimal $\min\{\RegL, \RegDev\}$ bounds listed in \pref{tab:compare} (except for infinite-horizon MDPs which will be discussed in \pref{sec: average-reward}).
Our definitions of $L$ are the same as in previous works,
and our definitions of $\dev$ are sometimes larger by some problem-dependent factors (such as $d$ and $H$) in order to fit \pref{assum:assump2}.
More specifically, for (contextual) bandits, our \masteralg combined with UCB1 and ILTCB recovers the same optimal bounds (in terms of all parameters) achieved by~\citep{auer2019adaptively, chen2019new}. 
\masteralg with FALCON obtains a similar bound as in \citep{chen2019new} but with a different definition of $\dev$ specific to the regressor setting. 
For other settings, we present our results in terms of the common definition of the non-stationarity measure (denoted by $\hatdev$) and compare them with the state of the art:



\begin{itemize}
%
%
%
\item \masteralg + OFUL: $\Regret=\otil( \min\{d\sqrt{LT}, d\hatdev^{\nicefrac{1}{3}}T^{\nicefrac{2}{3}}+d\sqrt{T}\})$, where $\hatdev = \sum_t \|\theta_t-\theta_{t+1}\|_2$.
This improves~\citep{cheung2019learning, russac2019weighted, kim2019randomized, pmlr-v108-zhao20a, zhao2021nonstationary} which get $\otil(d^{\nicefrac{7}{8}}\hatdev^{\nicefrac{1}{4}}T^{\nicefrac{3}{4}} + d\sqrt{T})$ when $\hatdev$ is known.
\item \masteralg + GLM-UCB:  $\Regret=\otil\Big( \min\Big\{\frac{k_\mu}{c_\mu}d\sqrt{LT}, \frac{k_\mu^{\nicefrac{4}{3}}}{c_\mu}d\hatdev^{\nicefrac{1}{3}}T^{\nicefrac{2}{3}}+\frac{k_\mu}{c_\mu}d\sqrt{T}\Big\}\Big)$, where $\hatdev = \sum_t \|\theta_t-\theta_{t+1}\|_2$.
This improves~\citep{russac2020algorithms} which gets $\otil\big(\frac{k_\mu}{c_\mu}d^{\nicefrac{2}{3}}L^{\nicefrac{1}{3}}T^{\nicefrac{2}{3}}\big)$ when $L$ is known, and \citep{faury2021regret} which gets $\otil\big(\frac{k_\mu}{c_\mu}d^{\nicefrac{9}{10}}\hatdev^{\nicefrac{1}{5}}T^{\nicefrac{4}{5}}\big)$. 
\item \masteralg + Q-UCB:
$\Regret=\otil(\min\{\sqrt{H^5SALT}, (H^{7}SA\hatdev)^{\nicefrac{1}{3}} T^{\nicefrac{2}{3}} + \sqrt{H^5SAT}\})$, where $\hatdev = \sum_{t,h}\max_{s,a}(|r_h^{t}(s,a)-r_{h}^{t+1}(s,a)| + \|p_h^{t}(\cdot|s,a)-p_{h}^{t+1}(\cdot|s,a)\|_1)$.\footnote{%
Due to the scaling mentioned in \pref{fn:scaling1},
here, we first scale down $\regbound(\cdot)$ and $\dev$ by an $H$ factor, then apply \pref{thm: regret bound}, and finally scale up the final bound by an $H$ factor. \label{fn:scaling2}
} 
\citep[Theorem 3]{mao2020nearoptimal} gets $\otil((H^5SA\hatdev)^{\nicefrac{1}{3}}T^{\nicefrac{2}{3}} + \sqrt{H^3SAT})$ when $\hatdev$ is known.\footnote{The bound reported in \citep{mao2020nearoptimal} is $(H^3SA\hatdev)^{\nicefrac{1}{3}}T^{\nicefrac{2}{3}} + \sqrt{H^2SAT}$; however, their $T$ is the total number of timesteps while our $T$ is the number of episodes, and we have performed a proper translation between notations here. Their bound has a better $H$ dependency thanks to the use of Freedman-style confidence bounds.
The same idea unfortunately does not improve our bound due to the lower-order term $c_2$ in the definition of $C(t)$. \label{fn:translation1} }
\item \masteralg + LSVI-UCB:
$\Regret=\otil(\min\{\sqrt{d^3H^4 LT}, (d^4 H^{6}\hatdev)^{\nicefrac{1}{3}}T^{\nicefrac{2}{3}} + \sqrt{d^3 H^4 T}\})$, where $\hatdev = \sum_{t, h} (\|\theta_h^{t}-\theta_{h}^{t+1}\|_2 + \|\mu_h^{t}-\mu_{h}^{t+1}\|_{F})$.
This improves~\citep{zhou2020nonstationary, touati2020efficient} which get $\otil((d^5 H^8 \hatdev)^{\nicefrac{1}{4}}T^{\nicefrac{3}{4}}+ \sqrt{d^3H^4T})$ when $\hatdev$ is known.\footnote{The same scaling as in \pref{fn:scaling2} and \pref{fn:translation1} has been performed here.}
\end{itemize}

\subsection{High-level ideas}\label{subsec:ideas}

To get a high-level idea of our approach, first consider what could go wrong when running \alg alone in a non-stationary environment and how to fix that intuitively.
Decompose the dynamic regret as follows:
\begin{align}
    \underbrace{\sum_{\tau=1}^t \left(f_\tau^\star - \tildef_\tau\right)}_{\term_1} + \underbrace{\sum_{\tau=1}^t \left(\tildef_\tau - R_\tau\right)}_{\term_2}.    \label{eq: initial decomposition}
\end{align}
As mentioned, in a stationary environment, \alg ensures that $\term_1$ is simply non-positive and $\term_2$ is bounded by $C(t)$ directly.
In a non-stationary environment, however, both terms can be substantially larger.
If we can detect the event that either of them is abnormally large, we know that the environment has changed substantially, and should just restart \alg.
This detection can be easily done for $\term_2$ since both $\tildef_\tau$ and $R_\tau$ are observable, but not for $\term_1$ since $f_\tau^\star$ is of course unknown.
Note that, a large $\term_1$ implies that a policy, possibly suboptimal in the past, now becomes the optimal one with a much larger reward.
A single instance of \alg run from the beginning thus cannot detect this because suboptimal polices are naturally selected very infrequently.
\begin{figure}[t]
    \centering
    \includegraphics[width=0.9\textwidth, trim={0.5cm 4.3cm 0cm 3cm}, clip]{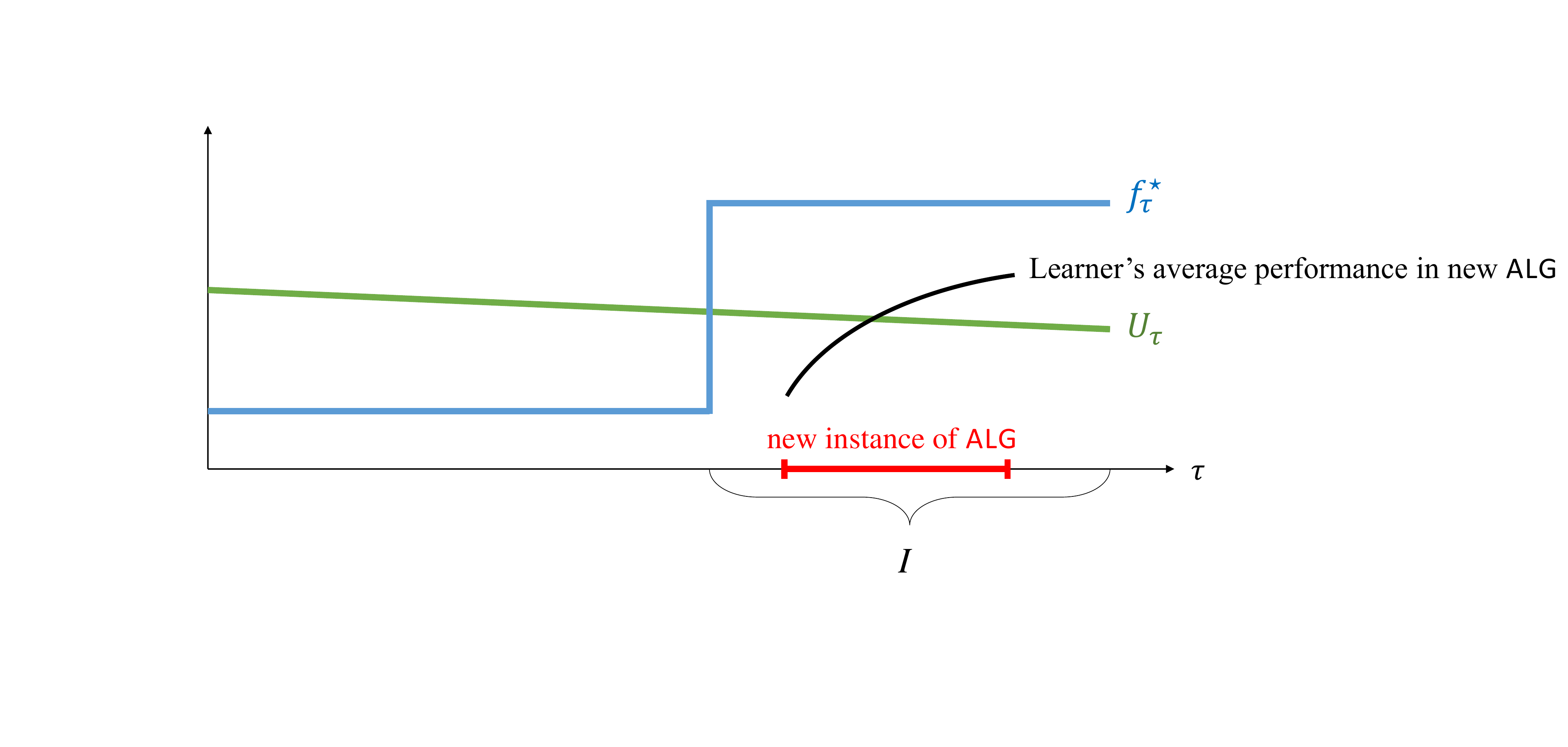}
    \caption{An illustration of how we detect non-stationarity via multiple instances of \alg}
    \label{fig: detection}
\end{figure}
To address this issue, our main idea is to maintain different instances of \alg to facilitates non-stationarity detection, illustrated via an example in \pref{fig: detection}.
Here, there is one distribution change that happens in interval $\calI$, making the value of $f_\tau^\star$ (the blue curve) drastically increase.
If within this interval, we start running another instance of \alg (the red interval),
then its performance (the black curve) will gradually approach $f_\tau^\star$ due to its regret guarantee in a stationary environment.
Hypothetically, if another instance of \alg run from the beginning could coexist with this new instance, we would see that the latter significantly outperforms the former and infer that the environment has changed.
The issue is that we cannot have multiple instances running and making decisions simultaneously, and here is where the optimistic estimators $\tildef_\tau$'s can help.
Specifically, since the quantity $U_\tau=\min_{s\leq \tau}\tildef_s$ (the green non-increasing curve) should always be an upper bound of the learner's performance in a stationary environment, if we find that the new instance of \alg significantly outperforms this quantity at some point (as shown in \pref{fig: detection}), we can also infer that the environment has changed, and prevent $\term_1 \leq \sum_{\tau=1}^t (f_\tau^\star - U_\tau)$ from growing too large by restarting.

To formally implement the ideas above, we need to decide when to start a new instance, how long it should last, which instance should be active if multiple exist, and others.
In \pref{sec: multi-scale}, we propose a randomized multi-scale scheme to do so, which is reminiscent of the ideas of {sampling obligation} in~\citep{auer2019adaptively} and {replay phase} in~\citep{chen2019new},
although their mechanisms are highly specific to their algorithms and problems.

\section{Algorithm}
\label{sec: multi-scale}
In this section, we first introduce \multialg, an algorithm that schedules and runs multiple instances of the base algorithm \alg in a multi-scale manner (\pref{subsec: multiscale}).
Then, equipping \multialg with non-stationarity detection, we introduce our final black-box reduction \masteralg (\pref{sec: final algorithm sec}).

\subsection{{\normalsize\textsf{MALG}}: ~Running the Base Algorithm with Multiple Scales}
\label{subsec: multiscale}

We always run \multialg for an interval of length $2^n$, which we call a \emph{block}, for some integer $n$ (unless it is terminated by the non-stationarity detection mechanism).
During initialization, \multialg uses \pref{proc: alg profile} to schedule multiple instances of \alg within the block in the following way:
for every $m = n, n-1, \ldots, 0$, partition the block equally into $2^{n-m}$ sub-intervals of length $2^m$, and for each of these sub-intervals, with probability $\frac{\avgreg(2^n)}{\avgreg(2^m)} \leq 1$ schedule an instance of \alg (otherwise skip this sub-interval).
We call these instances of length $2^m$ \emph{order-$m$} instances.

Note that by definition there is always an order-$n$ instance covering the entire block.
We use $\inst$ to denote a particular instance of \alg, and use $\inst.s$ and $\inst.e$ to denote its start and end time.

After the initialization, \multialg starts interacting with the environment as follows.
In each time $t$, the unique instance covering this time step with the shortest length is considered as being \emph{active}, while all others are \emph{inactive}.
\multialg follows the decision of the active instance, and update it after receiving feedback from the environment.
All inactive instances do not make any decisions or updates, that is, they are paused (they might be resumed at some point though).
We use $\tildeg_t$ to denote the scalar $\tildef_t$ output by the active instance.
See \pref{alg: multialg} for the pseudocode.

\begin{procedures}[t]
    \caption{A procedure that randomly
    schedules \alg of different lengths within $2^n$ rounds}
    \label{proc: alg profile}
    {\nonl \textbf{input}: $n$, $\avgreg(\cdot)$} \\
    \For{$\tau=0, \ldots, 2^{n}-1$}{
       \For{$m=n, n-1, \ldots, 0$}{ 
          If $\tau$ is a multiple of $2^m$, with probability $\frac{\avgreg(2^n)}{\avgreg(2^m)}$, schedule a new instance  \inst of \alg \\
          {\nonl that starts at $\inst.s=\tau+1$ and ends at $\inst.e=\tau+2^m$.}  
       }
    }
\end{procedures}

\begin{algorithm2e}[t]
\caption{\multialg (Multi-scale \alg)}
\label{alg: multialg}
{\nonl \textbf{input}: $n$, $\avgreg(\cdot)$} \\
\textbf{Initialization}: run \pref{proc: alg profile} with inputs $n$ and $\avgreg(\cdot)$. 

At each time $t$, let the unique active instance be $\inst$, output $\tildeg_t$ (which is the $\tildef_t$ output by $\inst$), follow $\inst$'s decision $\pi_t$, and update $\inst$ after receiving feedback from the environment.
\end{algorithm2e}
\begin{figure}[H]
    \includegraphics[width=\textwidth, trim={1cm 1cm 0.3cm 5.2cm}, clip]{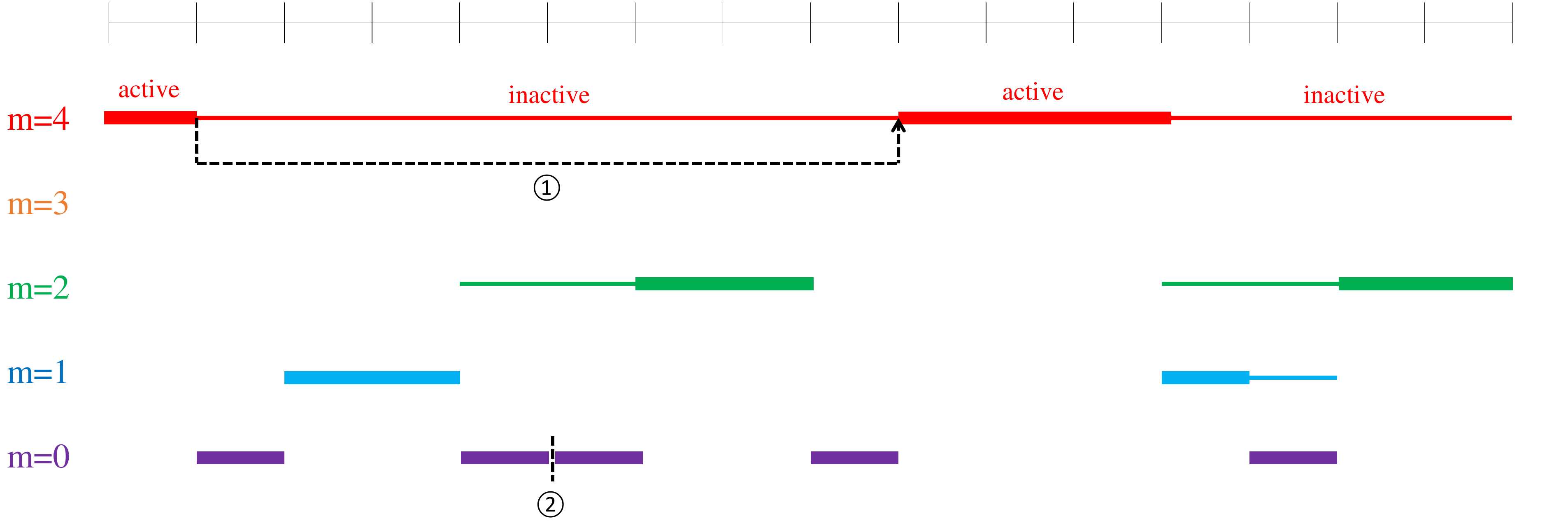}
    \caption{An illustration of \multialg with $n=4$ (see detailed explanation in Section~\ref{subsec: multiscale})}
    \label{fig: multiple schedule}
\end{figure}

For better illustration, we give an example with $n=4$ in \pref{fig: multiple schedule}. 
Suppose that the realization of the random scheduling by \pref{proc: alg profile} is: 
one order-$4$ instance (red), zero order-$3$ instance, two order-$2$ instances (green), two order-$1$ instances (blue), and five order-$0$ instances (purple).
The bolder part of the segment indicates the period of time when the instances are active, while the thinner part indicates the inactive period.
For example, the red order-$4$ instance is active for the first round, then paused for the next $8$ rounds, and then resumed (from the frozen internal states) for another $3$ rounds before becoming inactive again.
The dashed black arrow marked with \scalebox{0.8}{\circled{1}} indicates that \alg is executed as if the two sides of the arrow are concatenated. On the other hand,
as another example, the two purple instances on the two sides of the dashed line marked with \scalebox{0.8}{\circled{2}} are two \emph{different} order-$0$ instances, so the second one should start from scratch even though they are consecutive. 
One can see that at any point of time, the active instance is always the one with the shortest length.

\paragraph{Regret analysis of \multialg}

The multi-scale nature of \multialg allows the learner's regret to also enjoy a multi-scale structure, as shown in the next lemma (proof deferred to \pref{app:multi-scale}).

\begin{lemma}
     \label{lemma: multi-scale reg}
     Let $\nmax=\log_2 T+1$ and $\multiavg(t) = 6\nmax\log(T/\delta)\avgreg(t)$. 
     \multialg with input $n \leq \log_2 T$ guarantees the following: 
     for any instance \inst that \multialg maintains and any $t \in [\inst.s, \inst.e]$,
     as long as $\dev_{[\inst.s, t]}\leq \avgreg(t')$ where $t' =  t-\inst.s+1$,
     we have with probability at least $1-\frac{\delta}{T}$: 
    \begin{align}
        &\tildeg_t \geq \min_{\tau\in [\inst.s, t]} f_\tau^\star - \dev_{[\inst.s,t]}, 
        \qquad
        \frac{1}{t'}\sum_{\tau=\inst.s}^{t} \left(\tildeg_\tau - R_\tau\right) \leq \multiavg(t') + \nmax\dev_{[\inst.s,t]}, \label{eq: complicated bound}
    \end{align}
    and the number of instances started within $[\inst.s, t]$ is upper bounded by $6\nmax\log(T/\delta)\frac{\regbound(t')}{\regbound(1)}$. 
\end{lemma}

Note that \pref{eq: complicated bound} is essentially the analogue of \pref{eq: general condition 2} (up to logarithmic terms) with the starting time changed from $1$ to $\inst.s$.
It shows that even if we have multiple instances interleaving in a complicated way, 
the regret for a specific interval is still almost the same as running \alg alone on this interval, thanks to the carefully chosen probability in~\pref{proc: alg profile}.
Recall that there is always an order-$n$ instance starting from the beginning of the block, so \multialg is always providing a stronger multi-scale guarantee compared to running \alg alone.
This richer guarantee facilitates non-stationarity detection as we show next.

\subsection{{\normalsize\textsf{MASTER}}: ~Equipping {\normalsize\textsf{MALG}} with Stationarity Tests}
\label{sec: final algorithm sec}

We are now ready to present our final algorithm \masteralg, short for {\small\textbf{\textsf{MA}}\textsf{LG}} with \textbf{S}tationarity \textbf{TE}sts and \textbf{R}estarts (see \pref{alg: final adaptive alg}).
\masteralg runs \multialg in a sequence of blocks with doubling lengths ($2^0, 2^1, \ldots$).
Within each block of length $2^n$ (with $t_n$ being the starting time),
\masteralg simply runs a new instance of \multialg and records the minimum optimistic predictor thus far  for this block $U_t = \min_{\tau \in [t_n, t]} \tildeg_\tau$.
At the end of each time, \masteralg performs two tests (\testone and \testtwo),
and if either of them returns {\it fail},
\masteralg restarts from scratch.

The two tests exactly follow the ideas described in \pref{subsec:ideas} (recall \pref{fig: detection}).
Following \pref{eq: initial decomposition}, we decompose the regret on $[t_n, t]$ as $\term_1 + \term_2$ where $\term_1 = \sum_{\tau=t_n}^t \left(f_\tau^\star - \tildeg_\tau\right)$ and $\term_2 = \sum_{\tau=t_n}^t \left(\tildeg_\tau - R_\tau\right)$.
\testone prevents $\term_1 \leq \sum_{\tau=t_n}^t \left(f_\tau^\star - U_\tau\right)$ from growing too large by testing if there is some order-$m$ instance's interval during which the learner's average performance $\frac{1}{2^m}\sum_{\tau=\inst.s}^{\inst.e} R_{\tau}$ is larger than the promised performance upper bound $U_t$ by an amount of $9\multiavg(2^m)$.
On the other hand, \testtwo presents $\term_2$ from growing too large by directly testing if its average is large than something close to the promised regret bound $3\multiavg(t-t_n+1)$.

\begin{algorithm2e}[t]
    \caption{{\small\textbf{\textsf{MA}}\textsf{LG}} with \textbf{S}tationarity \textbf{TE}sts and \textbf{R}estarts (\masteralg)} 
    \label{alg: final adaptive alg}
    {\nonl \textbf{input}: $\multiavg(\cdot)$ (defined in \pref{lemma: multi-scale reg})}\\
    \textbf{Initialize}: $t\leftarrow1$ \\
         \For{$n=0, 1, \ldots$}{ \label{line: restart line}
              Set $t_n\leftarrow t$
              and initialize an \multialg (\pref{alg: multialg}) for the block $[t_n, t_n+2^n-1]$. \\
              \While{$t< t_n + 2^n$}{    \label{line: start of block}
                  Receive $\tildeg_t$ and $\pi_t$ from \multialg, execute $\pi_t$, and receive reward $R_t$. \\
                   Update \multialg with any feedback from the environment, and set $U_t = \min_{\tau \in [t_n, t]} \tildeg_\tau$.
                  \label{line: tracking repeated} \\ 
                  Perform \testone and \testtwo (see below). 
                  Increment $t\leftarrow t+1$. \\
                  \lIf{either test returns \textit{fail}}{
                      restart from \pref{line: restart line}.    
                  }
              }
         }
\nonl

\testone: \ \ If $t=\inst.e$ for some order-$m$ $\inst$ and
$\frac{1}{2^m}\sum_{\tau=\inst.s}^{\inst.e} R_{\tau} \geq  U_t  + 9\multiavg(2^m)$, 
return \textit{fail}.  

\testtwo: \ \  If 
$\frac{1}{t-t_n+1}\sum_{\tau=t_n}^{t} \left(\tildeg_{\tau} - R_{\tau}\right) \geq 3\multiavg(t-t_n+1)$, 
return \textit{fail}. 

\end{algorithm2e}



It is now clear that \masteralg indeed does not require the knowledge of $L$ or $\dev$ at all.
To analyze \masteralg, we prove the following key lemma that bounds the regret on a single block $[t_n, E_n]$ where $E_n$ is either $t_n+2^n - 1$ or something smaller in the case where a restart is triggered.


\begin{lemma}   \label{lemma: bound term1}
With high probability, the dynamic regret of \masteralg on any block $\calJ=[t_n, E_n]$ where $E_n \leq t_n+2^n -1$ is bounded as
    \begin{align}
        \sum_{\tau \in \calJ} \left(f_\tau^\star - R_\tau \right) \leq \otil\left(\sum_{i=1}^{\ell} \regbound(|\calI_i'|) + \sum_{m=0}^n \frac{\avgreg(2^m)}{\avgreg(2^n)}\regbound(2^m)\right).   \label{eq: block regret bound}
    \end{align}
    where $\{\calI_1', \ldots, \calI_\ell'\}$ is any partition of $\calJ$ such that $\dev_{\calI_i'}\leq \avgreg(|\calI_i'|)$ for all $i$.  
\end{lemma}


See \pref{app: key proofs} for the proof.
When $\rho(t) = \Theta(1/\sqrt{t})$ (as in all our examples), the first term is $\otil(\sum_{i=1}^\ell \sqrt{|\calI_i'|})=\otil(\sqrt{\ell|\calJ|})$ by Cauchy-Schwarz; the second term is of order $\otil(\sqrt{2^n})$. 
To derive a bound in terms of $L$, we can simply choose the partition $\{\calI_1', \ldots, \calI_\ell'\}$  in a way such that $\dev_{\calI_i'} = 0$ and $\ell = L_{\calJ}$,
while to derive a bound in terms of $\dev$, the partition needs to be chosen more carefully depending on the value of $\dev_{\calJ}$.
Noting that the number of blocks between two restarts is always at most $\log_2 T$,
to finally prove \pref{thm: regret bound}, it remains to bound the number of restarts, which intuitively should scale with $L$ or $\dev$ because by design a restart will not be triggered when the environment is stationary.
The complete proof is deferred to \pref{app: epoch regret}--\pref{app: omitted proof}.

\section{Extension to Reinforcement Learning in Infinite-horizon Communicating MDPs}
\label{sec: average-reward}

As mentioned, applying our results to infinite-horizon RL~\citep{jaksch2010near} requires some extra care and extensions.
We refer the reader to \citep{cheung2020reinforcement} for a thorough introduction on the problem setup of infinite-horizon RL in time-varying communicating MDPs.
Here, we only highlight its difference compared to episode RL and explain how to fit it into our framework.
Specifically, in episodic RL, we have treated each episode (consisting of multiple steps in an MDP) as one round of our framework, each state-to-action mapping as a policy $\pi$, and the expected reward of executing $\pi$ in the MDP for round $t$ as $f_t(\pi)$.
In infinite-horizon RL, while the meaning of $\pi$ and $f_t$ remains the same, there is no episode any more and the learner interacts with the changing MDP from the start to the end without any reset on her state.
In this case,  we treat each step (that is, each state transition) in the MDP as one round in our framework, and the meaning of the reward feedback $R_t$ has now changed from a noisy observation of the policy's reward $f_t(\pi_t)$ to just the reward of $\pi_t$ for this single step.
With this change, the dynamic regret definition remains the same.

Due to the black-box nature of our approach, if one has a base algorithm that satisfies something close to \pref{assum:assump2} within this setup, then it is not hard to imagine that the same idea of \masteralg can be applied.
In \pref{subsec:ACW}, we provide such a base algorithm,
and in \pref{subsec:MUCRL}, we combine it with appropriate multi-scale scheduling and detection to obtain our final results.

\subsection{{\normalsize\textsf{UCRL}} with Adaptive Confidence Widening}\label{subsec:ACW}
Our base algorithm, \acw, is an improvement of the standard \ucrl algorithm~\citep{jaksch2010near} and its variant \ucrlcw~\citep{cheung2020reinforcement}.
The pseudocode is shown in \pref{alg: base alg UCRL} (\pref{app:ucrl_details}), where we highlight the differences compared to \ucrl and \ucrlcw in blue.

The first difference is the explicit mention that the next state of the learner might sometime be {\it arbitrarily} assigned instead of following the transition of the current MDP (\pref{line: reassignment}).
This is necessary because of the multi-scale scheduling of \multialg.
Indeed, recall that in \multialg, an instance of the base algorithm can sometimes be paused and then resumed later.
In the infinite-horizon RL setup, this means that the instance can be resumed from an arbitrary state.
Other than making this detail explicit, however, nothing really needs to be changed in the algorithm, since this happens infrequently and only incurs small additional regret due to the communicating property of the MDPs.

The second key difference is an adaptive version of the {\it Confidence Widening} technique of~\citep{cheung2020reinforcement} (see \pref{line: adaptive start}--\pref{line: active end}).
As pointed out in~\citep{cheung2020reinforcement}, in non-stationary environments, the Extended Value Iteration (EVI) subroutine of \ucrl might return a bias vector ($\tildeh_k$) with span much larger than $D_{\max}$, the maximum diameters of all the MDPs.
To address this issue, their confidence widening technique adds a constant $\eta$,  tuned based on $\dev$, to the confidence level of the confidence set $\calP_k$,
which eventually leads to sub-optimal regret $\dev^{\nicefrac{1}{4}}T^{\nicefrac{3}{4}}$.
Our \emph{adaptive confidence widening}, on the other hand, adaptively selects the value of $\eta$ in a doubling manner, so that in a relatively stationary environment we only widen the confidence set slightly, while in a more non-stationary environment the widening is more significant.
To avoid incurring too much additional regret in the latter case,
we also monitor the cumulative widening amount and terminate the algorithm if it exceeds a certain threshold (\pref{line: accumulate}--\pref{line: return terminate}), because this implies that the environment is highly non-stationary.
(This termination will also be a restart signal for \masteralg.)

Finally, notice that our black-box approach requires knowing the regret bound $\avgreg(\cdot)$ of the base algorithm, which in this case depends on $D_{\max}$, a potentially unknown quantity.
To deal with this issue, \acw takes a guess $\dupp$ on the value of $D_{\max}$ as an additional input.
In the next subsection, we discuss how \masteralg decides the value of $\dupp$ when $D_{\max}$ is unknown.

With all these modifications, our base algorithm \acw indeed provides a guarantee similar to \pref{eq: general condition 2} of \pref{assum:assump2}; see \pref{lem: modified UCRL}.

\subsection{Multi-scale {\normalsize\textsf{UCRL-ACW}} and Its Combination with {\normalsize\textsf{MASTER}}}
\label{subsec:MUCRL}
Now, we use the same idea as in \pref{subsec: multiscale} to create a multi-scale version of \acw, under a fixed input $\dupp$. The resulted algorithm is called Multi-scale \acw or \mucrl for short (see \pref{alg: multialg-ucrl}).  \mucrl is basically identical to \multialg with \acw as the base algorithm, except that we let \mucrl terminate whenever the currently active \acw instance makes a restart signal (due to having an abnormally large cumulative widening amount).  
The guarantee for \mucrl is provided in \pref{lemma: ucrl aggregated regret}, which parallels \pref{lemma: multi-scale reg}.  

Next, as in \pref{sec: final algorithm sec}, we further combine \mucrl with non-stationarity tests, leading to \masterucrl (see \pref{alg: master ucrl}). 
The only difference compared to \masteralg is an additional condition to restart (highlighted in blue) --- when \mucrl terminates due to a restart signal from an \acw instance. We provide a single-block regret bound guarantee for \masterucrl under a fixed $\dupp$ in \pref{lemma: block lemma for RL}, which parallels \pref{lemma: bound term1}.   
Finally, we discuss three different cases with knowledge of different parameters (if any), leading to the three results listed in \pref{tab:compare}.

\paragraph{Known $D_{\max}$}
When $D_{\max}$ is known, we simply set $\dupp=D_{\max}$. 
In this case, all restarts of \masterucrl are due to non-stationarity, and we can bound their number in terms of $L$ or $\dev$.
Together with the single-block regret guarantee from \pref{lemma: block lemma for RL},
we prove that \masterucrl's dynamic regret is $\otil(\min\{\RegL, \RegDev\})$;
see \pref{thm: known Dmax case} for the dependence on other parameters.


\paragraph{Unknown $D_{\max}$ and Known $L$ or $\dev$}  When $D_{\max}$ is unknown, we unfortunately require the knowledge of $L$ to get $\RegL$ and the knowledge of $\dev$ to get $\RegDev$.
However, as shown in \pref{tab:compare}, 
this still significantly improves over the best existing bounds $\otil(L^{\nicefrac{1}{3}}T^{\nicefrac{2}{3}})$ and $\otil(\dev^{\nicefrac{1}{4}}T^{\nicefrac{3}{4}})$ when $L$ and $\dev$ are known.
Specifically, we apply a doubling trick to set the value of $\dupp$ following the 
strategy below, where we call the interval between two restarts an {\it epoch}:
\begin{enumerate}
    \item Initialize $\dupp\leftarrow 1$. 
    \item Run \masterucrl with $\dupp$. If the number of epochs exceeds $\overline{N}$, then double $\dupp$ and repeat this step. 
    Here, $\overline{N}$ is set to $L$ if $L$ is known or $1+3(S^{-2}A^{-1}\dev^2T)^{\nicefrac{1}{3}}$ if $\dev$ is known.
\end{enumerate}

The rationale behind monitoring the number of epochs is that, when $\dupp$ is too small, \acw might have an abnormally large cumulative widening amount and signal a restart even in a fairly stationary environment.
In \pref{lemma: RL number epoch}, we show that if $\dupp\geq D_{\max}$, the number of epochs produced by \masterucrl is upper bounded by the value of $\overline{N}$ set above. Therefore, if it exceeds this number, we can infer $\dupp < D_{\max}$ and double its value. 
This allows us to prove the regret bound $\RegL$ or $\RegDev$ again; see \pref{thm: doubling trick algo for RL} for the details.

\paragraph{No prior knowledge at all}  
When nothing is known, we apply the Bandit-over-Reinforcment-Learning (BoRL) framework of~\citep{cheung2019learning, cheung2020reinforcement} to get a suboptimal bound of order $\otil(\min\{\RegL, \RegDev\} +T^{\nicefrac{3}{4}})$.
BoRL also serves as a black-box reduction to obtain parameter-free algorithms (albeit suboptimal), so applying it to our algorithm is straightforward.
We omit the details and only give the concrete bound in \pref{app: borl discuss}. 
We leave the question of whether the optimal bound is achievable when $L$, $\dev$, and $D_{\max}$ are all unknown as a future direction.



\section{Conclusion and Future Directions}
In this work, we study reinforcement learning in non-stationary environments. We propose a general black-box approach that can convert an algorithm with near-optimal regret in a (near-)stationary environment to another algorithm with near-optimal dynamic regret in a non-stationary environment. Prior to our work, the bound of $\otil(\Delta^{\nicefrac{1}{3}}T^{\nicefrac{2}{3}})$ is only achievable with the knowledge of $\Delta$, and no algorithm achieves the bound of $\otil(\sqrt{LT})$ even with the knowledge of $L$. Our algorithm achieves both bounds simultaneously without any prior knowledge. 

It would be interesting to see whether algorithms with data-dependent bounds work with our black-box approach. Previous work in this direction \citep{wei2016tracking} achieves an improved dynamic regret bound for multi-armed bandits when the cumulative variance of the loss is small; however, their approach crucially relies on the knowledge on the degree of non-stationarity as well as the cumulative variance. On the other hand, there are some immediate difficulties in applying our black-box approach to data-dependent algorithms. For example, the monotonicity of the the average regret $\avgreg(\cdot)$ may not hold anymore, and it is unclear how to set the probability of initiating a new base algorithm. Therefore, the task of achieving data-dependent dynamic bounds without prior knowledge seems to be challenging and requires other innovations. 

Another future direction is to study a class of contextual bandit problems where the context is adversarially generated \citep{abbasi2011improved, cheung2019learning, foster2020beyond}. In this case, the expected reward of the optimal policy changes over time even if the environment is stationary, so our current algorithm cannot be directly applied. For linear contextual bandits with adversarial contexts \citep{abbasi2011improved, cheung2019learning}, the fix is straightforward though: instead of requiring the base algorithm to generate a scalar $\tildef_t$ in each round, we let it generate a \emph{confidence set} for the hidden parameter, and check the inconsistency of the confidence set over time. However, for general contextual bandits with adversarial contexts, where algorithms do not necessarily maintain a confident set for the hidden parameter~\citep{foster2020beyond}, the extension is less clear and is left for future investigation. 

Finally, we are not aware of any near-optimal \emph{convex bandit} algorithm satisfying our \pref{assum:assump2}, so achieving near-optimal dynamic regret bound in general convex bandits is also left open.  

\paragraph{Acknowledgments}
We thank Peng Zhao for pointing out the technical flaw made in previous works on non-stationary linear bandits as well as a fix in~\citep{zhao2021nonstationary}, and thank Ruihao Zhu for informing us their non-stationary linear bandit algorithm with the $\RegDev$ bound \citep{cheung2018hedging}.  
We also thank anonymous reviewers for pointing out the relation between our algorithm and regret balancing \citep{abbasi2020regret, pacchiano2020regret}. 
This work is supported by NSF
Award IIS-1943607 and a Google Faculty Research Award.

\bibliographystyle{plainnat}
\bibliography{ref}

\begin{thebibliography}{58}
\providecommand{\natexlab}[1]{#1}
\providecommand{\url}[1]{\texttt{#1}}
\expandafter\ifx\csname urlstyle\endcsname\relax
  \providecommand{\doi}[1]{doi: #1}\else
  \providecommand{\doi}{doi: \begingroup \urlstyle{rm}\Url}\fi

\bibitem[Abbasi-Yadkori et~al.(2011)Abbasi-Yadkori, P{\'a}l, and
  Szepesv{\'a}ri]{abbasi2011improved}
Yasin Abbasi-Yadkori, D{\'a}vid P{\'a}l, and Csaba Szepesv{\'a}ri.
\newblock Improved algorithms for linear stochastic bandits.
\newblock \emph{Advances in neural information processing systems},
  24:\penalty0 2312--2320, 2011.

\bibitem[Abbasi-Yadkori et~al.(2020)Abbasi-Yadkori, Pacchiano, and
  Phan]{abbasi2020regret}
Yasin Abbasi-Yadkori, Aldo Pacchiano, and My~Phan.
\newblock Regret balancing for bandit and rl model selection.
\newblock \emph{arXiv preprint arXiv:2006.05491}, 2020.

\bibitem[Agarwal et~al.(2014)Agarwal, Hsu, Kale, Langford, Li, and
  Schapire]{agarwal2014taming}
Alekh Agarwal, Daniel Hsu, Satyen Kale, John Langford, Lihong Li, and Robert
  Schapire.
\newblock Taming the monster: A fast and simple algorithm for contextual
  bandits.
\newblock In \emph{International Conference on Machine Learning}, pages
  1638--1646, 2014.

\bibitem[Arora et~al.(2012)Arora, Dekel, and Tewari]{arora2012deterministic}
Raman Arora, Ofer Dekel, and Ambuj Tewari.
\newblock Deterministic mdps with adversarial rewards and bandit feedback.
\newblock In \emph{Proceedings of the Twenty-Eighth Conference on Uncertainty
  in Artificial Intelligence}, pages 93--101, 2012.

\bibitem[Auer et~al.(2002{\natexlab{a}})Auer, Cesa-Bianchi, and
  Fischer]{auer2002finite}
Peter Auer, Nicolo Cesa-Bianchi, and Paul Fischer.
\newblock Finite-time analysis of the multiarmed bandit problem.
\newblock \emph{Machine learning}, 47\penalty0 (2-3):\penalty0 235--256,
  2002{\natexlab{a}}.

\bibitem[Auer et~al.(2002{\natexlab{b}})Auer, Cesa-Bianchi, Freund, and
  Schapire]{auer2002nonstochastic}
Peter Auer, Nicolo Cesa-Bianchi, Yoav Freund, and Robert~E Schapire.
\newblock The nonstochastic multiarmed bandit problem.
\newblock \emph{SIAM journal on computing}, 32\penalty0 (1):\penalty0 48--77,
  2002{\natexlab{b}}.

\bibitem[Auer et~al.(2019)Auer, Gajane, and Ortner]{auer2019adaptively}
Peter Auer, Pratik Gajane, and Ronald Ortner.
\newblock Adaptively tracking the best bandit arm with an unknown number of
  distribution changes.
\newblock In \emph{Conference on Learning Theory}, pages 138--158, 2019.

\bibitem[Cai et~al.(2020)Cai, Yang, Jin, and Wang]{cai2020provably}
Qi~Cai, Zhuoran Yang, Chi Jin, and Zhaoran Wang.
\newblock Provably efficient exploration in policy optimization.
\newblock In \emph{International Conference on Machine Learning}, pages
  1283--1294. PMLR, 2020.

\bibitem[Chen et~al.(2021)Chen, Luo, and Wei]{chen2020minimax}
Liyu Chen, Haipeng Luo, and Chen-Yu Wei.
\newblock Minimax regret for stochastic shortest path with adversarial costs
  and known transition.
\newblock In \emph{Conference on Learning Theory}, 2021.

\bibitem[Chen et~al.(2020)Chen, Wang, Zhao, and Zheng]{chen2020combinatorial}
Wei Chen, Liwei Wang, Haoyu Zhao, and Kai Zheng.
\newblock Combinatorial semi-bandit in the non-stationary environment.
\newblock \emph{arXiv preprint arXiv:2002.03580}, 2020.

\bibitem[Chen et~al.(2019)Chen, Lee, Luo, and Wei]{chen2019new}
Yifang Chen, Chung-Wei Lee, Haipeng Luo, and Chen-Yu Wei.
\newblock A new algorithm for non-stationary contextual bandits: Efficient,
  optimal and parameter-free.
\newblock In \emph{Conference on Learning Theory}, pages 696--726, 2019.

\bibitem[Cheung et~al.(2018)Cheung, Simchi-Levi, and Zhu]{cheung2018hedging}
Wang~Chi Cheung, David Simchi-Levi, and Ruihao Zhu.
\newblock Hedging the drift: Learning to optimize under non-stationarity.
\newblock \emph{Available at SSRN 3261050}, 2018.

\bibitem[Cheung et~al.(2019)Cheung, Simchi-Levi, and Zhu]{cheung2019learning}
Wang~Chi Cheung, David Simchi-Levi, and Ruihao Zhu.
\newblock Learning to optimize under non-stationarity.
\newblock In \emph{The 22nd International Conference on Artificial Intelligence
  and Statistics}, pages 1079--1087. PMLR, 2019.

\bibitem[Cheung et~al.(2020)Cheung, Simchi-Levi, and
  Zhu]{cheung2020reinforcement}
Wang~Chi Cheung, David Simchi-Levi, and Ruihao Zhu.
\newblock Reinforcement learning for non-stationary markov decision processes:
  The blessing of (more) optimism.
\newblock In \emph{International Conference on Machine Learning}, pages
  1843--1854. PMLR, 2020.

\bibitem[Daniely et~al.(2015)Daniely, Gonen, and
  Shalev-Shwartz]{daniely2015strongly}
Amit Daniely, Alon Gonen, and Shai Shalev-Shwartz.
\newblock Strongly adaptive online learning.
\newblock In \emph{International Conference on Machine Learning}, pages
  1405--1411, 2015.

\bibitem[Dekel and Hazan(2013)]{dekel2013better}
Ofer Dekel and Elad Hazan.
\newblock Better rates for any adversarial deterministic mdp.
\newblock In \emph{International Conference on Machine Learning}, pages
  675--683, 2013.

\bibitem[Dick et~al.(2014)Dick, Gyorgy, and Szepesvari]{dick2014online}
Travis Dick, Andras Gyorgy, and Csaba Szepesvari.
\newblock Online learning in markov decision processes with changing cost
  sequences.
\newblock In \emph{International Conference on Machine Learning}, pages
  512--520, 2014.

\bibitem[Domingues et~al.(2021)Domingues, M{\'e}nard, Pirotta, Kaufmann, and
  Valko]{domingues2020kernel}
Omar~Darwiche Domingues, Pierre M{\'e}nard, Matteo Pirotta, Emilie Kaufmann,
  and Michal Valko.
\newblock A kernel-based approach to non-stationary reinforcement learning in
  metric spaces.
\newblock In \emph{International Conference on Artificial Intelligence and
  Statistics}, pages 3538--3546. PMLR, 2021.

\bibitem[Even-Dar et~al.(2009)Even-Dar, Kakade, and Mansour]{even2009online}
Eyal Even-Dar, Sham~M Kakade, and Yishay Mansour.
\newblock Online markov decision processes.
\newblock \emph{Mathematics of Operations Research}, 34\penalty0 (3):\penalty0
  726--736, 2009.

\bibitem[Faury et~al.(2021)Faury, Russac, Abeille, and
  Calauzènes]{faury2021regret}
Louis Faury, Yoan Russac, Marc Abeille, and Clément Calauzènes.
\newblock Regret bounds for generalized linear bandits under parameter drift.
\newblock \emph{arXiv preprint arXiv:2103.05750}, 2021.

\bibitem[Fei et~al.(2020)Fei, Yang, Wang, and Xie]{fei2020dynamic}
Yingjie Fei, Zhuoran Yang, Zhaoran Wang, and Qiaomin Xie.
\newblock Dynamic regret of policy optimization in non-stationary environments.
\newblock \emph{Advances in Neural Information Processing Systems}, 33, 2020.

\bibitem[Filippi et~al.(2010)Filippi, Capp{\'e}, Garivier, and
  Szepesv{\'a}ri]{filippi2010parametric}
Sarah Filippi, Olivier Capp{\'e}, Aur{\'e}lien Garivier, and Csaba
  Szepesv{\'a}ri.
\newblock Parametric bandits: the generalized linear case.
\newblock In \emph{Proceedings of the 23rd International Conference on Neural
  Information Processing Systems-Volume 1}, pages 586--594, 2010.

\bibitem[Foster and Rakhlin(2020)]{foster2020beyond}
Dylan Foster and Alexander Rakhlin.
\newblock Beyond ucb: Optimal and efficient contextual bandits with regression
  oracles.
\newblock In \emph{International Conference on Machine Learning}, pages
  3199--3210. PMLR, 2020.

\bibitem[Gajane et~al.(2018)Gajane, Ortner, and Auer]{gajane2018sliding}
Pratik Gajane, Ronald Ortner, and Peter Auer.
\newblock A sliding-window algorithm for markov decision processes with
  arbitrarily changing rewards and transitions.
\newblock \emph{arXiv preprint arXiv:1805.10066}, 2018.

\bibitem[Hazan and Seshadhri(2007)]{hazan2007adaptive}
Elad Hazan and Comandur Seshadhri.
\newblock Adaptive algorithms for online decision problems.
\newblock In \emph{Electronic colloquium on computational complexity (ECCC)},
  volume~14, 2007.

\bibitem[Jaksch et~al.(2010)Jaksch, Ortner, and Auer]{jaksch2010near}
Thomas Jaksch, Ronald Ortner, and Peter Auer.
\newblock Near-optimal regret bounds for reinforcement learning.
\newblock \emph{Journal of Machine Learning Research}, 11\penalty0 (4), 2010.

\bibitem[Jin et~al.(2018)Jin, Allen-Zhu, Bubeck, and Jordan]{jin2018q}
Chi Jin, Zeyuan Allen-Zhu, Sebastien Bubeck, and Michael~I Jordan.
\newblock Is q-learning provably efficient?
\newblock In \emph{Advances in neural information processing systems}, pages
  4863--4873, 2018.

\bibitem[Jin et~al.(2020{\natexlab{a}})Jin, Jin, Luo, Sra, and
  Yu]{jin2020learning}
Chi Jin, Tiancheng Jin, Haipeng Luo, Suvrit Sra, and Tiancheng Yu.
\newblock Learning adversarial markov decision processes with bandit feedback
  and unknown transition.
\newblock In \emph{International Conference on Machine Learning}, pages
  4860--4869. PMLR, 2020{\natexlab{a}}.

\bibitem[Jin et~al.(2020{\natexlab{b}})Jin, Yang, Wang, and
  Jordan]{jin2020provably}
Chi Jin, Zhuoran Yang, Zhaoran Wang, and Michael~I Jordan.
\newblock Provably efficient reinforcement learning with linear function
  approximation.
\newblock In \emph{Conference on Learning Theory}, pages 2137--2143. PMLR,
  2020{\natexlab{b}}.

\bibitem[Jin and Luo(2020)]{jin2020simultaneously}
Tiancheng Jin and Haipeng Luo.
\newblock Simultaneously learning stochastic and adversarial episodic mdps with
  known transition.
\newblock \emph{Advances in Neural Information Processing Systems}, 33, 2020.

\bibitem[Jun et~al.(2017)Jun, Orabona, Wright, and Willett]{jun2017improved}
Kwang-Sung Jun, Francesco Orabona, Stephen Wright, and Rebecca Willett.
\newblock Improved strongly adaptive online learning using coin betting.
\newblock In \emph{Artificial Intelligence and Statistics}, pages 943--951,
  2017.

\bibitem[Kim and Tewari(2020)]{kim2019randomized}
Baekjin Kim and Ambuj Tewari.
\newblock Randomized exploration for non-stationary stochastic linear bandits.
\newblock In \emph{Uncertainty in Artificial Intelligence}, 2020.

\bibitem[Lai and Robbins(1985)]{lai1985asymptotically}
Tze~Leung Lai and Herbert Robbins.
\newblock Asymptotically efficient adaptive allocation rules.
\newblock \emph{Advances in applied mathematics}, 6\penalty0 (1):\penalty0
  4--22, 1985.

\bibitem[Lancewicki et~al.(2020)Lancewicki, Rosenberg, and
  Mansour]{lancewicki2020learning}
Tal Lancewicki, Aviv Rosenberg, and Yishay Mansour.
\newblock Learning adversarial markov decision processes with delayed feedback.
\newblock \emph{arXiv preprint arXiv:2012.14843}, 2020.

\bibitem[Lee et~al.(2020)Lee, Luo, Wei, and Zhang]{lee2020bias}
Chung-Wei Lee, Haipeng Luo, Chen-Yu Wei, and Mengxiao Zhang.
\newblock Bias no more: high-probability data-dependent regret bounds for
  adversarial bandits and mdps.
\newblock \emph{Advances in Neural Information Processing Systems}, 33, 2020.

\bibitem[Li and Li(2019)]{li2019online}
Yingying Li and Na~Li.
\newblock Online learning for markov decision processes in nonstationary
  environments: A dynamic regret analysis.
\newblock In \emph{2019 American Control Conference (ACC)}, pages 1232--1237.
  IEEE, 2019.

\bibitem[Luo and Schapire(2015)]{luo2015achieving}
Haipeng Luo and Robert~E Schapire.
\newblock Achieving all with no parameters: Adanormalhedge.
\newblock In \emph{Conference on Learning Theory}, pages 1286--1304, 2015.

\bibitem[Luo et~al.(2018)Luo, Wei, Agarwal, and Langford]{luo2018efficient}
Haipeng Luo, Chen-Yu Wei, Alekh Agarwal, and John Langford.
\newblock Efficient contextual bandits in non-stationary worlds.
\newblock In \emph{Conference On Learning Theory}, pages 1739--1776. PMLR,
  2018.

\bibitem[Lykouris et~al.(2021)Lykouris, Simchowitz, Slivkins, and
  Sun]{lykouris2019corruption}
Thodoris Lykouris, Max Simchowitz, Aleksandrs Slivkins, and Wen Sun.
\newblock Corruption robust exploration in episodic reinforcement learning.
\newblock In \emph{Conference on Learning Theory}, 2021.

\bibitem[Mao et~al.(2021)Mao, Zhang, Zhu, Simchi-Levi, and
  Ba{\c{s}}ar]{mao2020nearoptimal}
Weichao Mao, Kaiqing Zhang, Ruihao Zhu, David Simchi-Levi, and Tamer
  Ba{\c{s}}ar.
\newblock Is model-free learning nearly optimal for non-stationary rl?
\newblock In \emph{International Conference on Machine Learning}, 2021.

\bibitem[Neu et~al.(2010)Neu, Gy{\"o}rgy, and Szepesv{\'a}ri]{neu2010online}
Gergely Neu, Andr{\'a}s Gy{\"o}rgy, and Csaba Szepesv{\'a}ri.
\newblock The online loop-free stochastic shortest-path problem.
\newblock In \emph{COLT}, volume 2010, pages 231--243. Citeseer, 2010.

\bibitem[Neu et~al.(2012)Neu, Gyorgy, and Szepesv{\'a}ri]{neu2012adversarial}
Gergely Neu, Andras Gyorgy, and Csaba Szepesv{\'a}ri.
\newblock The adversarial stochastic shortest path problem with unknown
  transition probabilities.
\newblock In \emph{Artificial Intelligence and Statistics}, pages 805--813,
  2012.

\bibitem[Neu et~al.(2013)Neu, Gy{\"o}rgy, Szepesv{\'a}ri, and
  Antos]{neu2013online}
Gergely Neu, Andr{\'a}s Gy{\"o}rgy, Csaba Szepesv{\'a}ri, and Andr{\'a}s Antos.
\newblock Online markov decision processes under bandit feedback.
\newblock \emph{IEEE Transactions on Automatic Control}, 59\penalty0
  (3):\penalty0 676--691, 2013.

\bibitem[Ortner et~al.(2020)Ortner, Gajane, and Auer]{ortner2020variational}
Ronald Ortner, Pratik Gajane, and Peter Auer.
\newblock Variational regret bounds for reinforcement learning.
\newblock In \emph{Uncertainty in Artificial Intelligence}, pages 81--90. PMLR,
  2020.

\bibitem[Pacchiano et~al.(2020)Pacchiano, Dann, Gentile, and
  Bartlett]{pacchiano2020regret}
Aldo Pacchiano, Christoph Dann, Claudio Gentile, and Peter Bartlett.
\newblock Regret bound balancing and elimination for model selection in bandits
  and rl.
\newblock \emph{arXiv preprint arXiv:2012.13045}, 2020.

\bibitem[Rosenberg and Mansour(2019)]{rosenberg2019online}
Aviv Rosenberg and Yishay Mansour.
\newblock Online convex optimization in adversarial markov decision processes.
\newblock In \emph{International Conference on Machine Learning}, pages
  5478--5486, 2019.

\bibitem[Rosenberg and Mansour(2020)]{rosenberg2020adversarial}
Aviv Rosenberg and Yishay Mansour.
\newblock Stochastic shortest path with adversarially changing costs.
\newblock \emph{arXiv preprint arXiv:2006.11561}, 2020.

\bibitem[Russac et~al.(2019)Russac, Vernade, and Capp{\'e}]{russac2019weighted}
Yoan Russac, Claire Vernade, and Olivier Capp{\'e}.
\newblock Weighted linear bandits for non-stationary environments.
\newblock \emph{Advances in Neural Information Processing Systems}, 2019.

\bibitem[Russac et~al.(2020)Russac, Capp{\'e}, and
  Garivier]{russac2020algorithms}
Yoan Russac, Olivier Capp{\'e}, and Aur{\'e}lien Garivier.
\newblock Algorithms for non-stationary generalized linear bandits.
\newblock \emph{arXiv preprint arXiv:2003.10113}, 2020.

\bibitem[Shani et~al.(2020)Shani, Efroni, Rosenberg, and
  Mannor]{shani2020optimistic}
Lior Shani, Yonathan Efroni, Aviv Rosenberg, and Shie Mannor.
\newblock Optimistic policy optimization with bandit feedback.
\newblock In \emph{International Conference on Machine Learning}, pages
  8604--8613. PMLR, 2020.

\bibitem[Simchi-Levi and Xu(2020)]{simchi2020bypassing}
David Simchi-Levi and Yunzong Xu.
\newblock Bypassing the monster: A faster and simpler optimal algorithm for
  contextual bandits under realizability.
\newblock \emph{Available at SSRN}, 2020.

\bibitem[Touati and Vincent(2020)]{touati2020efficient}
Ahmed Touati and Pascal Vincent.
\newblock Efficient learning in non-stationary linear markov decision
  processes.
\newblock \emph{arXiv preprint arXiv:2010.12870}, 2020.

\bibitem[Wang et~al.(2020)Wang, Du, Yang, and Kakade]{wang2020long}
Ruosong Wang, Simon~S Du, Lin~F Yang, and Sham~M Kakade.
\newblock Is long horizon reinforcement learning more difficult than short
  horizon reinforcement learning?
\newblock \emph{Advances in Neural Information Processing Systems}, 2020.

\bibitem[Wei et~al.(2016)Wei, Hong, and Lu]{wei2016tracking}
Chen-Yu Wei, Yi-Te Hong, and Chi-Jen Lu.
\newblock Tracking the best expert in non-stationary stochastic environments.
\newblock \emph{Advances in neural information processing systems},
  29:\penalty0 3972--3980, 2016.

\bibitem[Zhao and Zhang(2021)]{zhao2021nonstationary}
Peng Zhao and Lijun Zhang.
\newblock Non-stationary linear bandits revisited.
\newblock \emph{arXiv preprint arXiv:2103.05324}, 2021.

\bibitem[Zhao et~al.(2020)Zhao, Zhang, Jiang, and Zhou]{pmlr-v108-zhao20a}
Peng Zhao, Lijun Zhang, Yuan Jiang, and Zhi-Hua Zhou.
\newblock A simple approach for non-stationary linear bandits.
\newblock In Silvia Chiappa and Roberto Calandra, editors, \emph{Proceedings of
  the Twenty Third International Conference on Artificial Intelligence and
  Statistics}, volume 108 of \emph{Proceedings of Machine Learning Research},
  pages 746--755. PMLR, 26--28 Aug 2020.

\bibitem[Zhou et~al.(2020)Zhou, Chen, Varshney, and
  Jagmohan]{zhou2020nonstationary}
Huozhi Zhou, Jinglin Chen, Lav~R Varshney, and Ashish Jagmohan.
\newblock Nonstationary reinforcement learning with linear function
  approximation.
\newblock \emph{arXiv preprint arXiv:2010.04244}, 2020.

\bibitem[Zimin and Neu(2013)]{zimin2013online}
Alexander Zimin and Gergely Neu.
\newblock Online learning in episodic markovian decision processes by relative
  entropy policy search.
\newblock \emph{Advances in neural information processing systems},
  26:\penalty0 1583--1591, 2013.

\end{thebibliography}
\newpage

%
\appendix

\section{Omitted Algorithms and Main Results in \pref{sec: average-reward}}\label{app:ucrl_details}

\begin{algorithm2e}[H]
    \caption{\ucrl with Adaptive Confidence Widening (\acw)}
    \label{alg: base alg UCRL}
    \textbf{input}: $\dupp\geq 1$ \\
    $t\leftarrow 1$. \ \ 
    $N_1(s,a)\leftarrow 0$ for all $s,a$. \  \highlight{$\Gamma\leftarrow 0$} \\
    \For{episode $k=1,\ldots, $}{
    Set $t_k=t$, $\nu_k(s,a)=0$ for all $s, a$. \\       
    Define for all $s,a$:
        \begin{align*}
            \hatp_{k}(s'|s,a) &= \frac{\sum_{\tau=1}^{t-1}\one[(s_\tau,a_\tau,s_{\tau+1}')=(s,a,s')]}{N_k^+(s,a)}, \\
            \hatr_{k}(s,a) &= \frac{\sum_{\tau=1}^{t-1}R_\tau\one[(s_\tau,a_\tau)=(s,a)]}{N_k^+(s,a)}.
        \end{align*}
        \\
        {\nonl and for any $\eta$: }
        \begin{align*}
            \calP_{k}^\eta(s,a) &= \left\{\tildep(\cdot|s,a)\in\Delta_S: \|\tildep(\cdot|s,a)-\hatp_{k}(\cdot|s,a)\|_1\leq  \sqrt{S}\cdot\conf_{k}(s,a) + \eta \right\}\\ 
            \calR_{k}(s,a) &= \left\{\tilder(s,a)\in[0,1]: |\tilder(s,a)-\hatr_{k}(s,a)|\leq \conf_{k}(s,a)\right\}
        \end{align*}
        {\nonl where $\conf_{k}(s,a)\triangleq 8\sqrt{\frac{\log(SAT/\delta)}{N_{k}^+(s,a)}}$ and $N_{k}^+(s,a) = \max\{1, N_k(s,a)\}$.}\\
        {\nonl \ \ \\}
        
        \highlight{
        $\eta \leftarrow \frac{1}{T}$    \label{line: adaptive start}\\
        \While{\textit{true}}{
            Perform EVI on $(\calP_{k}^\eta, \calR_{k})$ with error parameter $\epsilon_{k}=\sqrt{\frac{1}{t}}$, and obtain $\tildepi$, $\tildeh$, $\tildeJ$. \\
             \lIf{$\spn(\tildeh) \leq 2\dupp$}{\textbf{break}}    \label{line: enough widening}
             $\eta\leftarrow 2\eta$
        }   \label{line: active end}
        $\pi_{k}\leftarrow \tildepi, \ \ \tildeh_k\leftarrow \tildeh, \ \ \tildeJ_k\leftarrow \tildeJ, \ \ \eta_k\leftarrow \eta$   \myComment{Adaptive confidence widening}   \\
        {\nonl \ \ \\}
        }
        
        \While{$\nu_k(s,a)<N_k^+(s,a)$ for all $s,a$}{
            Choose action $a_t\sim \pi_k(s_t)$. \\
            $\nu_k(s_t,a_t)\leftarrow \nu_k(s_t,a_t) +1$ \\
            \highlight{
                \label{line: accumulate}$\Gamma \leftarrow \Gamma + \eta_k$   \label{line: accumulate}  \\
               \lIf{$\Gamma > 4S\sqrt{At\log(SAT/\delta)} $}{   \label{line: terminate condition}
                   \textit{terminate and signal restart}   \myComment{Early termination}    \label{line: return terminate}
               }
            }
            Observe the reward $R_t$ with $\E[R_t]= r_t(s_t,a_t)$ \\
            \highlight{Observe $s_{t+1}'\sim p_t(\cdot|s_t,a_t)$. \\
            The next state $s_{t+1}$ is either equal to $s_{t+1}'$, or re-assigned as an arbitrary state} \\
            \ \ \myComment{The next state might be re-assigned}  \label{line: reassignment}    \\
            $t\leftarrow t+1$
        }
        $N_{k+1}(s,a)\leftarrow N_k(s,a) + \nu_k(s,a)$ for all $s,a$. 
    }
\end{algorithm2e}

\begin{algorithm2e}[h]
\caption{Multi-scale \acw (\mucrl)}
\label{alg: multialg-ucrl}
{\nonl \textbf{input}: $n$, $\avgreg_{\ucrl}(\cdot~; \dupp)$, $\dupp$} \\
\textbf{Initialization}: run \pref{proc: alg profile} with base algorithm \acw and inputs $n$ and $\avgreg_{\ucrl}$. \\

At each time $t$, let the unique active instance be $\inst$, output $\tildeg_t$ (which is the quantity $\tildeJ_{k(t)}$ of $\inst$), follow $\inst$'s decision, and update $\inst$ after receiving feedback from the environment.
Additionally, \highlight{terminate if the $\inst$ signals restart}. 

\end{algorithm2e}

\newpage
\begin{algorithm2e}
    \caption{\masterucrl} 
    \label{alg: master ucrl}
    \textbf{input}:  $\avgreg_{\ucrl}(\cdot~; \dupp)$, $\dupp$\\
    \textbf{Initialize}: $t\leftarrow1$ \\
         \For{$n=0, 1, \ldots$}{ \label{line: restart line2}
              Set $t_n\leftarrow t$
              and initialize an \mucrl (\pref{alg: multialg-ucrl})  for the block $[t_n, t_n+2^n-1]$. \\
              \While{$t< t_n + 2^n$}{  
              Receive $\tildeg_t$ and $\pi_t$ from \mucrl, execute $\pi_t$, and receive reward $R_t$. \\
              Update \mucrl with any feedback from the environment, and set $U_t = \min_{\tau \in [t_n, t]} \tildeg_\tau$.
                    \label{line: tracking}\\
                  Perform \testone and \testtwo (see below). 
                  Increment $t\leftarrow t+1$. \\
                  \lIf{either test returns \textit{fail} \highlight{or \mucrl terminates} }{
                      restart from \pref{line: restart line2}. 
                  }
              }
         }
\testone: \ \ If $t=\inst.e$ for some order-$m$ $\inst$ and
$\frac{1}{2^m}\sum_{\tau=\inst.s}^{\inst.e} R_{\tau} \geq  U_t  + 9\multiavg_{\ucrl}(2^m; \dupp)$, 
return \textit{fail}.  

\testtwo: \ \  If 
$\frac{1}{t-t_n+1}\sum_{\tau=t_n}^{t} \left(\tildeg_{\tau} - R_{\tau}\right) \geq 3\multiavg_{\ucrl}(t-t_n+1; \dupp)$, 
return \textit{fail}. 

\end{algorithm2e}
The following is the main result for the infinite-horizon MDP case. Its proof requires several lemmas in the rest of this section, in addition to those from \pref{app:multi-scale}--\pref{app: epoch regret seperate} whose ideas are mostly aligned with the standard setting. The final analysis is done in \pref{app: omit proof for RL avg} and \pref{app: borl discuss} (see \pref{thm: known Dmax case}, \pref{thm: doubling trick algo for RL}, and the discussions in  \pref{app: borl discuss}). 
Note that to be consistent with prior works in this setting, we adopt the notation $J_t(\pi)$, which is the expected average reward of executing $\pi$ under the MDP for time $t$, and corresponds to the notation $f_t(\pi)$ we use in our general framework.
Similarly, define $J^\star_t  = \max_\pi J_t(\pi)$.

\begin{theorem}
     \label{thm: main theorem for RL}
     Define non-stationarity measures 
     \begin{align*}
          \dev &= \sum_{t=1}^{T-1} \left(\max_{s,a}|r_t(s,a)-r_{t+1}(s,a)| + \max_{s,a}\|p_t(\cdot|s,a)-p_{t+1}(\cdot|s,a)\|_1\right), \\
          L &= 1 + \sum_{t=1}^{T-1} \one\left\{\max_{s,a}|r_t(s,a)-r_{t+1}(s,a)| + \max_{s,a}\|p_t(\cdot|s,a)-p_{t+1}(\cdot|s,a)\|_1 \neq 0\right\}.
     \end{align*}
     There exists an algorithm that takes $D_{\max}$ as input and achieves
     \begin{align*}
          \sum_{t=1}^T \left(J^\star_t - R_t \right) = \otil\left(\min\left\{D_{\max}S\sqrt{ALT}, \ \ D_{\max}S^{\frac{2}{3}}A^{\frac{1}{3}}\dev^{\frac{1}{3}}T^{\frac{2}{3}} + D_{\max}S\sqrt{AT}\right\}\right)
     \end{align*}
     without knowing $L$ or $\dev$. 
     There is also an algorithm that takes $L$ or $\dev$ as input and achieves 
     \begin{align*}
          \sum_{t=1}^T \left(J^\star_t - R_t \right) = \otil\left(D_{\max}S\sqrt{ALT}\right) \quad \text{or}\quad 
          \sum_{t=1}^T \left(J^\star_t - R_t \right) = \otil\left(D_{\max}S^{\frac{2}{3}}A^{\frac{1}{3}}\dev^{\frac{1}{3}}T^{\frac{2}{3}} + D_{\max}S\sqrt{AT}\right)
     \end{align*}
     respectively, without knowing $D_{\max}$. Finally, there is an algorithm that achieves 
     \begin{align*}
    \sum_{t=1}^T \left(J^\star_t - R_t \right) = \otil\left( D_{\max}(S^2A)^{\nicefrac{1}{4}}T^{\nicefrac{3}{4}}  + \min\left\{ D_{\max}S\sqrt{ALT},\ \  D_{\max}(S^2A)^{\frac{1}{3}}\Delta^{\frac{1}{3}}T^{\frac{2}{3}}\right\}\right)
\end{align*}
without knowing $L, \Delta$, or $D_{\max}$. 
\end{theorem}

\subsection{Auxiliary Lemmas related to Extended Value Iteration and Bellman Equation}
In this subsection, we provide auxiliary lemmas related to EVI and Bellman Equation. The results are extracted from \citep{jaksch2010near, cheung2020reinforcement, ortner2020variational}. We restate them here for completeness.

\begin{lemma}[Properties 1 and 2 in \citep{cheung2020reinforcement}]
     Let $\tildeJ$, $\tildeh$, and $\tildepi$ be the set of solution obtained from EVI with confidence set $\calR$ and $\calP$ for reward and transition respectively, and error parameter $\epsilon$. Then
     \begin{align}
        \tildeJ + \tildeh(s) 
         &\geq \max_{a} \left(\max_{\tilder\in \calR(s,a)} \tilder(s,a) +  \max_{\tildep\in \calP(s,a)} \sum_{s'} \tildep(s'|s,a) \tildeh(s') \right),  \label{eq: optimistic bellman}\\
        \tildeJ + \tildeh(s) &\leq \max_{\tilder\in \calR(s,\tildepi(s))} \tilder(s,\tildepi(s)) +  \max_{\tildep\in \calP(s,\tildepi(s))} \sum_{s'} \tildep(s'|s,\tildepi(s)) \tildeh(s') + \epsilon. \label{eq: optimistic bellman 2}
    \end{align}
\end{lemma}

\begin{lemma}[Lemma 2 of \citep{cheung2020reinforcement}]
     \label{lemma: bound diamete}
    Let $\tildeJ$, $\tildeh$, and $\tildepi$ be the set of solution obtained from EVI with confidence set $\calR$ and $\calP$ for reward and transition respectively. If $\calP$ and $\calR$ contain an MDP with diameter upper bounded by $D$, then $\spn(\tildeh)\leq 2D$. 
\end{lemma}

\begin{lemma}[Eq. (16) of \citep{cheung2020reinforcement}]
\label{lemma: dual feasible}
    Let $r, p$ define the reward function and the transition kernel for a communicating MDP, respectively. Let $\tildeJ\in\mathbb{R}, \tildeh\in\mathbb{R}^S$ be bounded and satisfy
    \begin{align*}
         \tildeJ + \tildeh(s) \geq r(s,a) + \sum_{s'} p(s'|s,a)\tildeh(s')
    \end{align*}
    for all $s$ and $a$. Then $\tildeJ\geq J^\star$, where $J^\star$ is the average reward of the optimal policy under the MDP. 
\end{lemma}

\subsection{Guarantees of the \acw Algorithm (when running alone with an input $\dupp$)}

\begin{definition} \label{def: RL measure}
Define $ \dev^r(t) \triangleq \max_{s,a} |r_t(s,a) - r_{t+1}(s,a)|$, $\dev^p(t) \triangleq \max_{s,a} \| p_t(\cdot|s,a) - p_{t+1}(\cdot|s,a)\|_1$,  $\dev^J(t) \triangleq \max_{\pi} |J_t(\pi)-J_{t+1}(\pi)|$. Similar to \pref{def: deviation}, define $\Delta_{\calI}^{\Box}=\sum_{\tau=s}^{e-1}\dev^{\Box}(\tau)$ for interval $\calI=[s,e]$, where $\Box= r, p$, or $J$. Finally, we define $\dev_{\calI; \dupp} \triangleq  \dev_{\calI}^r + 2\dupp\dev_{\calI}^p + \dev_{\calI}^J$.
\end{definition}

\begin{lemma}[Theorem 1 of \citep{ortner2020variational}]
     \label{lemma: ortner lemma}
     $\dev^J(t)\leq \dev^r(t) + D_{\max} \dev^p(t)$. 
\end{lemma}

\begin{lemma}[\textit{c.f.} \pref{assum:assump2}]
     \label{lem: modified UCRL}
     When run alone, \pref{alg: base alg UCRL} with input $\dupp$ guarantees for all $t$ before it terminates:
     \begin{align*}
          \tildeJ_{k(t)} &\geq \min_{\tau\in [1,t]} J^\star_\tau - \dev_{[1,t]; \dupp} \\
          \frac{1}{t}\sum_{\tau=1}^{t}  \left( \tildeJ_{k(\tau)} - R_\tau \right) &\leq \avgreg_{\ucrl}\left(t; \dupp\right) + \frac{2}{t}\dupp\disc_{[1,t]} + \dev_{[1,t]; \dupp}, 
     \end{align*}
     where $\avgreg_{\ucrl}(t; \dupp) = \widetilde{\Theta}\left(\min\Big\{\dupp S\sqrt{\frac{A}{t}} + \frac{\dupp SA}{t}, \ \  \dupp\Big\}\right)$,  $\disc_{\calI}\triangleq \sum_{t\in\calI}\one[s_t \nsim p_{t-1}(\cdot|s_{t-1},a_{t-1}) ]$ is the number of state re-assignments within $\calI$ (\pref{line: reassignment} of \pref{alg: base alg UCRL}), and $k(t)$ is the index of the episode time $t$ belongs to. 
\end{lemma}

\begin{proof}\ \ \ 
    Suppose that at time $t$ the algorithm has not terminated. For any episode $k$ that starts before $t$, we have 
    \begin{align*}
        \barr_{k}(s,a) = \frac{\sum_{\tau=1}^{t_k-1} r_\tau(s,a)\one[(s_\tau,a_\tau)=(s,a)]}{N_k^+(s,a)}, \quad 
        \barp_k(s'|s,a) = \frac{\sum_{\tau=1}^{t_k-1} p_\tau(s'|s,a)\one[(s_\tau,a_\tau)=(s,a)]}{N_k^+(s,a)}.
    \end{align*}
    By Azuma's inequality, $\barr_k(s,a)\in\calR_k(s,a)$ and $\barp_k(\cdot|s,a)\in\calP_k^{\eta_k}(s,a)$ with high probability for all $k, s,a$.

    To show the first part of the lemma, we lower bound the right-hand side of \pref{eq: optimistic bellman}:
    \begin{align}
        \tildeJ_k + \tildeh_k(s) 
         &\geq \barr(s,a) +  \sum_{s'} \barp(s'|s,a) \tildeh_k(s')   \nonumber \\
         &\geq r_\tau(s,a) + \sum_{s'} p_\tau(s'|s,a)\tildeh_k(s') - \left(\dev^r_{[1,t]} + 2\dupp \dev^p_{[1,t]}\right), \tag{for any $\tau\in[1,t]$}
    \end{align}
where in the last inequality we use $|\barr(s,a) - r_\tau(s,a)|\leq \dev_{[1,t]}^r$ and $\sum_{s'} |\barp(s'|s,a) - p_\tau(s'|s,a)|\tildeh_k(s') \leq \|\barp(\cdot|s,a) -p_\tau(\cdot|s,a) \|_1 \spn(\tildeh_k)\leq 2 \dev_{[1,t]}^p \dupp$. 
    Using \pref{lemma: dual feasible}, we get \[\tildeJ_k + \dev_{[1,t]; \dupp} \geq \tildeJ_k + \left(\dev^r_{[1,t]} +2 \dupp \dev^p_{[1,t]}\right) \geq J_\tau^\star,\] implying the first part of the lemma.  
    
    To show the second part of the lemma, starting from \pref{eq: optimistic bellman 2}, we have with high probability
    \begin{align}
         &\tildeJ_k + \tildeh_k(s)\\
         &\leq \barr_k(s,\pi_k(s)) + \sum_{s'}\barp_k(s'|s,\pi_k(s))\tildeh_k(s') + 2\dupp \sqrt{S}\cdot \conf_k(s,\pi_k(s))  + 2 \dupp \eta_k +  \epsilon_k \tag{$\barr_k(s,a) \in \calR_k(s,a)$ and $\barp_k(s,a) \in \calP_k(s,a)$}  \\
         &\leq r_\tau(s,\pi_k(s)) + \sum_{s'}p_\tau(s'|s,\pi_k(s))\tildeh_k(s') + 2\dupp \sqrt{S}\cdot \conf_k(s,\pi_k(s))  + \left( \Delta^r_{[1,t]} + 2\dupp\Delta^p_{[1,t]} \right) +2 \dupp \eta_k + \epsilon_k.   \label{eq: to telescope}
    \end{align}

    Now, we apply \pref{eq: to telescope} with $(k,\tau,s) = \{k(\tau), \tau, s_\tau\}_{\tau=1}^t$ respectively, and sum them up. Notice that $a_\tau=\pi_{k(\tau)}(s_\tau)$. Then we get
    \begin{align*}
         \sum_{\tau=1}^t \left(\tildeJ_{k(\tau)} - R_{\tau}\right)
         &\leq \sum_{\tau=1}^t \left(\sum_{s'} p_{\tau}(s'|s_\tau,a_\tau)\tildeh_{k(\tau)}(s') - \tildeh_{k(\tau)}(s_\tau)\right) + \sum_{\tau=1}^t \left(r_\tau(s_\tau, a_\tau) - R_\tau\right) \\
         &\quad  + \sum_{\tau=1}^t 2\dupp\sqrt{S}\cdot\conf_{k(\tau)}(s_\tau, a_\tau)+  t\dev_{[1,t]; \dupp} + 2\dupp\sum_{\tau=1}^t \eta_{k(\tau)} +  \sum_{\tau=1}^t \epsilon_{k(\tau)}. 
    \end{align*}
    We bound the terms on the right-hand side individually: for the first term, notice that when there is no state-reassignment at time $\tau+1$, $\E_\tau [\tildeh_{k(\tau)}(s_{\tau+1})] = \sum_{s'}p_{\tau}(s'|s_\tau,a_\tau)\tildeh_{k(\tau)}(s')$. Therefore, 
    \begin{align*}
         &\sum_{\tau=1}^t \left(\sum_{s'} p_{\tau}(s'|s_\tau,a_\tau)\tildeh_{k(\tau)}(s') - \tildeh_{k(\tau)}(s_\tau)\right) \\
         &\leq \sum_{\tau=1}^t \left(\E_\tau \left[\tildeh_{k(\tau)}(s_{\tau+1})\right] - \tildeh_{k(\tau)}(s_\tau) \right) + 2\dupp\disc_{[1,t]} \\
         &\leq 2\dupp\sqrt{t\log(SAT)} + 2\dupp\sum_{\tau=1}^t \one\left[\tildeh_{k(\tau)} \neq \tildeh_{k(\tau+1)}\right] + 2\dupp\disc_{[1,t]} \tag{by Azuma's inequality} \\
         &\leq 2\dupp\sqrt{t\log(SAT)} + 2\dupp SA\log_2 T + 2\dupp\disc_{[1,t]}, 
    \end{align*}
    where in the last inequality we use the fact that the number of episodes cannot exceed $SA\log_2 T$.  For the other terms: 
    $\sum_{\tau=1}^t \left(r_\tau(s_\tau, a_\tau) - R_\tau\right) \leq \otil\left(\sqrt{t}\right)$ by Azuma's inequality; $\sum_{\tau=1}^t 2\dupp\sqrt{S}\cdot\conf_{k(\tau)}(s_\tau, a_\tau) = \otil\left( \dupp S\sqrt{At} \right)$ by the standard pigeonhole argument; $2\dupp\sum_{\tau=1}^t \eta_{k(\tau)} = \otil\left(\dupp S\sqrt{At}\right)$
 by the termination condition specified in \pref{line: terminate condition};  $\sum_{\tau=1}^t \epsilon_{k(\tau)}$ is also upper bounded by $\otil\left(\dupp S\sqrt{At}\right)$ by the way we choose the error parameter. Combining all the above arguments, we get 
 \begin{align*}
       \sum_{\tau=1}^t \left(\tildeJ_{k(\tau)} - R_{\tau}\right) \leq \otil\left(\dupp S\sqrt{At} + \dupp SA \right) + 2\dupp \disc_{[1,t]} + t\dev_{[1,t]; \dupp}
 \end{align*}
 with high probability. On the other hand, $\sum_{\tau=1}^t \left(\tildeJ_{k(\tau)} - R_{\tau}\right) \leq \dupp t$ is trivially true. Combining them we get the second claim of the lemma. 
\end{proof}

\section{Analysis for the Multi-scale Algorithms}\label{app:multi-scale}
\begin{proof}{\textbf{of \pref{lemma: multi-scale reg}} }
    Below, we fix an $\inst$ and fix a $t\in[\inst.s, \inst.e]$, and consider the case $\dev_{[\inst.s, t]}\leq \avgreg(t')$ as specified in the lemma statement. For the first part of the lemma, note that $\tildeg_t$ of \multialg is defined as $\tildef_t^{\inst'}$ where $\inst'$ is the active instance of \alg at round $t$. By \pref{proc: alg profile}, $\inst'$ can only be an instance that starts within $[\inst.s, t]$ (i.e., $\inst'.s\geq \inst.s$). Therefore, the distribution drift undergone by $\inst'$ up to $t$ is upper bounded by $\dev_{[\inst.s, t]}\leq \avgreg(t')$, which is further upper bounded by $\avgreg(t'')$ where $t''$ is the number of active rounds $\inst'$ runs within $[\inst.s, t]$, because $\avgreg(\cdot)$ is a decreasing function. Therefore, the conditions in \pref{assum:assump2} is satisfied for this $\inst'$, and thus we have 
    \begin{align*}
        \tildeg_t = \tildef_t^{\inst'} \geq \min_{\tau\leq t:~\inst' \text{\ is active at $\tau$}} f_\tau^\star -  \dev_{[\inst'.s,t]} \geq \min_{\tau \in [\inst.s, t]}f_{\tau}^\star - \dev_{[\inst.s, t]},  
    \end{align*}
    proving the first part. 
    
    Next, we prove the second part of the lemma. 
    We use $S_m$ to denote the set of order-$m$ instances which start within $[\inst.s, t]$. Note that
    \begin{align}
        \sum_{\tau=\inst.s}^{t} \left(\tildeg_\tau - R_\tau\right) 
        &= \sum_{\tau=\inst.s}^{t} \sum_{m=0}^n  \sum_{\inst'\in S_m} \one[\inst'\ \text{is active at\ } \tau]\left(\tildef_\tau^{\inst'}-R_\tau\right) \nonumber \\
        &= \sum_{m=0}^n  \underbrace{\sum_{\inst'\in S_m}\sum_{\tau=\inst.s}^{t}   \one[\inst'\ \text{is active at\ } \tau]\left(\tildef_\tau^{\inst'}-R_\tau\right)}_{(*)}.    \label{eq: regret decompose} 
    \end{align}
    The first equality holds because $\tildeg_\tau$ of \multialg is defined as the $\tildef_\tau$ of the active instance at round $t$. 
    
    Next, we focus on a specific $m$, and bound the $(*)$ term in \pref{eq: regret decompose}. Let $|S_m|=\ell$ and $S_m=\left\{\inst_1',\ldots, \inst_\ell'\right\}$, and let $\calI_i\triangleq [\inst_i'.s, \inst_i'.e]\cap [\inst.s, t]$ for $i=1,\ldots,\ell$ (i.e., $\calI_i$ are the rounds within $[\inst.s, t]$ where $\inst_i'$ is scheduled). 
    Clearly, $|\calI_i|\leq \min\{\inst'_i.e - \inst'_i.s +1 , t-\inst.s+1\} = \min\left\{2^m, t'\right\}$. By \pref{assum:assump2}, we have
    \begin{align}
        (*)&= \sum_{i=1}^\ell \sum_{\tau=\inst.s}^{t} \one[\inst_i'\text{\ is active at $\tau$}]\left(\tildef_\tau^{\inst_i'}-R_\tau\right) \nonumber \\
        &\leq \sum_{i=1}^\ell \left(\regbound(|\calI_i|) + |\calI_i|\dev_{\calI_i}\right)  \nonumber \\
        &\leq \ell\regbound(\min\{2^m, t'\}) + t'\dev_{[\inst.s, t]}, \label{eq: bound star}
    \end{align}
    where in the first inequality we use \pref{assum:assump2}, and that $\inst_i'$ updates for no more than $|\calI_i|$ rounds in the interval $[\inst.s, t]$ (also, the condition in \pref{assum:assump2} is satisfied because $\dev_{\calI_i}\leq \dev_{[\inst.s,t]}\leq \avgreg(t')\leq \avgreg(|\calI_i|)$). In the last inequality, for the first term, we use that $\regbound(\cdot)$ is increasing; for the second term, we use $|\calI_i|\leq t'$, and that $\dev_{\calI_1}+\cdots+\dev_{\calI_\ell}\leq \dev_{[\inst.s, t]}$ since $\calI_1, \ldots, \calI_\ell$ are non-overlapping intervals lying within $[\inst.s,t]$. 
    
    By \pref{proc: alg profile}, 
    for every $m$, the expected number of order-$m$ \alg's that starts within the interval $[\inst.s, t]$ can be upper bounded as 
    \begin{align}
        \E[|S_m|]
        &\leq \frac{\avgreg(2^n)}{\avgreg(2^m)} \left\lceil \frac{t'}{2^m} \right\rceil  
        \leq \frac{\avgreg(2^n)}{\avgreg(2^m)}\left( \frac{t'}{2^m}+1\right) \leq \frac{\avgreg(2^n)}{\avgreg(2^m)} \frac{t'}{2^m}+1     \label{eq: bound Sm} 
    \end{align}

    By Bernstein's inequality, with probability $1-\frac{\delta}{T}$,  $|S_m|\leq \E[|S_m|]+\sqrt{2\E[|S_m|]\log(T/\delta)} + \log(T/\delta) \leq 2\E[|S_m|] + 2\log(T/\delta)$. Thus, continuing from \pref{eq: bound star}, we have with probability at least $1-\frac{\delta}{T}$, 
    \begin{align}
        (*)
        &\leq 2 \cdot \left(\frac{\avgreg(2^n)}{\avgreg(2^m)} \frac{t'}{2^m} +1\right) \regbound(\min\{2^m, t'\}) + 2\log(T/\delta)\regbound(\min\{2^m, t'\}) + t'\dev_{[\inst.s, t]}    \nonumber \\
        &\leq 2 \left(\frac{\regbound(t')}{\regbound(2^m)}   + 2 \right) \log(T/\delta) \regbound(\min\{2^m, t'\}) + t'\dev_{[\inst.s, t]}   \tag{$\avgreg(2^n) t'\leq \avgreg(t')t'=\regbound(t')$}    \nonumber  \\
        &\leq 6 \regbound(t') \log(T/\delta) + t'\dev_{[\inst.s, t]}     \tag{$\regbound(\cdot)$ is an increasing function} \\
        \label{eq: upper bound star}
    \end{align}
    Finally, using this in \pref{eq: regret decompose}, we get the second claim of the lemma: with probability at least $1-\frac{\delta}{T}$, 
    \begin{align}
        \sum_{\tau=\inst.s}^t\left(\tildeg_\tau-R_\tau\right)
        &\leq 6(n+1) \regbound(t') \log(1/\delta) + t'(n+1)\dev_{[\inst.s, t]}.    \label{eq: regret temp}
    \end{align}

For the third part of the lemma, as we calculated above, with probability at least $1-\frac{\delta}{T}$, the number of instances started within $[\inst.s, t]$ is upper bounded by 
\begin{align*}
     \sum_{m=0}^n 2 \cdot \left(\frac{\avgreg(2^n)}{\avgreg(2^m)} \frac{t'}{2^m} +2\right) \log(T/\delta)\leq 2\nmax\left(\frac{\regbound(t')}{\regbound(1)} +2\right)\log(T/\delta) \leq 6\nmax\frac{\regbound(t')}{\regbound(1)}\log(T/\delta) 
\end{align*}
where we use $\avgreg(2^m)2^m = \regbound(2^m) \geq  \regbound(1)$ and $\avgreg(2^n)t'\leq \avgreg(t')t'=\regbound(t')$. 
\end{proof}

\begin{lemma}[\textit{c.f.} \pref{lemma: multi-scale reg}]
\label{lemma: ucrl aggregated regret}
      Before \mucrl terminates, for every \inst and $t\in[\inst.s, \inst.e]$, \mucrl guarantees with high probability
     \begin{align*}
          \tildeg_t &\geq \min_{\tau\in[\inst.s, t]} J_\tau^\star - \dev_{[\inst.s, t]; \dupp} \\
         \frac{1}{t'} \sum_{\tau=\inst.s}^{t}  \left(\tildeg_\tau - R_\tau\right)&\leq \multiavg_{\ucrl}\left(t'; \dupp\right) + \nmax\dev_{[\inst.s, t]; \dupp}
     \end{align*}
     where $t'=t-\inst.s+1$, $\nmax=\log_2 T+1$, and $\multiavg_{\ucrl}\left(t; \dupp\right)=18\nmax\log(T/\delta)\avgreg_{\ucrl}(t;\dupp)$. 
\end{lemma}

\begin{proof}
    This proof is similar to that of \pref{lemma: multi-scale reg}. For the first part of the lemma, we can simply follow the proof of the first part of \pref{lemma: multi-scale reg}, with $\tildef_t$ replaced by $\tildeJ_{k(t)}$, and $\dev_{[\inst.s, t]}$ by $\dev_{[\inst.s, t];\dupp}$. 

For the second part, the analysis still tightly follows that of \pref{lemma: multi-scale reg}, but we need to add the additional cost caused by state re-assignment (i.e., the $\dupp \disc_{[1,t]}$ term in \pref{lem: modified UCRL}). 
Following the same arguments as in proof as in \pref{eq: regret decompose} and \pref{eq: bound star}, we get 
 \begin{align}
        \sum_{\tau=\inst.s}^{t} \left(\tildeg_\tau - R_\tau\right)  
        &= \sum_{m=0}^n  \sum_{\inst'\in S_m}\sum_{\tau=\inst.s}^{t}   \one[\inst'\ \text{is active at\ } \tau]\left(\tildef_\tau^{\inst'}-R_\tau\right) \tag{$S_m\triangleq$ \text{the set of order-$m$ \alg initiated within $[\inst.s, t]$}} \\
        &\leq \sum_{m=0}^n \sum_{i=1}^{|S_m|}\sum_{\tau=\inst.s}^{t}   \one[\inst_{m,i}'\ \text{is active at\ } \tau]\left(\tildef_\tau^{\inst_{m,i}'}-R_\tau\right)    \tag{Let $S_m=\{\inst'_{m,1}, \inst'_{m,2}, \ldots\}$}\\
        &\leq \sum_{m=0}^n \sum_{i=1}^{|S_m|} \left(\regbound_{\ucrl}(|\calI_{m,i}|; \dupp) + |\calI_{m,i}|\dev_{\calI_{m,i}; \dupp} + 2\dupp\disc_{\calI_{m,i}}^{\inst_{m,i}'}\right) \tag{by \pref{lem: modified UCRL}}  \\
        &\label{eq: last expression MDP}
    \end{align}
    where in the last expression, we denote $\calI_{m,i}=[\inst'_{m,i}.s, \inst'_{m,i}.e] \cap [\inst.s, t]$ (the time within $[\inst.s, t]$ where $\inst'_{m,i}$ is scheduled), and $\disc^{\inst'}_{\calI}$ is the total number of times within $\calI$ when $\inst'$ encounters state-reassignments. 
   
    For a fixed $m$, observe that all order-$m$ instances are non-overlapping. Also, the aggregated number of state re-assignment for all order-$m$ instances started within $[\inst.s, t]$ is upper bounded by the total number of new instances of order not larger than $m-1$ started within $[\inst.s, t]$. The latter is further upper bounded by $6\nmax\log(T/\delta)\frac{\regbound_{\ucrl}(t'; \dupp)}{\regbound_{\ucrl}(1; \dupp)}$ according to the last claim of \pref{lemma: multi-scale reg}.  In other words, for every $m$, with probability $1 - \frac{\delta}{T}$, 
    \begin{align*}
          \sum_{i=1}^{|S_m|} \disc_{\calI_{m,i}}^{\inst_{m,i}'} = 6\nmax \log(T/\delta) \frac{\regbound_{\ucrl}(t'; \dupp)}{\regbound_{\ucrl}(1; \dupp)}. 
    \end{align*}
Following the same calculation as in \pref{eq: bound star}, \pref{eq: bound Sm} and \pref{eq: upper bound star}, we also have that for every $m$, with probability $1-\frac{\delta}{T}$, 
\begin{align*}
     \sum_{i=1}^{|S_m|} \left(\regbound_{\ucrl}(|\calI_{m,i}|; \dupp) + |\calI_{m,i}|\dev_{\calI_{m,i}; \dupp} \right) \leq 6\nmax \log(T/\delta)\regbound_{\ucrl}(t'; \dupp) + t'\dev_{[\inst.s, t]; \dupp}.
\end{align*}
Using the above two bounds in \pref{eq: last expression MDP}, we get 
\begin{align*}
     \sum_{\tau=\inst.s}^{t} \left(\tildeg_\tau - R_\tau\right)  
     &=  6\nmax \log(T/\delta)\regbound_{\ucrl}(t'; \dupp) + \nmax t'\dev_{[\inst.s, t]; \dupp} +2\dupp \times 6\nmax \log(T/\delta) \frac{\regbound_{\ucrl}(t'; \dupp)}{\regbound_{\ucrl}(1; \dupp)}\\
     &=  18\nmax\log(T/\delta)\regbound_{\ucrl}(t';\dupp) + \nmax t'\dev_{[\inst.s, t]; \dupp} 
\end{align*}
where we use $\regbound_{\ucrl}(1; \dupp)\geq \dupp$ (by the definition of $\avgreg_{\ucrl}(\cdot~,\dupp)$ in \pref{lem: modified UCRL}). Dividing both sides by $t'$ finishes the proof. 
\end{proof}




\section{Single-block Regret Analysis I} 
\label{app: key proofs}
In this section, we focus on the regret in a block of index $n$. The analysis applies to both the standard case (\pref{lemma: bound term1}), and the infinite-horizon RL case summarized in the following lemma.
\begin{lemma}[\textit{c.f.} \pref{lemma: bound term1}] 
\label{lemma: block lemma for RL}
     In a block of index $n$ that starts from $t_n$ and ends on $E_n$ ($E_n$ could be equal to $t_n+2^n-1$, or smaller, if any stationarity test fails or \mucrl terminates), we have
    \begin{align*}
    \sum_{\tau=t_n}^{E_n} \left(f_\tau^\star - R_\tau\right) \leq \otil\left( \sum_{i=1}^\ell \regbound_{\ucrl}(|\calI_i'|; \dupp) + \sum_{m=0}^n \frac{\avgreg(2^m; \dupp)}{\avgreg(2^n; \dupp)} \regbound_{\ucrl}(2^m; \dupp) \right)
    \end{align*}
    where $\calI_1', \ldots, \calI_\ell'$ are intervals that partition $[t_n, E_n]$ such that $\dev_{\calI_i'; \dupp} \leq \avgreg_{\ucrl}(|\calI_i'|; \dupp)$ for all $i$. 
\end{lemma}

Throughout this section, if infinite-horizon RL is considered, $\avgreg(\cdot)\triangleq \avgreg_{\ucrl}(\cdot~; \dupp)$, $\multiavg(\cdot)\triangleq \multiavg_{\ucrl}(\cdot~; \dupp)$,  $\dev_{\calI}\triangleq \dev_{\calI; \dupp}=\dev^r_{\calI} + 2\dupp \dev^p_{\calI} + \dev^J_{\calI}$ with a fixed $\dupp$, and $f^\star_t\triangleq J^\star_t$. 



For the purpose of conducting analysis, we divide $[t_n, t_n+2^n-1]$ into consecutive intervals $\calI_1=[s_1, e_1], \calI_2=[s_2, e_2], \ldots, \calI_K=[s_K, e_K]$ ($s_1=t_n$, $e_{i}+1=s_{i+1}$, $e_K=t_n+2^n-1$) in a way such that for all $i$:  
\begin{align}
     \dev_{\calI_i} \leq \avgreg(|\calI_i|)   \label{eq: i smaller than alpha}
\end{align}
One simple way to divide the intervals is to let $\Delta_{\calI_i}=0$ in each $\calI_i$. Then the number of intervals $K$ would be upper bounded by the number of stationary intervals within $[t_n, t_n+2^n-1]$.  Intuitively, the number of intervals can also be related to $\Delta_{[t_n, t_n+2^n-1]}$. 
We defer the calculation of the required number of intervals to \pref{lemma: interval divide}. For now, we only need the fact that the partition satisfies \pref{eq: i smaller than alpha}. From a high level, this partition makes the distribution in each interval close to stationary. Notice that this partition is independent of the learner's behavior in block $n$. 

For convenience, we further define the following quantities that depend on the learner's behavior in block $n$: 
\begin{definition}
    Define $E_n$ as the index of the last round in block $n$. Since the block might terminate earlier than planned, we have $E_n \leq t_n+2^n -1$. 
    Let $\ell\in[K]$ be such that $E_n\in\calI_\ell$ (that is, $\ell$ is the index of the interval where block $n$ ends).  Define $e_i'=\min\{e_i, E_n\}$ and $\calI_i'=[s_i, e_i']$ (therefore, $\calI_i'=\emptyset$ for $i>\ell$). 
\end{definition}

Recall the definition of $\nmax$ and $\multiavg(t)$ from \pref{lemma: multi-scale reg} (or \pref{lemma: ucrl aggregated regret}).
For simplicity, we define $\gap_m\triangleq \avgreg(2^m)$, $\multigap_m\triangleq \multiavg(2^m)$, and also $\multireg(t) \triangleq t\multiavg(t)$.  
Furthermore, we define the following technical quantities. 

\begin{definition}
For every $i\in\{1,\ldots,K\}$, and every $m\in\left\{0, 1, \ldots, n\right\}$, define 
\begin{align*}
     \tau_i(m) = \min\left\{\tau\in \calI_i': ~ f_\tau^\star - \tildeg_\tau \geq 12\multigap_m \right\};
\end{align*}
that is, $\tau_i(m)$ is the first time $\tau$ in $\calI_i'=\calI_i\cap [t_n, E_n]$ such that $f_\tau^\star-\tildeg_\tau$ exceeds $12\multigap_m$. If such $\tau$ does not exist or $\calI_i'$ is empty, we let $\tau_i(m)=\infty$.

Besides, we define $\xi_i(m)=[e_i'-\tau_i(m)+1]_+$ where $[a]_+ = \max\{0, a\}$ (which is the length of the interval $[\tau_i(m), e_i']$ when $\tau_i(m)$ is not $\infty$).  

\end{definition}
The intuition for $\tau_i(m)$ and $\xi_i(m)$ is as follows. Suppose that block $n$ has not ended at $\tau$.
If there exists some $\tau\in\calI_i$ such that $f_\tau^\star - \tildeg_\tau \geq 12\multigap_m$ (which first happens at $\tau_i(m)$), and if $\calI_i$ is long enough (i.e., $\xi_i(m)$ is large enough) so that after $\tau_i(m)$, an order-$m$ instance of \alg can  run entirely within $\calI_i$, then the learner is able to discover the fact that $f_\tau^\star - \tildeg_\tau$ is large, and  then restart. This coincides with our explanation in \pref{fig: detection}. The derivation in this section will formalize this intuition.

\begin{lemma}\label{lemma: block regret hard term}
    Let the high-probability events described in \pref{lemma: multi-scale reg} (or \pref{lemma: ucrl aggregated regret}) hold. Then with high probability, 
    \begin{align*}
        \sum_{\tau=t_n}^{E_n}\left(\tildeg_\tau-R_\tau\right) &\leq 4\multireg(2^n), \\ 
        \sum_{\tau =t_n}^{E_n} \left(f_\tau^\star - \tildeg_\tau \right) &\leq 96\nmax\sum_{i=1}^{\ell} \multireg(|\calI_i'|) + 60\sum_{m=0}^n \frac{\gap_m}{\gap_n}\multireg(2^m)\log(T/\delta)
    \end{align*}
    (notations are defined at the beginning of this section).
\end{lemma}

\begin{proof}
     $\sum_{\tau=t_n}^{E_n}\left(\tildeg_\tau-R_\tau\right)$ is trivially upper bounded by $3\multireg(E_n-t_n+1) + 1\leq 4\multireg(2^n)$ because it is guarded by \testtwo. Below we focus on the second claim. 

     Note that we can write for all $i=1,\ldots, K$, 
\begin{align*}
    &\sum_{\tau\in\calI_i'} \left(f_\tau^\star - \tildeg_\tau\right) \\
    &\leq 12 \sum_{\tau\in\calI_i'} \left( \one
    \Big[  f_\tau^\star - \tildeg_\tau \leq 12\multigap_n \Big] \multigap_n 
     + \sum_{m=1}^{n}  \one\Big[ 12\multigap_{m} < f_\tau^\star - \tildeg_\tau \leq 12\multigap_{m-1} \Big] \multigap_{m-1} 
      +\one\Big[ f_\tau^\star - \tildeg_\tau > 12\multigap_0 \Big]1 \right)\\
    &\leq 12\left(|\calI_i'| \multigap_n + \sum_{m=1}^{n} \multigap_{m-1} \xi_i(m) + \avgreg(1)\xi_i(0)\right)   \tag{$\avgreg(1)\geq 1$ by \pref{assum:assump2}} \\
    &\leq 12|\calI_i'| \multigap_n + 24\sum_{m=0}^{n} \multigap_m \xi_i(m) \tag{$\multigap_m=\frac{\multireg(2^m)}{2^m}\leq \frac{\multireg(2^{m+1})}{2^m} = 2\multigap_{m+1}$}
\end{align*} 
where in the second-to-last inequality we use $\sum_{\tau\in\calI_i'} \one\big[f_\tau^\star - \tildeg_\tau \geq 12\multigap_m\big] = \sum_{\tau\in[\tau_i(m), e_i']} \one\big[f_\tau^\star - \tildeg_\tau \geq 12\multigap_m\big] \leq 
 \xi_i(m)$ by the definition of $\tau_i(m)$. 

Summing the above over intervals $i$ and notice that $\sum_{i=1}^\ell |\calI_i'|\leq 2^n$, we get 
\begin{align}
    \sum_{\tau=t_n}^{E_n}\left(f_\tau^\star - \tildeg_\tau\right) 
    &\leq 12\cdot 2^n\multigap_n + 24\sum_{m=0}^n  \sum_{i=1}^\ell \multigap_m\xi_i(m) = 12\multireg(2^n) + 24\sum_{m=0}^n  \sum_{i=1}^\ell \multigap_m\xi_i(m).   \label{eq: tmptmp}
\end{align} 
Next, we upper bound $ \sum_{i=1}^\ell \multigap_m\xi_i(m)$ for each $m$. 
\begin{align}
      \sum_{i=1}^\ell \multigap_m\xi_i(m)
     &= \sum_{i=1}^\ell \multigap_m \min\left\{\xi_i(m), 4\cdot 2^m\right\} + \sum_{i=1}^\ell \multigap_m \left[\xi_i(m)-4\cdot 2^m\right]_+.     \label{eq: bound two individual term}\\
     & \tag{using $a= \min\{a,b\} + [a-b]_+$}
\end{align}
The first term on the right-hand side of \pref{eq: bound two individual term} can be bounded as below: 
\begin{align*}
     \sum_{i=1}^\ell  \multigap_m \min\left\{\xi_i(m), 4\cdot 2^m\right\} 
     &\leq 4\sum_{i=1}^\ell  \multiavg(2^m) \times \min\left\{ \xi_i(m),  2^m\right\} \\
     &\leq 4\sum_{i=1}^\ell  \multiavg(\min\{\xi_i(m), 2^m\}) \times \min\left\{ \xi_i(m),  2^m\right\}  \tag{$\multiavg(\cdot)$ is a decreasing function} \\ 
     &= 4\sum_{i=1}^{\ell} \multireg(\min\{\xi_i(m), 2^m\}) \\
     &\leq 4\sum_{i=1}^{\ell} \multireg(|\calI_i'|).    \tag{$\multireg(\cdot)$ is an increasing function} 
\end{align*}
The second term on the right-hand side of \pref{eq: bound two individual term} is bounded using \pref{lemma: key term reg bound} below. Combining them into \pref{eq: tmptmp} finishes the proof. 
\end{proof}

\begin{lemma}
     \label{lemma: key term reg bound}
     Let the high probability events described in \pref{lemma: multi-scale reg} (or \pref{lemma: ucrl aggregated regret}) hold. Then with high probability, 
     \begin{align*}
         \sum_{i=1}^\ell \multigap_m \left[\xi_i(m) -4\cdot 2^m\right]_+ \leq \frac{2\gap_m}{\gap_n}\multireg(2^m)\log(T/\delta).  
     \end{align*}
\end{lemma}
\begin{proof}
     Using the fact that $[[a]_+ - b]_+ = [a-b]_+$ when $b\geq 0$, we have 
     \begin{align}
         [\xi_i(m) - 4\cdot 2^m]_+=\left[e_i'-\tau_i(m) +1 -4\cdot 2^m  \right]_+. \label{eq: simple relation 1}   
     \end{align}
     
     Next, we consider the following quantity: ``the number of rounds in the interval $[\tau_i(m), e_i'-2\cdot 2^m B]$ which are candidate starting points of an order-$m$ \alg''. By \pref{proc: alg profile}, this quantity can be written and lower bounded as 
     \begin{align*}
          A_i\triangleq \sum_{t\in\calI_i} \one\Big[ t \in [\tau_i(m),~ e_i'- 2\cdot 2^m], \quad  (t-t_n) \text{\ mod\ } 2^m = 0\Big] \geq  \frac{\left[e_i'-\tau_i(m) +1 - 4\cdot 2^m\right]_+}{2^m}
     \end{align*}
     where we use the fact in an interval of length $w$, there are at least $\frac{w+2-2u}{u}$ points whose indices are multiples of $u$. 
     Notice that the right-hand side is related to what we want to upper bound in the lemma according to \pref{eq: simple relation 1}. Thus we continue to upper bound the left-hand side above. We define the following events: 
     \begin{align*}
          W_{t} &= \left\{\tau_i(m) \leq t \leq  e_i-2\cdot 2^m \text{\ where\ }i\text{\ is such that\ } t\in\calI_i \right\}, \\
          X_t&= \left\{ t\leq E_n-2\cdot 2^m \right\}, \\
          Y_t&=\left\{ t\leq E_n \text{\ and\ } (t-t_n)\text{\ mod\ }2^m=0 \right\}, \\
          Z_t&= \left\{ \exists \text{\ order-$m$\ } \inst \text{\ such that\ } \inst.s=t \right\}, \\
          V_t&= \left\{ \exists \tau\in[t_n, t] \text{\ such that\ } W_{\tau}\cap Y_{\tau}\cap Z_{\tau} \right\}.
     \end{align*}
     Then we can write (recall the definition of $K$ in the beginning of this section)
     \begin{align*}
         \sum_{i=1}^\ell A_i  = \sum_{i=1}^K A_i 
         &= \sum_{t=t_n}^{t_n+2^n-1} \one[W_t, X_t, Y_t] \leq \underbrace{\sum_{t=t_n}^{t_n+2^n-1} \one[W_t, Y_t, \overline{V_t}]}_{\term_3} + \underbrace{\sum_{t=t_n}^{t_n+2^n-1}\one[X_t, V_t]}_{\term_4}  
     \end{align*}
     For $\term_3$, notice that conditioned on $W_t\cap Y_t$, the event $Z_t$ happens with a constant probability $\frac{\gap_n}{\gap_m}$ (by \pref{proc: alg profile}). Therefore, $\term_3$ counts the number of trials up to the first success in a repeated trial with success probabiliy $\frac{\gap_n}{\gap_m}$. Therefore, with probability $1-\frac{\delta}{T}$, $\term_3\leq 1+ \frac{\log(T/\delta)}{-\log\left(1-\frac{\gap_n}{\gap_m}\right)}  \leq \frac{2\gap_m}{\gap_n}\log(T/\delta)$. 
     
     Next, we deal with $\term_4$. Below we show that $\term_4=0$. The event $V_t$ implies that there exists some order-$m$ \inst which starts at $\inst.s=t^\star$, where $t^\star\leq t$ and $\tau_i(m) \leq t^\star\leq e_i-2\cdot 2^m $. Therefore, we have $\inst.e=\inst.s+2^{m}-1= t^\star + 2^m-1 \leq e_i-2^m-1 < e_i$, and thus $[\inst.s, \inst.e]\subseteq \calI_i$. Together with $X_t$, the event $V_t\cap X_t$ implies that $\inst.e = \inst.s + 2^m-1 \leq t+2^m -1 < E_n$, and therefore, and time $\inst.e$, block $n$ has not ended. 
     
     Since at time $\inst.e$, block $n$ is still on-going, the learner performs \testone. By \pref{lemma: multi-scale reg} (or \pref{lemma: ucrl aggregated regret} for the infinite-horizon RL case), with high probability, we have
    \begin{align*}
        \frac{1}{2^m}\sum_{\tau=\inst.s}^{\inst.e} R_\tau 
        &\geq \frac{1}{2^m } \sum_{\tau=\inst.s}^{\inst.e} \tildeg_\tau - \multigap_m - \nmax \dev_{[\inst.s, \inst.e]}  \tag{\pref{lemma: multi-scale reg} or \pref{lemma: ucrl aggregated regret}}\\
        &\geq  \min_{\tau\in \calI_i } f_\tau^\star - \multigap_m - (\nmax+1) \dev_{\calI_i} \tag{because $[\inst.s, \inst.e]\subseteq \calI_i$} \\
        &\geq  f^\star_{\tau_i(m)}  - \multigap_m - (\nmax+3) \dev_{\calI_i}  \tag{$| \min_{\tau\in \calI_i } f_\tau^\star - f^\star_{\tau_i(m)}| \leq 2\dev_{\calI_i}$} \\
        &\geq  \tildeg_{\tau_i(m)} +12\multigap_m   -  2\multigap_m  \tag{by the definition of $\tau_i(m)$ and $\dev_{\calI_i}\leq \avgreg(|\calI_i|) \leq \avgreg(2^m)\leq \frac{\multigap_m}{6\nmax}$} \\
        &\geq  U_{\inst.e} + 10\multigap_m  \tag{Because $\inst.e\geq \tau_i(m)$, $U_{\inst.e}\leq \tildeg_{\tau_i(m)}$ by the algorithm}
    \end{align*}
    This should trigger the restart at time $\inst.e < E_n$, contradicting the definition of $E_n$. Therefore, $\one[X_t, V_t]=0$. 
    
    Finally, combining all previous arguments, we have that with high probability, 
    \begin{align*}
         \sum_{i=1}^\ell \multigap_m \left[\xi_i(m)-4\cdot 2^m\right]_+  &= \sum_{i=1}^\ell \multigap_m \left[e_i'-\tau_i(m)+1-4\cdot 2^m\right]_+  \leq \multigap_m  2^m\sum_{i=1}^\ell A_i \\
         &= \multireg(2^m) \sum_{i=1}^\ell A_i \leq \frac{2\gap_m}{\gap_n}\multireg(2^m)\log(T/\delta),
    \end{align*}
    finishing the proof.
\end{proof}

\section{Single-block Regret Analysis II (under a Special Form of $\regbound(\cdot)$)}
\label{app: epoch regret}
In \pref{app: key proofs}, we have derived the regret bound in a single block for both the standard setting and the infinite-horizon MDP setting (\pref{lemma: bound term1} and \pref{lemma: block lemma for RL}). They are both of the form 
\begin{align}
     \sum_{\tau\in\calJ}(f_\tau^\star - R_\tau) = \otil\left( \sum_{i=1}^\ell \regbound(|\calI_i'|) + \sum_{m=0}^n \frac{\avgreg(2^n)}{\avgreg(2^m)}\regbound(2^m) \right).    \label{eq: block regret restate}
\end{align}
(replacing $\regbound(\cdot)$ and $\avgreg(\cdot)$ by $\regbound_{\ucrl}(\cdot; \dupp)$ and $\avgreg(\cdot; \dupp)$ for the case of infinite-horizon MDP). 

In this section, we further derive more concrete dynamic regret bounds for both cases by assuming that $\regbound(\cdot)$ is of some specific form. The form of $\regbound(\cdot)$ we consider in this section is defined as follows: 

\begin{definition}
\label{def: standard}
     We define a form of $\regbound(t)$ as $\regbound(t)=\min\{c_1 t^p+c_2, c_3t\}$ for some $p\in[\frac{1}{2}, 1)$ and some $c_1$, $c_2$, $c_3$ ($c_3\geq 1$) that capture dependencies on $\log(T/\delta)$ and other problem-dependent constants. 
\end{definition}
In fact, usually, a regret bound is only written in the form of $c_1 t^p + c_2$. However, since the reward is bounded between $0$ and $1$, the regret bound of $\min\{c_1 t^p + c_2, t\}$ is also trivially correct. \pref{def: standard} is slightly more general than this by allowing a coefficient $c_3\geq 1$ (the regret bound would still be trivially correct). In some cases, we make $c_1, c_2, c_3$ larger than their tightest possible values to make the final regret bound better --- notice that the choice of $c_1, c_2, c_3$ affects the probability specified in \pref{proc: alg profile}, and thus smaller $c_1, c_2, c_3$ does not necessarily make the final regret bound smaller. This subtle issue can be observed from the analysis. 


To get a concrete bound, we also need to decide the number $\ell$ in the single-block regret bound above. In \pref{app: key proofs}, we have stated the condition (i.e., \pref{eq: i smaller than alpha}) that should be satisfied by $\calI_1', \ldots, \calI_\ell'$ (or $\calI_1, \ldots, \calI_K$). In the next lemma, we upper bound the value of $\ell$ that is required to fulfill the condition. 

\begin{lemma}\label{lemma: interval divide}
     Let $\calJ=[t_n, E_n]$. Then we have $\ell \leq L_{\calJ}$. Furthermore, if $\regbound(t)$ is in the form specified in \pref{def: standard}, we also have $\ell\leq 1 + 2\left(c_1^{-1}\dev_{\calJ}|\calJ|^{1-p}\right)^{\frac{1}{2-p}} + c_3^{-1}\dev_\calJ$. 
\end{lemma}

\begin{proof}
     The fact that $\ell\leq L_{\calJ}$ is straightforward to see (and has been explained in \pref{app: key proofs}):  to satisfy the condition \pref{eq: i smaller than alpha}, one way to divide the block is to make each $\calI_i$ a stationary interval, which makes $\dev_{\calI_i}=0$ for all $i\in[K]$. This way of division leads to $\ell \leq L_{\calJ}$. 
      
      For the second claim, we follow the same procedure as decribed in the proof of Lemma 5 in \citep{chen2019new}. Basically, the procedure divides $[t_n, t_n+2^n-1]$ in a greedy way, making all $\calI_i=[s_i, e_i]$ satisfy $\dev_{[s_i, e_i]} \leq \avgreg(e_i-s_i+1)$ and $\dev_{[s_i, e_i+1]} > \avgreg(e_i-s_i+2)$ for all $i\in[K-1]$ (i.e., except for the last interval). Then we have 
\begin{align*}
     \dev_{\calJ} &\geq \sum_{i=1}^{\ell-1} \dev_{[s_i, e_i+1]} \tag{by the definition of $\dev_{[\cdot,\cdot]}$} \\
     &> \sum_{i=1}^{\ell-1} \avgreg(e_i-s_i+2) \\
     &\geq \sum_{i=1}^{\ell-1} \min\left\{c_1 (e_i-s_i+2)^{p-1},  c_3\right\}    \tag{by \pref{def: standard}} \\
     &\geq \sum_{i=1}^{\ell-1} \min\left\{\frac{1}{2}c_1 (e_i-s_i+1)^{p-1}, c_3\right\} \tag{$(x+2)^{p-1} \geq \left(2(x+1)\right)^{p-1}\geq \frac{1}{2}(x+1)^{p-1}$ for any $x\geq 0$ and $p\leq 1$}  \\
     &= \frac{1}{2}\sum_{i=1}^{\ell_1} c_1 (e_i-s_i+1)^{p-1} + \sum_{i=1}^{\ell_2} c_3
\end{align*}
where in the last equality we separate the intervals where $\min\left\{\frac{1}{2}c_1 (e_i-s_i+1)^{p-1}, c_3\right\}$ takes the former or the latter value. Note that $\ell_1+\ell_2=\ell-1$. 

The above inequality implies that $\dev_{\calJ} $ upper bounds both $\frac{1}{2}\sum_{i=1}^{\ell_1} c_1 (e_i-s_i+1)^{p-1} $ and $\sum_{i=1}^{\ell_2} c_3 $. Thus, $\ell_2\leq c_3^{-1}\dev_{\calJ}$, and 
by H\"{o}lder's inequality, 
\begin{align*}
     \ell_1 &\leq \left(\sum_{i=1}^{\ell_1}  (e_i-s_i+1)^{p-1} \right)^{\frac{1}{2-p}} \left(\sum_{i=1}^{\ell_1} (e_i-s_i+1) \right)^{\frac{1-p}{2-p}} \leq \left(\frac{2\dev_{\calJ}}{c_1}\right)^{\frac{1}{2-p}}|\calJ|^{\frac{1-p}{2-p}}. 
\end{align*}
Combining them finishes the proof. 
\end{proof}

In the following \pref{lemma: single block regret form}, we bound the regret within a block by combining \pref{eq: block regret restate} and \pref{lemma: interval divide}. 
We will frequently use the following two properties: let $\{\calS_1, \calS_2, \ldots, \calS_K\}$ be a partition of the interval $\calS$. Then
    \begin{align}
         \sum_{i=1}^K L_{\calS_i} &\leq L_{\calS} + (K-1),  \label{eq: L partition}  \\
        \sum_{i=1}^K \dev_{\calS_i} &\leq \dev_{\calS}.  \label{eq: dev partition}
    \end{align}  
    They can be derived using the definitions of $L_{[\cdot,\cdot]}$ and $\dev_{[\cdot,\cdot]}$.

\begin{lemma}
\label{lemma: single block regret form}
    If $\regbound(t)$ is of the form specified in \pref{def: standard}, then 
    \begin{align*}
         \sum_{\tau=t_n}^{E_n} (f_\tau^\star - R_\tau)  \leq  \otil\left(  \min\Big\{\Reg_L(\calJ), \Reg_{\dev}(\calJ)\Big\}+  c_12^{np} + \frac{c_2c_3}{c_1}2^{n(1-p)} + \frac{c_2^2}{c_3} \right), 
    \end{align*}
    where $\Reg_L(\calJ) \triangleq c_1 L_\calJ^{1-p} |\calJ|^{p} + c_2L_\calJ$ and
    \begin{align*}
        \Reg_{\dev}(\calJ) &\triangleq \left(c_1 \dev_{\calJ}^{1-p}|\calJ|\right)^{\frac{1}{2-p}}  + c_1|\calJ|^p + c_1(c_3^{-1}\dev_{\calJ})^{1-p} |\calJ|^{p} + c_2\left(c_1^{-1}\dev_{\calJ}|\calJ|^{1-p}\right)^{\frac{1}{2-p}} + c_2 + c_2c_3^{-1}\dev_\calJ.
    \end{align*}
\end{lemma}
\begin{proof}
 We bound each term in \pref{eq: block regret restate} using \pref{def: standard}. First, notice that
    \begin{align}
        \otil\left( \sum_{i=1}^\ell \regbound(|\calI_i'|) \right) 
        &= \otil\left( \sum_{i=1}^\ell \min\left\{c_1|\calI_i'|^p + c_2, \ c_3t \right\} \right)   \nonumber \\
        &\leq \otil\left(\sum_{i=1}^\ell (c_1|\calI_i'|^p + c_2) \right) 
        \leq \otil\left(c_1 \ell^{1-p} |\calJ|^{p} + c_2\ell\right).   \label{eq: regret first term}
    \end{align}
    Using the first upper bound for $\ell$ given in \pref{lemma: interval divide}, \pref{eq: regret first term} can be bounded by $        \otil\left(\Reg_L(\calJ)\right) $;
    using the second upper bound, \pref{eq: regret first term} can be bounded by 
    $\otil\left(\Reg_{\dev}(\calJ)\right)$. 
    Next, we have
    \begin{align*}
         &\otil\left(\frac{\avgreg(2^m)}{\avgreg(2^n)}\regbound(2^m)\right)
         = \otil\left(c_1 2^{np} + \frac{c_2c_3}{c_1}2^{n(1-p)} + \frac{c_1^2}{c_3} 2^{m(2p-1)} + \frac{c_2^2}{c_3} 2^{-m} \right). 
    \end{align*}
    by \pref{lem: alpha calculation} below. 
    Notice that because $c_3\geq 1$ and $p\geq \frac{1}{2}$, $\frac{c_1^2}{c_3} 2^{m(2p-1)} \leq c_1^2 2^{n(2p-1)} \leq c_12^{np}$ when $c_1\leq 2^{n(1-p)}$. This is indeed the regime we care about since if $c_1 > 2^{n(1-p)}$ then the first term $c_1 2^{np} > 2^n$, which is a vacuous bound for the regret of block $n$. 
    Therefore, we can drop this term. Thus, the dynamic regret in block $n$ can be summarized as the following based on \pref{eq: block regret restate}: 
    \begin{align}
         \otil\left(  \min\Big\{\Reg_L(\calJ), \Reg_{\dev}(\calJ)\Big\}+  c_12^{np} + \frac{c_2c_3}{c_1}2^{n(1-p)} + \frac{c_2^2}{c_3} \right),
    \end{align}
    finishing the proof.
\end{proof}

\begin{lemma}\label{lem: alpha calculation}
     Let $\regbound(t)$ be of the form in \pref{def: standard}. Then 
     \begin{align*}
     \frac{\avgreg(2^m)}{\avgreg(2^n)}\regbound(2^m)=\order\left(c_1 2^{np} + \frac{c_2c_3}{c_1}2^{n(1-p)} + \frac{c_1^2}{c_3} 2^{m(2p-1)} + \frac{c_2^2}{c_3} 2^{-m} \right). 
     \end{align*}
\end{lemma}
\begin{proof}
This is by direct calculation:
     \begin{align*}
        \frac{\avgreg(2^m)}{\avgreg(2^n)}\regbound(2^m)
          &= \frac{\regbound(2^m)^2 }{\regbound(2^n)} 2^{n-m}\\
         &=\order\left(  \frac{\min\{c_1^2 2^{2mp} + c_2^2,\ \  c_3^2 2^{2m} \}}{c_12^{np} + c_2}2^{n-m}+\frac{\min\{c_1^2 2^{2mp} + c_2^2,\ \  c_3^2 2^{2m} \}}{c_3 2^n}2^{n-m}\right) \\
         &=\order\left(  \min\left\{c_1 2^{np}2^{(n-m)(1-2p)} + \frac{c_2^2}{c_1}2^{n(1-p)-m}, \ \  \frac{c_3^2}{c_1}2^{n(1-p)+m}  \right\} + \frac{c_1^2}{c_3} 2^{m(2p-1)} + \frac{c_2^2}{c_3} 2^{-m} \right) \\
         &= \order\left(c_1 2^{np}+ \min\left\{\frac{c_2^2}{c_1}2^{n(1-p)-m}, \ \  \frac{c_3^2}{c_1}2^{n(1-p)+m}  \right\}+ \frac{c_1^2}{c_3} 2^{m(2p-1)} + \frac{c_2^2}{c_3} 2^{-m} \right) \\
         &= \order\left(c_1 2^{np} + \frac{c_2c_3}{c_1}2^{n(1-p)} + \frac{c_1^2}{c_3} 2^{m(2p-1)} + \frac{c_2^2}{c_3} 2^{-m} \right). 
    \end{align*}
\end{proof}

\section{Single-epoch Regret Analysis}\label{app: epoch regret seperate}
We call $[t_0, E]$ an \emph{epoch} if $t_0$ is the first step after restart (or $t_0=1$), and $E$ is the first time after round $t_0$ when the restart is triggered. 
In this section, we continue the discussion in \pref{app: epoch regret} and bound the regret in a single epoch. Recall that the we consider cases where the single-block regret can be written as \pref{eq: block regret restate} and $\regbound(\cdot)$ is in the form of \pref{def: standard}. This holds both for the case of the standard setting and the infinite-horizon MDP setting. 

\begin{lemma}\label{lemma: sum epoch regret}
     Let $\calE$ be an epoch. Then 
     \begin{align*}
          \sum_{\tau\in\calE}\leq \otil\left(\min\big\{  \Reg_L(\calE), \Reg_{\dev}(\calE) \big\} + \frac{c_2c_3}{c_1}|\calE|^{1-p} + \frac{c_2^2}{c_3}\right) 
     \end{align*}
     ($\Reg_L(\cdot)$ and $\Reg_\dev(\cdot)$ are defined in \pref{lemma: single block regret form})
\end{lemma}

\begin{proof}
     Let $\calE$ be an epoch whose last block is indexed by $n$. Then $|\calE|=\Theta(2^n)$. Let $\calJ_1, \ldots, \calJ_n$ be blocks in $\calE$. Then by \pref{lemma: single block regret form}, the dynamic regret in $\calE$ is upper bounded by 
    \begin{align*}
         &\otil\left(\min\left\{ \sum_{m=0}^n \Reg_L(\calJ_m), \ \sum_{m=0}^n \Reg_\dev(\calJ_m)  \right\}+ c_1\sum_{m=0}^n 2^{mp} + \frac{c_2c_3}{c_1}\sum_{m=0}^n 2^{m(1-p)} + \sum_{m=0}^n \frac{c_2^2}{c_3}\right). 
    \end{align*}
    By H\"{o}lder's inequality, 
    \begin{align*}
          \sum_{m=0}^n \Reg_L(\calJ_m) 
          &= c_1\left(\sum_{m=0}^n L_{\calJ_m} \right)^{1-p}\left(\sum_{m=0}^n  |\calJ_m|\right)^p + c_2 \sum_{m=0}^n  L_{\calJ_m} \\
          &\leq c_1 \left(L_{\calE} + n \right)^{1-p} |\calE|^p + c_2\left(L_{\calE} + n \right)    \tag{using \pref{eq: L partition}}\\
          &\leq \otil\left(c_1 L_\calE^{1-p}  |\calE|^{p} + c_2 L_\calE \right) 
           = \otil\left(\Reg_L(\calE)\right)    \tag{because $n=\order(\log T)=\otil(1)$}
    \end{align*}
    Similarly, $\sum_{m=0}^n \Reg_{\dev}(\calJ_m) = \otil\left(\Reg_{\dev}(\calE)\right)$. 
    On the other hand, $c_1\sum_{m=0}^n 2^{mp} + \frac{c_2c_3}{c_1}\sum_{m=0}^n 2^{m(1-p)} + \sum_{m=0}^n \frac{c_2^2}{c_3} = \otil\left(c_1 2^{np} + \frac{c_2c_3}{c_1} 2^{n(1-p)} + \frac{c_2^2}{c_3}\right) = \otil\left(c_1 |\calE|^p + \frac{c_2c_3}{c_1} |\calE|^{1-p} + \frac{c_2^2}{c_3} \right)$. In summary, the dynamic regret within an epoch is of order 
    \begin{align}
         \otil\left(\min\big\{  \Reg_L(\calE), \Reg_{\dev}(\calE) \big\} + \frac{c_2c_3}{c_1}|\calE|^{1-p} + \frac{c_2^2}{c_3}\right)  \label{eq: epoch regret}
    \end{align}
    (the $c_1|\calE|^p$ term is absorbed into $\min\left\{  \Reg_L(\calE), \Reg_{\dev}(\calE) \right\}$). 
\end{proof}



\section{Proof of \pref{thm: regret bound}} \label{app: omitted proof}

We are now ready to prove \pref{thm: regret bound} after showing the following two lemmas.

\begin{lemma}\label{lem: infreq restart}
     Let $t$ be in an epoch starting from $t_0$. If $\dev_{[t_0, t]}\leq \avgreg(t-t_0+1)$, then with high probability, no restart would be triggered at time $t$. 
\end{lemma}
\begin{proof}   
We first verify that \testone would not fail with high probability. Let $t=\inst.e$ where $\inst$ is any order-$m$ \alg in block $n$. Then with high probability, 
\begin{align*}
    U_t 
    &= \min_{\tau\in [t_n, t]} \tildeg_\tau \\
    &\geq \min_{\tau\in[t_n, t]} f^\star_\tau - \dev_{[t_n, t]}  \tag{by \pref{lemma: multi-scale reg}}  \\
    &\geq \frac{1}{2^m}\sum_{\tau\in[\inst.s, t]} f^\star_\tau - 3\dev_{[t_n, t]}  \tag{$[\inst.s, t]\subseteq [t_n, t]$} \\
    &\geq \frac{1}{2^m} \sum_{\tau\in[\inst.s, t]} R_\tau - 2\sqrt{\frac{\log(T/\delta)}{2^m}} - 3\avgreg(t-t_0+1)     \tag{$\E[R_\tau]=\E[f_\tau(\pi_t)]\leq f_\tau^\star$ and we use Azuma's inequality}\\
    &\geq \frac{1}{2^m} \sum_{\tau\in[\inst.s, t]} R_\tau - \multiavg(2^m) - 3\avgreg(t-t_0+1)  \tag{By \pref{assum:assump2}, $\multiavg(2^m)\geq 6\log(T/\delta)\avgreg(2^m)\geq 6\log(T/\delta)\sqrt{\frac{1}{2^m}}$}\\
    &\geq \frac{1}{2^m} \sum_{\tau\in[\inst.s, t]} R_\tau - 2\multiavg(2^m).  \tag{$\avgreg(t-t_0+1)\leq \avgreg(2^m)$ because $\avgreg(\cdot)$ is decreasing}
\end{align*}
So with high probability, \testone will not return fail. 

Furthermore, by \pref{lemma: multi-scale reg}, with high probability, 
\begin{align*}
     \frac{1}{t-t_n+1}\sum_{\tau=t_n}^t \left(\tildeg_\tau - R_\tau\right)\leq \multiavg(t-t_n+1) + \dev_{[t_n, t]} \leq 2\multiavg(t-t_n+1). 
\end{align*}
Therefore, with high probability, \testtwo will not return fail either. 
\end{proof}

\begin{lemma}
      \label{lemma: epoch upper}
      With high probability, the number of epochs is upper bounded by $L$. If $\regbound(\cdot)$ is in the form of \pref{def: standard}, the number of epochs is also upper bounded by $1+2\left(c_1^{-1}\dev T^{1-p}\right)^{\frac{1}{2-p}} + c_3^{-1}\dev$. 
\end{lemma}

\begin{proof}
     By \pref{lem: infreq restart}, if $[t_0, E]$ is not the last epoch, then $\dev_{[t_0, E]}> \avgreg(E-t_0+1)$ with high probability. Then following the exact same arguments as in \pref{lemma: interval divide} proves the lemma. 
\end{proof}

\begin{proof}[Proof of \pref{thm: regret bound}]\ \ If $\regbound(t)=c_1t^p + c_2$ satisfies \pref{assum:assump2}, then $\regbound(t)=\min\{c_1 t^p + c_2, t\}$ also satisfies it (since the reward is bounded in $[0, 1]$). Below we use $\regbound(t)=\min\{c_1 t^p + c_2, t\}$ as the input to our algorithm. Notice that this is in the form of \pref{def: standard} with $c_3=1$.   Let $\calE_1, \ldots, \calE_N$ be epochs in $[1, T]$. Then by \pref{lemma: sum epoch regret}, the dynamic regret in $[1, T]$ is upper bounded by 
    \begin{align}
         &\otil\left( \min\left\{ \sum_{i=1}^N \Reg_L(\calE_i), \sum_{i=1}^N \Reg_\dev(\calE_i)  \right\} + \frac{c_2}{c_1}\sum_{i=1}^N |\calE_i|^{1-p} + c_2^2 N \right).    \label{eq: epoch regret bound ha}
    \end{align} 
    By H\"{o}lder's inequality and \pref{eq: L partition}, 
    \begin{align*}
         \sum_{i=1}^N \Reg_L(\calE_i) \leq \otil\left( c_1 \left( L + N-1 \right)^{1-p} T^p + c_2 (L+N-1)\right) \leq \otil\left(c_1L^{1-p}T^{p} + c_2L\right), 
    \end{align*}
    where in the last inequality we use \pref{lemma: epoch upper} to bound $N$. 

    Similarly, 
    \begin{align*}
         &\sum_{i=1}^N \Reg_\dev(\calE_i)   \\
         &\leq  \otil\left( \left(c_1 \dev^{1-p}T\right)^{\frac{1}{2-p}} + c_1N^{1-p}T^p + c_1\dev^{1-p} T^{p} +  c_2\left(c_1^{-1}\dev T^{1-p}\right)^{\frac{1}{2-p}} + c_2 N + c_2\dev\right) \\
         &\leq  \otil\left( \left(c_1 \dev^{1-p}T\right)^{\frac{1}{2-p}} +  c_1T^p +c_1\dev^{1-p} T^{p} + c_2\left(c_1^{-1}\dev T^{1-p}\right)^{\frac{1}{2-p}} + c_2 + c_2\dev\right).   \tag{using \pref{lemma: epoch upper} to bound $N$}
    \end{align*}
    Then we deal with the second term in \pref{eq: epoch regret bound ha}:  
    \begin{align*}
         \frac{c_2}{c_1} \sum_{i=1}^N |\calE_i|^{1-p} \leq \frac{c_2}{c_1} N^{p}T^{1-p},   
    \end{align*}
    which can be either bounded by $\otil\left(\frac{c_2}{c_1} L^p T^{1-p} \right)$ or 
    \begin{align*}
         \otil\left(\frac{c_2}{c_1}T^{1-p} +  \frac{c_2}{c_1}\left(c_1^{-p}\dev^{p}T^{2-2p}\right)^{\frac{1}{2-p}} + \frac{c_2}{c_1}\dev^p T^{1-p}\right)
    \end{align*}
    using the upper bound for $N$ in \pref{lemma: epoch upper}. Finally, the third term in \pref{eq: epoch regret bound ha} can be upper bounded either by $\otil\left(c_2^2 L\right)$ or
    \begin{align*}
         \otil\left(c_2^2 +  c_2^2\left(c_1^{-1}\dev T^{1-p}\right)^{\frac{1}{2-p}} + c_2^2\dev \right). 
    \end{align*}

    With all terms expanded, below, we collect the dominant terms for the cases of $p=\frac{1}{2}$ and $p>\frac{1}{2}$.  We say term $a(T)$ is dominated by $b(T)$ if $\lim_{T\rightarrow \infty} a(T)/b(T)=0$ under any sublinear growth rate of $L$ or $\dev$ (e.g., $\sqrt{\dev T}$ is dominated by $\dev^{\nicefrac{1}{3}}T^{\nicefrac{2}{3}}$ and $L$ is dominated by $\sqrt{LT}$). And below we only write down terms that are not dominated by other terms. 
    
    \paragraph{The case for $p=\frac{1}{2}$:} 
    \begin{align*}
    \otil\left( \min\left\{\left(c_1 + \frac{c_2}{c_1}\right)\sqrt{LT}, \quad  \left(c_1^{\nicefrac{2}{3}} + c_2c_1^{-\nicefrac{4}{3}}\right)\dev^{\nicefrac{1}{3}}T^{\nicefrac{2}{3}} +  \left(c_1 + \frac{c_2}{c_1}\right)\sqrt{T} \right\} \right) ; 
    \end{align*}
      
    \paragraph{The case for $p>\frac{1}{2}$:}
    \begin{align*}
         \otil\left(\min\left\{c_1 L^{1-p} T^{p}, \quad  \left(c_1\dev^{1-p}T\right)^{\frac{1}{2-p}} + c_1T^{p}\right\}\right).
    \end{align*}
    This finishes the proof.
\end{proof}


\section{Main Results for Infinite-horizon MDP}\label{app: omit proof for RL avg}

\begin{lemma}[\textit{c.f.} \pref{lem: infreq restart}]
\label{lemma: infrequent RL restart}
Let $t$ be in an epoch started from round $t_0$. If $\dev_{[t_0, t]; \dupp} < \dupp S\sqrt{\frac{A}{t-t_0+1}}$ and $\dupp\geq D_{\max}$, then with high probability, no restart will be triggered at time $t$. 
\end{lemma}

\begin{proof}
To verify that \testone will not fail with high probability, we follow very similar steps as in \pref{lem: infreq restart}.  Let $t=\inst.e$ where $\inst$ is an order-$m$ \alg in block $n$. Then with high probability (the following calculation is same as that in the proof of \pref{lem: infreq restart} except for the third inequality), 
\begin{align*}
    U_t 
    &= \min_{\tau\in [t_n, t]} \tildeg_\tau \\
    &\geq \min_{\tau\in[t_n, t]} J^\star_\tau - \dev_{[t_n, t]; \dupp}  \tag{by \pref{lemma: ucrl aggregated regret}}  \\
    &\geq \frac{1}{2^m}\sum_{\tau\in[\inst.s, t]} J^\star_\tau - 3\dev_{[t_n, t]; \dupp}  \tag{$[\inst.s, t]\subseteq [t_0, t]$} \\
    &\geq \frac{1}{2^m} \sum_{\tau\in[\inst.s, t]} R_\tau - 4\dupp \sqrt{\frac{\log(T/\delta)}{2^m}} - 4\dev_{[t_n, t]; \dupp}    \tag{explained below}\\
    &\geq \frac{1}{2^m} \sum_{\tau\in[\inst.s, t]} R_\tau - 4\multiavg_\ucrl(2^m; \dupp) - 4\avgreg_\ucrl(t-t_0+1; \dupp)  \\
    &\geq \frac{1}{2^m} \sum_{\tau\in[\inst.s, t]} R_\tau - 8\multiavg_\ucrl(2^m; \dupp).  \tag{$\avgreg(t-t_0+1)\leq \avgreg(2^m)$ because $\avgreg(\cdot)$ is decreasing}
\end{align*}
where the third inequality is based on the following calculation:  for all $\tau\in[\inst.s, t]$, 
\begin{align*}
     J_t^\star 
     &= r_t(s_\tau, a_\tau) + \sum_{s'} p_t(s'|s_\tau, a_\tau) h_t^\star(s') - h_t^\star(s_\tau) \\
     &\geq r_\tau(s_\tau, a_\tau) + \sum_{s'} p_\tau(s'|s_\tau, a_\tau) h_t^\star(s') - h_t^\star(s_\tau) - \left(\dev^r_{[\inst.s, t]} + D_{\max} \dev^p_{[\inst.s, t]}\right) \\
     &\geq r_\tau(s_\tau, a_\tau) + \sum_{s'} p_\tau(s'|s_\tau, a_\tau) h_t^\star(s') - h_t^\star(s_\tau) - \left(\dev^r_{[\inst.s, t]} + \dupp \dev^p_{[\inst.s, t]}\right)   \tag{by the assumption $\dupp\geq D_{\max}$}
\end{align*}
and thus 
\begin{align*}
     \frac{1}{2^m}\sum_{\tau\in[\inst.s, t]} J^\star_\tau 
     &\geq J^\star_t  - \dev_{[\inst.s, t]}^J \\
     &\geq \frac{1}{2^m} \sum_{\tau\in[\inst.s, t]} \left(r_\tau(s_\tau, a_\tau) + \sum_{s'} p_\tau(s'|s_\tau, a_\tau) h_t^\star(s') - h_t^\star(s_\tau)\right) - \dev_{[\inst.s, t]; \dupp}   \\
     &\geq \frac{1}{2^m}\sum_{\tau\in[\inst.s, t]} \Big(R_\tau + h_t^\star(s_{\tau+1}) - h_t^\star(s_\tau)\Big) -  2D_{\max}\sqrt{\frac{2\log(SAT/\delta)}{2^m}}   - \dev_{[\inst.s, t]; \dupp} \tag{Azuma's inequality} \\
     &\geq \frac{1}{2^m}\sum_{\tau\in[\inst.s, t]} R_\tau - 4\dupp\sqrt{\frac{\log(SAT/\delta)}{2^m}}  - \dev_{[\inst.s, t]; \dupp}.   \tag{$D_{\max}\leq \dupp$}
\end{align*}
So with high probability, \testone will not return fail. 

Furthremore, by \pref{lemma: ucrl aggregated regret}, with high probability, 
\begin{align*}
     \frac{1}{t-t_n+1}\sum_{\tau=t_n}^t \left(\tildeg_\tau - R_\tau\right)\leq \multiavg_\ucrl(t-t_n+1; \dupp) + \dev_{[t_n, t]; \dupp} \leq 2\multiavg_{\ucrl}(t-t_n+1; \dupp)
\end{align*}
where the last inequality is by the condition on $\dev_{[t_0, t], \dupp}$. 
Therefore, with high probability, \testtwo will not return fail either. 

It remains to show that the \acw will not terminate and call for restart under the specified condition. By \pref{lemma: bound diamete}, if $\calP_k^\eta$ contains an MDP whose diameter is upper bounded by $\dupp$, then the span of the output bias vector is upper bounded by $2\dupp$, and then the if-statement in \pref{line: enough widening} of \pref{alg: base alg UCRL} will be triggered.  Therefore, to show that \acw will not terminate, we upper bound the $\eta_k$ that needs to be added to $\calP_k$ in order to make at least one true MDP (whose diameter is upper bounded by $D_{\max}\leq \dupp$) lie in $\calP_k^{\eta_k}$. Then we further argue that $\sum_{\tau=t_0}^{t} \eta_{k(\tau)}$ is not large enough to reach the condition in \pref{line: terminate condition} of \pref{alg: base alg UCRL}. 


For all episode $k$ that starts before $t$, by Azuma's inequality, 
\begin{align*}
     \left\|\barp_k(\cdot|s,a) - \hatp_k(\cdot|s,a)\right\|_1 \leq 2\sqrt{\frac{S\log(1/\delta)}{N_k^+(s,a)}}. 
\end{align*}
By the condition on $\dev_{[t_0,t]; \dupp}$, we have 
\begin{align*}
     \left\|p_{t_0}(\cdot|s,a) - \barp_k(\cdot|s,a) \right\|_1 \leq \frac{1}{\dupp}\dev_{[t_0,t]; \dupp} \leq S\sqrt{\frac{A}{t-t_0+1}}.
\end{align*}
Combining them, we get 
\begin{align*}
     \left\|p_{t_0}(\cdot|s,a) - \hatp_k(\cdot|s,a)\right\|_1 \leq 2\sqrt{\frac{S\log(1/\delta)}{N_k^+(s,a)}} + S\sqrt{\frac{A}{t-t_0+1}}. 
\end{align*}
Therefore, we see that in \pref{line: terminate condition} of \pref{alg: base alg UCRL}, as long as $\eta \geq S\sqrt{\frac{A}{t-t_0+1}}$, $p_{t_0}$ is contained in $\calP^\eta_k$. Then we have $\spn(\tildeh)\leq 2D_{\max}\leq 2\dupp$ by \pref{lemma: bound diamete}, and the for-loop will be broken at this $\eta$. 

Thus we conclude that $\eta_k\leq 2S\sqrt{\frac{A}{t-t_0+1}}$ for all episode $k$ started before $t$. Thus, $\sum_{\tau=t_0}^t \eta_{k(\tau)} \leq (t-t_0+1)\times 2S\sqrt{\frac{A}{t-t_0+1}}=2S\sqrt{A(t-t_0+1)}$, and thus the algorithm will not terminate and call for restart at time $t$. 
\end{proof}

\begin{lemma}[\textit{c.f.} \pref{lemma: epoch upper}]
\label{lemma: RL number epoch}
If $\dupp\geq D_{\max}$, then the number of epochs is upper bounded by $\min\left\{L,\ \ 1+ 3\left(\frac{\Delta^r + \Delta^p}{S\sqrt{A}}\right)^{\frac{2}{3}} T^{\frac{1}{3}}\right\}$. 
\end{lemma}

\begin{proof}
    Let $\calE_1, \ldots, \calE_N$ be the epochs. By \pref{lemma: infrequent RL restart}, for $i\leq N-1$, we must have $\dev_{\calE_i; \dupp}\geq \dupp S\sqrt{\frac{A}{|\calE_i|}}$. By H\"{o}lder's inequality, 
    \begin{align*}
         N-1 
         &\leq \left(\sum_{i=1}^{N-1}\frac{1}{\sqrt{|\calE_i|}}\right)^{\frac{2}{3}} \left(\sum_{i=1}^{N-1} |\calE_i| \right)^{\frac{1}{3}}\leq \left(\frac{\dev_{[1, T]; \dupp}}{\dupp S\sqrt{A}}\right)^{\frac{2}{3}}T^{\frac{1}{3}}. 
    \end{align*}
    We can further upper bound the term $\frac{1}{\dupp} \dev_{[1,T]; \dupp}$ as follows:
    \begin{align*}
         \frac{1}{\dupp} \dev_{[1,T]; \dupp} 
         &= \frac{1}{\dupp}\left(\dev^r + 2\dupp \dev^p + \dev^J \right) \\
         &\leq \frac{1}{\dupp}\left(2\dev^r+ (2\dupp + D_{\max}) \dev^p  \right)   \tag{by \pref{lemma: ortner lemma}} \\
         &\leq \frac{3}{\dupp}\left(\dev^r + \dupp \dev^p \right) \\
         &\leq 3(\Delta^r + \Delta^p). 
    \end{align*}
    Thus we get 
    \begin{align*}
         N\leq 1+ 3\left(\frac{\Delta^r + \Delta^p}{S\sqrt{A}}\right)^{\frac{2}{3}} T^{\frac{1}{3}}. 
    \end{align*}
    
    Also, by \pref{lemma: infrequent RL restart}, when $\dupp \geq D_{\max}$, an epoch is created only when the reward function or the transition function changes. Thus the number of epochs is also upper bounded by $L$. 
\end{proof}

\begin{lemma}\label{lemma: single epoch RL regret}
     In every epoch $\calE$, the dynamic regret of \masteralg-\ucrl is upper bounded by 
     \begin{align*}
          \otil\left(X + \dupp S\sqrt{A|\calE|} + \dupp S^2A^2 \right), 
     \end{align*}
     where $X$ is the minimum of the following two terms: 
     \[
         \dupp S\sqrt{AL_{\calE}|\calE|} + \dupp SA L_{\calE} 
     \] 
     and
     \[
       \left(\dupp^2 S^2 A\dev_{\calE; \dupp} |\calE|^2\right)^{\nicefrac{1}{3}}  + SA\sqrt{\dupp \dev_{\calE; \dupp}|\calE|} + \left(\dupp S A^2 \dev_{\calE; \dupp}^2|\calE|\right)^{\nicefrac{1}{3}} + SA\dev_{\calE; \dupp}.
      \]
\end{lemma}

\begin{proof}
     Let $\calE_1, \ldots, \calE_N$ be the epochs. 
     By \pref{lemma: sum epoch regret}, we know that the regret within an epoch $\calE$ is $\otil\left(\min\left\{\Reg_L(\calE), \Reg_{\dev}(\calE)\right\} + \frac{c_2c_3}{c_1}|\calE|^{1-p} + \frac{c_2^2}{c_3}\right)$
     with 
     
     \begin{align*}
          \Reg_{L}(\calE)&=c_1 L_{\calE}^{1-p} |\calE|^{p} + c_2L_{\calE}, \\
          \Reg_{\dev}(\calE) &= \left(c_1 \dev_{\calE; \dupp}^{1-p}|\calE|\right)^{\frac{1}{2-p}}  + c_1|\calE|^p + c_1(c_3^{-1}\dev_{\calE; \dupp})^{1-p} |\calE|^{p} + c_2\left(c_1^{-1}\dev_{\calE; \dupp}|\calE|^{1-p}\right)^{\frac{1}{2-p}} + c_2 + c_2c_3^{-1}\dev_{\calE; \dupp}
     \end{align*}
     when $\regbound(t)$ is in the form of \pref{def: standard}. In our case $\regbound_{\ucrl}(t; \dupp)$ is in this form with $c_1=\dupp S\sqrt{A}$, $c_2=\dupp SA$, $c_3=\dupp$, and $p=\frac{1}{2}$. Using them in the bound above, we get that in an epoch, the dynamic regret is upper bounded by 
     \begin{align*}
          \otil\left(\min\left\{\Reg_L(\calE), \Reg_{\dev}(\calE)\right\} + \dupp\sqrt{A|\calE|} +  \dupp S^2A^2\right)
     \end{align*}
     where 
     \begin{align*}
          \Reg_L(\calE) &=\dupp S\sqrt{AL_{\calE}|\calE|} + \dupp SA L_{\calE}\\
          \Reg_{\dev}(\calE) &=  \left(\dupp^2 S^2 A\dev_{\calE; \dupp} |\calE|^2\right)^{\nicefrac{1}{3}} + \dupp S\sqrt{A|\calE|} + SA\sqrt{\dupp \dev_{\calE; \dupp}|\calE|} + \left(\dupp S A^2 \dev_{\calE; \dupp}^2|\calE|\right)^{\nicefrac{1}{3}} + \dupp SA + SA\dev_{\calE; \dupp}. 
     \end{align*}
     Collecting terms finishes the proof. 
\end{proof}

\begin{theorem}
\label{thm: known Dmax case}
     If $D_{\max}\leq \dupp \leq 2D_{\max}$, then \masterucrl guarantees the following dynamic regret bound: 
     \begin{align*}
          \otil\left( \min\left\{D_{\max}S\sqrt{ALT}, D_{\max}\left(S^2A\right)^{\nicefrac{1}{3}}(\dev^r + \dev^r)^{\nicefrac{1}{3}}T^{\nicefrac{2}{3}} + D_{\max}S\sqrt{AT} \right\}  \right).
     \end{align*}
\end{theorem}

\begin{proof}
      Let $\calE_1, \ldots, \calE_N$ be the epochs. The per epoch dynamic regret is given by \pref{lemma: single epoch RL regret}.  Combining them with H\"{o}lder's inequality and \pref{eq: L partition}, \pref{eq: dev partition}, the dynamic regret in $[1, T]$ can be upper bounded by 
      
      \begin{align}
          \otil\left(\min\{\Reg_L, \Reg_{\dev}\} + \dupp S\sqrt{ANT} + \dupp S^2A^2 N \right)  \label{eq: tmptmp6}
      \end{align}
      where  
      \begin{align}
           \Reg_{L}=\dupp S\sqrt{A(L+N)T} + \dupp SA(L+N)   \label{eq: tmptmp7}
      \end{align}
      and 
      
      \begin{align}
          \Reg_{\dev}
          &=\left(\dupp ^2S^2A \dev_{[1, T]; \dupp} T^2\right)^{\nicefrac{1}{3}} + SA\sqrt{\dupp \dev_{[1,T]; \dupp} T} + \left(\dupp SA^2 \dev_{[1,T]; \dupp}^2 T\right)^{\nicefrac{1}{3}} + SA\dev_{[1,T]; \dupp}.   \label{eq: tmptmp8}
      \end{align}
      Since $\dupp\geq D_{\max}$, the number of epochs can be bounded using \pref{lemma: RL number epoch}: 
      \begin{align*}
            N\leq \min\left\{L,\ \ 1+ 3\left(\frac{\Delta^r + \Delta^p}{S\sqrt{A}}\right)^{\frac{2}{3}} T^{\frac{1}{3}}\right\}. 
      \end{align*}

      With $N\leq L$, \pref{eq: tmptmp6}, and \pref{eq: tmptmp7}, the dynamic regret in $[1,T]$ can bounded by (omitting lower order terms)
      \begin{align}
           \otil\left(\dupp S\sqrt{ALT}\right).   \label{eq: tmp11}
      \end{align}
      With $N\leq 1+ 3\left(\frac{\Delta^r + \Delta^p}{S\sqrt{A}}\right)^{\frac{2}{3}} T^{\frac{1}{3}}$, \pref{eq: tmptmp6}, and \pref{eq: tmptmp8}, the regret can alternatively be upper bounded by (omitting lower order terms)
      \begin{align}
          \otil\left( \left(\dupp ^2S^2A \dev_{[1, T]; \dupp} T^2\right)^{\nicefrac{1}{3}} + \dupp S\sqrt{AT} + \dupp (S^2A)^{\nicefrac{1}{3}}(\dev^r + \dev^p)^{\nicefrac{1}{3}}T^{\nicefrac{2}{3}} \right).   \label{eq: tmp10}
      \end{align}
Then notice that $\dupp \leq 2D_{\max}$ and thus $\dev_{[1, T]; \dupp} = \dev^r + 2\dupp \dev^p + \dev^J = \order(\dev^r + D_{\max}\dev^p)$ where we use \pref{lemma: ortner lemma}. Using these in \pref{eq: tmp11} and \pref{eq: tmp10} finishes the proof. 
\end{proof}

\begin{theorem}\label{thm: doubling trick algo for RL}
    The doubling trick strategy described in Section~\ref{subsec:MUCRL} for the unknown $D_{\max}$ and known $L$ case has a dynamic regret bound of $\otil\left(D_{\max}S\sqrt{ALT}\right)$; for the unknown $D_{\max}$ and known $\Delta$ case, the bound is  
    \begin{align*}
          \otil\left(D_{\max} S\sqrt{AT} + D_{\max}(S^2A)^{\nicefrac{1}{3}}(\dev^r+ \dev^p)^{\nicefrac{1}{3}}T^{\nicefrac{2}{3}}\right).
    \end{align*}
\end{theorem}

\begin{proof}
     For the known $L$ case, when $\dupp \leq D_{\max}$, recall that the number of epochs is forced to be $N\leq L$. Similar to the proof of \pref{thm: known Dmax case}, the regret in any of these epochs is upper bounded by 
     \begin{align*}
           \otil\left(\dupp S\sqrt{ANT}\right) = \otil\left(\dupp S\sqrt{ALT}\right). 
      \end{align*}
     Summing the above over $\dupp = 1,2,4,\ldots, D_{\max}$, we get $\otil\left(D_{\max} S\sqrt{ALT}\right)$. 
    When $\dupp$ first enters $[D_{\max}, 2D_{\max}]$, we use \pref{thm: known Dmax case} to bound the regret in the rest of the rounds, which is still of order $\otil\left(D_{\max} S\sqrt{ALT}\right)$. 

For the case of known $\Delta = \Delta^r + \Delta^p$, the analysis is similar: when $\dupp\leq D_{\max}$, we force $N=1+3(S^{-2}A^{-1}\Delta^2 T)^{\nicefrac{1}{3}}$, and thus the regret within any of these epochs is upper bounded by (similarly to the proof of \pref{thm: known Dmax case})
\begin{align*}
          \otil\left( \left(\dupp ^2S^2A \dev_{[1, T]; \dupp} T^2\right)^{\nicefrac{1}{3}} + \dupp S\sqrt{AT} + \dupp (S^2A)^{\nicefrac{1}{3}}(\dev^r + \dev^p)^{\nicefrac{1}{3}}T^{\nicefrac{2}{3}} \right). 
      \end{align*}
Summing this over $\dupp=1, 2, \ldots, D_{\max}$ and using $\Delta_{[1, T]; \dupp}=\Delta^r + 2\dupp \Delta^p + \Delta^J = \order(\Delta^r + D_{\max}\Delta^p)$ for $\dupp=\order(D_{\max})$, we get  $ \otil\left( D_{\max}\left(S^2A \dev T^2\right)^{\nicefrac{1}{3}} + D_{\max} S\sqrt{AT} \right)$. When $\dupp$ first enters $[D_{\max}, 2D_{\max}]$, we use \pref{thm: known Dmax case} to bound the regret in the rest of the rounds, which is still of the same order. 
\end{proof}

\section{Bandit-over-Reinforcement-Learning Approach}\label{app: borl discuss}

The idea of the BoRL framework is to run a multi-armed bandit algorithm over a set of sub-algorithms each using a different parameter. In our case, each sub-algorithm is a \masterucrl with a different guess on $D_{\max}$. The set of $\dupp$ only needs to span the range of $[1,\sqrt{T}]$, since if $D_{\max}=\Omega(\sqrt{T})$, the regret bound would be vacuous. 

We divide the horizon into $\frac{T}{B}$ equal-length intervals each of length $B=S\sqrt{AT}$. In each interval, sub-algorithm $i$ restarts a \masterucrl with $\dupp=2^{i-1}$. The reward of sub-algorithm $i$ in interval $b\in[\frac{T}{B}]$ is its total reward gained in the MDP for this interval. We denote $i^\star$ as the sub-algorithm that uses $\dupp \in [D_{\max}, 2D_{\max}]$. 

On top of these sub-algorithms, we run the EXP3.P algorithm \citep{auer2002nonstochastic}. The ``arms'' are the sub-algorithms. From the above description, for this EXP3.P, there are $M=\lceil \log_2 \sqrt{T}\rceil$ arms, the algorithm proceeds for $\frac{T}{B}$ rounds, and in each round the reward range is $B$. By the standard regret bound of EXP3.P, the learner's regret against sub-algorithm $i^\star$ is of order 
\begin{align*}
     \otil\left(B\sqrt{M\frac{T}{B}} + BM\right) = \otil\left(\sqrt{BT}\right). 
\end{align*}
with high probability. 

On the other hand, in each interval $b\in\left[\frac{T}{B}\right]$, since sub-algorithm $i^\star$ uses a correct guess of $\dupp$, by \pref{thm: known Dmax case}, its regret against the best sequence of policy in that interval is 
\begin{align*}
      \otil\left( \min\left\{D_{\max}S\sqrt{AL_b B}, \ \ D_{\max}(S^2A)^{\frac{1}{3}}(\Delta_b)^{\frac{1}{3}}B^{\frac{2}{3}} + D_{\max}S\sqrt{AB}\right\} \right)
\end{align*}
where we abuse notations and denote $L_b=L_{[(b-1)B+1, bB]}$, $\Delta_b=\Delta_{[(b-1)B+1, bB]}$. 

Combining the two bounds above, we get that the regret of the learner against the best sequence of policies in $[1,T]$ is
\begin{align*}
    &\otil\left(\sqrt{BT} + \sum_{b=1}^{\frac{T}{B}}  \min\left\{D_{\max}S\sqrt{AL_b B}, \ \ D_{\max}(S^2A)^{\frac{1}{3}}(\Delta_b)^{\frac{1}{3}}B^{\frac{2}{3}} + D_{\max}S\sqrt{AB}\right\} \right) \\
    &= \otil\left(\sqrt{BT} + \min\left\{ D_{\max}S\sqrt{A\left(L + \frac{T}{B}\right)T},\ \  D_{\max}(S^2A)^{\frac{1}{3}}(\Delta)^{\frac{1}{3}}T^{\frac{2}{3}} + D_{\max}S\sqrt{AB}\times \frac{T}{B} \right\}\right)    \tag{using \pref{eq: L partition} and \pref{eq: dev partition}}\\
    &= \otil\left(\sqrt{BT} + D_{\max}S\sqrt{\frac{A}{B}}T  + \min\left\{ D_{\max}S\sqrt{ALT},\ \  D_{\max}(S^2A)^{\frac{1}{3}}\Delta^{\frac{1}{3}}T^{\frac{2}{3}}\right\}\right).     
\end{align*}
Using the $B$ that we specified above, we get 
\begin{align*}
     \otil\left( D_{\max}(S^2A)^{\nicefrac{1}{4}}T^{\nicefrac{3}{4}}  + \min\left\{ D_{\max}S\sqrt{ALT},\ \  D_{\max}(S^2A)^{\frac{1}{3}}\Delta^{\frac{1}{3}}T^{\frac{2}{3}}\right\}\right). 
\end{align*}


\section{Verifying \pref{assum:assump2} for Several Algorithms}
\label{app: verify example}

To prove \pref{eq: general condition 2}, it suffices to prove the following.
\paragraph{Assumption 1'}  \textit{There exist universal constants $\const_1, \const_2, \const_3, \const_4, \const_5, \const_6>0$ such that for all $t=1,2,\ldots$, as long as $\dev_{[1,t]}\leq \const_1\avgreg(t)$, the following holds with probability $1-\frac{\delta}{T}$: 
\begin{align}
        &\tildef_t \geq \min_{\tau\in [1,t]} f_\tau^\star - \const_2\dev_{[1,t]} \label{eq: general condition 3} \\ 
        &\frac{1}{t}\sum_{\tau=1}^t \left(\tildef_\tau - R_\tau\right) \leq \const_3\avgreg(t) + \const_4\dev_{[1,t]}.   \label{eq: general condition 4}
\end{align}
Furthermore, $\avgreg(t)\geq \frac{\const_5}{\sqrt{t}}$, $\dev(t) \geq \const_6 \max_{\pi}|f_t(\pi)-f_{t+1}(\pi)|$.  
}\\
This is because for an algorithm satisfying Assumption~1', we can redefine $\dev(t)\leftarrow (\const_3/\const_1 + \const_2 + \const_4 + 1/(\const_1\const_5)+ 1/\const_6) \dev(t)$ and $\avgreg(t)\leftarrow (\const_3+\const_1\const_2 + \const_1 \const_4 + 1/\const_5 + \const_1/\const_6)\avgreg(t)$. Then \pref{eq: general condition 2} is satisfied. Our verification below is thus mostly based on Assumption~1' for simplicity. 

The following proofs are brief (some of them are just sketches) since they follow standard analysis and mostly appear in previous works. Please find more details in the references. We sometimes make minor modifications to the original algorithm to make them more aligned with our framework. 

\subsection{UCB1 for Multi-armed Bandits}


\begin{exampl}
    \caption{UCB1 for multi-armed bandits}  
    \label{ex: ucb1}
    \textbf{input}: $A$ (number of arms), $T, \delta$.  \\
    \For{$t=1, \ldots, T$}{
        Choose $a_t = \argmax_{a\in[A]}  \left(\hatr_{t,a} +  c\sqrt{\frac{\log(T/\delta)}{N_{t,a}^+}}\right)$ \myComment{$c>0$ is some universal constant} \\
        where 
        \begin{align}
            \hatr_{t,a}= \frac{\sum_{\tau =1}^{t-1} R_\tau \one[a_\tau=a]}{N_{t,a}^+}, \qquad N_{t,a}^+=\max\left\{1, \sum_{\tau=1}^{t-1} \one[a_\tau=a]\right\}.  \label{eq: hatr and N} 
        \end{align} \\
        Receive $R_t$ with $\E[R_t]=r_{t,a_t}$. 
    }
\end{exampl}

In this subsection, we consider the multi-armed bandit problem and the UCB1 algorithm by \cite{auer2002finite}. Suppose there are $A$ arms, and let $r_{t,a}$ denote the expected reward of arm $a$ at time $t$. Then the multi-armed bandit problem fits in our framework with $\Pi=[A]$ and $f_t(a)=r_{t,a}$. Below, we show that the UCB1 algorithm satisfies Assumption 1'.

The pseudocode of UCB1 is presented in \pref{ex: ucb1}.  
 At time $t$, UCB1 chooses the arm that has the highest optimistic reward estimator $\tilder_{t,a}\triangleq  \hatr_{t,a} +  c\sqrt{\frac{\log(T/\delta)}{ N_{t,a}^+}}$, where $\hatr_{t,a}$ is the empirical mean of the reward of arm $a$ up to time $t-1$, $N_{t,a}$ is the cumulative number of pulls of arm $a$ up to time $t-1$ and $N_{t,a}^+=\max\{1, N_{t,a}\}$, all defined in \pref{eq: hatr and N}; $c>0$ is some universal constant that is determined by Azuma's inequality. 


To see that UCB1 satisfies Assumption 1', we define
\begin{align}
    \dev(t) = \max_{a}|r_{t,a}-r_{t+1,a}|, \qquad 
    \tildef_t = \max_{a} \tilder_{t,a}, \qquad 
    \avgreg(t) = \sqrt{\frac{A\log(T/\delta)}{t}} + \frac{A\log(T/\delta)}{t}. \label{eq: UCB1 choices}
\end{align}
Furthermore, denote $\barr_{t,a}= \frac{\sum_{\tau=1}^{t-1} r_{\tau, a} \one[a_t=a]}{N_{t,a}}$ (define $\barr_{t,a}=1$ if $N_{t,a}=0$ for simplicity). Note that with high probability, 
\begin{align*}
    \tildef_t \geq \max_{a}  \barr_{t,a} \geq \max_a \max_{\tau\leq t}  r_{\tau,a} - \dev_{[1,t]} \geq \min_{\tau\leq t} \max_a r_{\tau,a}- \dev_{[1,t]}, 
\end{align*}
where the first inequality is because with high probability, $\tilder_{t,a}\geq \barr_{t,a}$ by Azuma's inequality. This verifies \pref{eq: general condition 3}.

On the other hand, by the selection rule $a_t=\argmax_a \tilder_{t,a}$, we have with probability $1-\delta$, 
\begin{align*}
     \sum_{\tau=1}^t (\tildef_\tau - R_\tau) 
     &\leq \sum_{\tau=1}^t (\tilder_{\tau,a_\tau} - r_{\tau, a_\tau}) + \sum_{\tau=1}^t (r_{\tau, a_\tau}-R_\tau)\\
     &= \sum_{\tau=1}^t \left(\barr_{\tau,a_\tau} - r_{\tau, a_\tau}  + c\sqrt{\frac{\log(T/\delta)}{N_{\tau,a}^+}}\right) + \sum_{\tau=1}^t (r_{\tau, a_\tau}-R_\tau)\\
     &\leq t\dev_{[1,t]} + \order\left(\sqrt{At\log(T/\delta)}+A\log(T/\delta)\right)
\end{align*}
where in the last inequality we use $\barr_{\tau,a_\tau} - r_{\tau, a_\tau}\leq \dev_{[1,t]}$ and the standard pigeonhole argument, and use Azuma's inequality to bound $\sum_{\tau=1}^t (r_{\tau, a_\tau}-R_\tau)$. This proves \pref{eq: general condition 4}.  Note that the condition $\dev_{[1,t]}=\order(\avgreg(t))$ in Assumption~1' is even not needed. 


\subsection{OFUL for Linear Bandits}
\begin{exampl}
     \caption{OFUL for linear bandits}
     \label{ex: oful for linear bandit}
     \textbf{input: } $\calA \subset \mathbb{R}^d$ (action set), $T, \delta$.\\ 
     \For{$t=1, 2, \ldots, T$}{
         Choose $a_t=\argmax_{a\in\calA}\left(a^\top \hattheta_t + 2\beta\|a\|_{\Lambda_t^{-1}}\right)$, \\
         where 
         \begin{align}
             \beta = 4\sqrt{d\log(T/\delta)},\qquad  \Lambda_{t}=I + \sum_{\tau=1}^{t-1}a_\tau a_\tau^\top, \qquad \hattheta_t = \Lambda_{t}^{-1}\sum_{\tau=1}^{t-1} R_\tau a_\tau.  \label{eq: oful definitions}
         \end{align}\\
         Receive $R_t$ with $\E[R_t]=a_t^\top \theta_t$. 
     }  
\end{exampl}
In this subsection, we consider linear bandits with a fixed action set, and the OFUL algorithm by \cite{abbasi2011improved}. The original OFUL algorithm handles the case where the action set can change over time (also known as the linear contextual bandit setting), but this is beyond the main focus of this paper. Let $\calA$ be the action set, and $\theta_t$ be the reward vector at time $t$. Then the linear bandit problem fits in our framework with $\Pi=\calA$ and $f_t(a)=a^\top \theta_t$.

The pseudocode of OFUL (with a fixed action set) is presented in \pref{ex: oful for linear bandit}. 
For simplicity, assume that for all actions $a\in\calA$, $\|a\|_2\leq 1$, and for all $t$, the reward vector $\theta_t$ satisfies $\|\theta_t\|_2\leq 1$. The OFUL algorithm chooses the action $a_t=\argmax_{a} a^\top \hattheta_t + 2\beta\|a\|_{\Lambda_t^{-1}}$ at time $t$, 
where $\beta$, $\Lambda_{t}=I + \sum_{\tau=1}^{t-1}a_\tau a_\tau^\top$, and $\hattheta_t$ are defined in \pref{eq: oful definitions}. 

Then we define
\begin{align}
     \dev(t) =d\sqrt{\log (T/\delta)}\|\theta_t-\theta_{t+1}\|_2, \;\; 
     \tildef_t = \max_{a\in\calA} \left(a^\top \hattheta_t + 2\beta \|a\|_{\Lambda_{t}^{-1}}\right), \;\; 
     \avgreg(t) = \beta\sqrt{\frac{d\log (T/\delta)}{t}}.  \label{eq: oful choices}
\end{align}
Below, we verify that OFUL satisfies Assumption 1' with the choices in \pref{eq: oful choices}. Under the assumption that $\dev_{[1,t]}\leq \avgreg(t)$, for any action $a$, by similar arguments as in \cite[Lemma 1]{pmlr-v108-zhao20a},  
    \begin{align}
        \left|a^\top (\theta_t - \widehat{\theta}_t)\right| 
        &\leq \left|a^\top \Lambda_{t}^{-1}\sum_{\tau=1}^{t-1}a_\tau a_\tau^\top (\theta_s - \theta_t)\right| + \beta \|a\|_{\Lambda_{t}^{-1}}  \nonumber\\ 
        &\leq \sum_{\tau=1}^{t-1}\left| a^\top \Lambda_{t}^{-1}a_\tau \right| \left| a_\tau^\top (\theta_\tau - \theta_t) \right| + \beta \|a\|_{\Lambda_{t}^{-1}}  \nonumber\\
        &\leq \frac{\dev_{[1,t]}}{d\sqrt{\log (T/\delta)}}\times \left( \sum_{\tau=1}^{t-1}\|a\|_{\Lambda_{t}^{-1}}\|a_\tau\|_{\Lambda_{t}^{-1}}\right) + \beta \|a\|_{\Lambda_{t}^{-1}} \tag{ $a_\tau^\top (\theta_\tau-\theta_t) \leq \|\theta_\tau-\theta_t\|_2\leq \frac{\dev_{[1,t]}}{d\sqrt{\log (T/\delta)}}$}    \nonumber\\
        &\leq \frac{\dev_{[1,t]}}{d\sqrt{\log (T/\delta)}} \times  \|a\|_{\Lambda_{t}^{-1}} \times \sqrt{(t-1)\sum_{\tau=1}^{t-1} \|a_\tau\|_{\Lambda_{t}^{-1}}^2} + \beta \|a\|_{\Lambda_{t}^{-1}} \tag{Cauchy-Schwarz}   \nonumber\\
        &\leq  \left(\beta + \dev_{[1,t]}\sqrt{\frac{t}{d\log (T/\delta)}}\right) \|a\|_{\Lambda_{t}^{-1}} \tag{$\sum_{\tau=1}^{t-1}\|a_\tau\|_{V_{t-1}^{-1}}^2=\text{tr}(\Lambda_{t}^{-1}\sum_{\tau=1}^{t-1}a_\tau a_\tau^\top)\leq d$}   \nonumber\\
        &\leq 2\beta \|a\|_{\Lambda_{t}^{-1}}. \tag{by the assumption $\dev_{[1,t]}\leq \avgreg(t)$} \\
        \label{eq: error bound 1}
    \end{align}
Thus, 
\begin{align*}
    \sum_{\tau=1}^t \left(\widetilde{f}_\tau - R_\tau\right)  
    &=\sum_{\tau=1}^t \left(\widetilde{f}_\tau - a_\tau^\top \theta_\tau\right) + \sum_{\tau=1}^t \left(a_\tau^\top \theta_\tau - R_\tau\right) \\ 
    &= \sum_{\tau=1}^t a_\tau^\top \left(\widehat{\theta}_\tau - \theta_\tau\right) + 2 \sum_{\tau=1}^t \beta\|a_\tau\|_{\Lambda^{-1}_{\tau}} + \order\left(\sqrt{t\log(T/\delta)}\right) \tag{by the definition of $\tildef_\tau$ and that OFUL chooses $a_\tau = \argmax_{a}\left(a^\top \widehat{\theta}_\tau + 2 \beta\|a\|_{\Lambda_{\tau}^{-1}}\right)$}\\
    & = \order\left(\sum_{s=1}^t \beta\|a_s\|_{\Lambda^{-1}_{s}}\right) + \order\left(\sqrt{t\log(T/\delta)}\right) \tag{by \pref{eq: error bound 1}}\\
    & = \order\left(\beta\sqrt{dt\log t}\right) = \order\left(t\rho(t)\right)\leq \order\left(t\avgreg(t) + t\dev_{[1,t]}\right).    
\end{align*}
This verifies \pref{eq: general condition 4}. Also, by \pref{eq: error bound 1}, 
\begin{align*}
    \widetilde{f}_t 
    &= \max_{a}\left(a^\top \widehat{\theta}_t + 2 \beta\|a\|_{\Lambda_{t}^{-1}}\right) \geq \max_a a^\top \theta_t = f_t^\star \geq \min_{\tau\in[1,t]}f_\tau^\star - \dev_{[1,t]}. 
\end{align*}
This verifies \pref{eq: general condition 3}.  
\subsection{GLM-UCB for Generalized Linear Bandits}

\begin{exampl}
    \caption{GLM-UCB for generalized linear bandits} 
    \label{ex: glm-ucb}
    \textbf{input}: $\calA\subset \mathbb{R}^d, T, \delta, \mu$ (link function), $\lambda$.  \\
    \textbf{define}: $k_\mu = \sup_{x\in[0,1]} \frac{\mathrm{d}\mu(x)}{\mathrm{d}x}$,\quad  $c_\mu = \inf_{x\in[0,1]} \frac{\mathrm{d}\mu(x)}{\mathrm{d}x}>0$.   \\
    \For{$t=1, \ldots, T$}{
        Choose $a_t = \argmax_{a\in\calA}\left( \mu(a^\top \hattheta_t) +  2\beta\|a\|_{\Lambda_t^{-1}} \right)$ 
        \\
        where
        \begin{align*}
            \beta =  \frac{4k_\mu}{c_\mu}\left(\sqrt{d\log(c_\mu T/(\lambda\delta))} + c_\mu\sqrt{ \lambda}\right), \qquad \Lambda_t = \lambda I + \sum_{\tau=1}^{t-1}a_\tau  a_\tau^\top,
        \end{align*}\\
        and $\hattheta_t$ is the unique solution of the following set of equations (define $g_t(x)\triangleq \lambda c_\mu x + \sum_{\tau=1}^{t-1} \mu(a_\tau^\top x)a_\tau$): 
        \begin{align*}
            g_t(\theta_t') =  \sum_{\tau=1}^{t-1}R_\tau a_\tau, \qquad 
            \hattheta_t = \argmin_{\theta: \|\theta\|_2\leq 1} \norm{g_t(\theta_t') - g_t(\theta)} _{\Lambda_t^{-1}}. 
        \end{align*}
        \\
        Receive $R_t$ with $\E[R_t] = \mu(a_t^\top \theta_t)$. 
    }
\end{exampl}

Generalized linear bandit is proposed by \cite{filippi2010parametric} and extended to the non-stationary case by \cite{cheung2019learning, pmlr-v108-zhao20a, russac2020algorithms, faury2021regret}. We refer the readers to these papers for the introduction of the setting. Again, we consider the special case where the action set is fixed over time, and for simplicity, we assume that the action set $\calA$ is a subset of $\{a\in\mathbb{R}^d:~\|a\|_2\leq 1\}$ and the hidden parameter $\theta_t$ satisfies $\|\theta_t\|_2\leq 1$. The generalized linear bandit problem is accompanied with an increasing link function $\mu: \mathbb{R}\rightarrow \mathbb{R}$. It fits in our framework with $\Pi=\calA$ and $f_t(a) = \mu(a^\top \theta_t)$.

The standard GLM-UCB is presented in 
\pref{ex: glm-ucb}. 
Below we show that GLM-UCB satisfies Assumption 1' with the following definitions: 
\begin{align*}
    \dev(t) = \frac{k_\mu^2d}{c_\mu} \sqrt{\log (T/\delta)}\|\theta_t - \theta_{t+1}\|_2, \;\; \tildef_t = \max_{a\in\calA} \left(\mu(a^\top \hattheta_t) + 2\beta \|a\|_{\Lambda_{t}^{-1}}\right), \;\; 
     \avgreg(t) = \beta\sqrt{\frac{d\log (T/\delta)}{t}},
\end{align*}
where $c_\mu$ and $k_\mu$ are the infimum and supremum of the derivative of $\mu$ (defined in \pref{ex: glm-ucb}). 
Define $G_t\triangleq \sum_{\tau=1}^{t-1}\left[\int_{v=0}^1 \dot{\mu}\left(\inner{a_\tau, (1-v)\hattheta_t + v\theta_t}  \right) \mathrm{d}v\right]a_\tau a_\tau^\top + \lambda c_\tau I \succeq c_\mu \Lambda_t$. 
Under the assumption that $\dev(t)\leq \avgreg(t)$, for all $a\in\calA$, 
\begin{align}
    \left|\mu(a^\top \theta_t) - \mu(a^\top \hattheta_t) \right| 
    &\leq k_u \left|a^\top (\theta_t - \hattheta_t)\right| 
    \leq k_\mu \left| a^\top G_t^{-1} (g_t(\theta_t) - g_t(\hattheta_t)) \right| \nonumber \\
    &\leq k_\mu \|a\|_{G_t^{-1}} \norm{g_t(\theta_t) - g_t(\hattheta_t)}_{G_t^{-1}} \nonumber \\
    &\leq \frac{k_\mu}{c_\mu} \|a\|_{\Lambda_t^{-1}} \norm{g_t(\theta_t) - g_t(\hattheta_t)}_{\Lambda_t^{-1}} \nonumber \\
    &\leq \frac{k_\mu}{c_\mu} \|a\|_{\Lambda_t^{-1}} \norm{g_t(\theta_t) - g_t(\theta_t')}_{\Lambda_t^{-1}} \nonumber \\
    &= \frac{k_\mu}{c_\mu} \|a\|_{\Lambda_t^{-1}} \norm{ \sum_{\tau=1}^{t-1} \left(\mu(a_\tau^\top \theta_t) - \mu(a_\tau^\top \theta_\tau)\right)a_\tau + \sum_{\tau=1}^{t-1} \left(\mu(a_\tau^\top \theta_\tau) - R_\tau)\right)a_\tau + \lambda c_\mu \theta_t    }_{\Lambda_t^{-1}} \nonumber \\
    &\leq \frac{k_\mu}{c_\mu} \|a\|_{\Lambda_t^{-1}} \left(k_\mu\max_{\tau\leq t}\norm{\theta_t - \theta_\tau}_2 \sum_{\tau=1}^{t-1} \norm{ a_\tau }_{\Lambda_t^{-1}} + \norm{\sum_{\tau=1}^{t-1}\eta_\tau a_\tau}_{\Lambda_t^{-1}} + \sqrt{\lambda}c_\mu\right)  \tag{define $\eta_\tau=\mu(a_\tau^\top \theta_\tau) - R_\tau$}\\
    &\leq \frac{k_\mu}{c_\mu} \|a\|_{\Lambda_t^{-1}} \left(k_\mu\frac{c_\mu\dev_{[1,t]}}{k_\mu^2d\sqrt{\log(T/\delta)}} \sqrt{dt} + \sqrt{d\log(c_\mu T/\delta)} + \sqrt{\lambda}c_\mu\right) \nonumber \\
    &\leq \frac{k_\mu}{c_\mu} \|a\|_{\Lambda_t^{-1}} \left(\frac{c_\mu\avgreg(t)}{k_\mu} \sqrt{\frac{t}{d\log(T/\delta)}} + \sqrt{d\log(c_\mu T/\delta)} + \sqrt{\lambda}c_\mu\right) \tag{by the assumption $\dev_{[1,t]}\leq \avgreg(t)$} \\
    &\leq 2 \beta \|a\|_{\Lambda_t^{-1}}. \label{eq: glm optimistic}
\end{align}
Thus, 
\begin{align*}
    \sum_{\tau=1}^t \left(\tildef_\tau - R_\tau\right) 
    &= \sum_{\tau=1}^t \left(\tildef_\tau - \mu(a_\tau^\top \theta_\tau)\right) + \sum_{\tau=1}^t \left(\mu(a_\tau^\top \theta_\tau) - R_\tau\right) \\
    &\leq \sum_{\tau=1}^t \left(\mu(a_\tau^\top \hattheta_\tau) - \mu(a_\tau^\top \theta_\tau)\right) + 2\beta \sum_{\tau=1}^t \|a_\tau\|_{\Lambda_t^{-1}} + \order\left(\sqrt{t\log(T/\delta)}\right) \\
    &\leq \order\left(\beta \sum_{\tau=1}^t \|a_\tau\|_{\Lambda_t^{-1}}\right) = \order\left(\beta \sqrt{dt\log(T/\delta)}\right)\\
    &= \order\left(t\avgreg(t)\right) = \order\left(t\avgreg(t) + t\dev_{[1,t]}\right). 
\end{align*} 
This verifies \pref{eq: general condition 4}. Furthermore, by \pref{eq: glm optimistic}, 
\begin{align*}
    \tildef_t = \max_{a\in\calA} \left(\mu(a^\top \hattheta_t) + 2\beta \|a\|_{\Lambda_{t}^{-1}}\right) \geq \max_{a\in\calA} \mu(a^\top \theta_t) = f_t^\star \geq \min_{\tau\in[1,t]}f_\tau^\star - \dev_{[1,t]}. 
\end{align*}
This verifies \pref{eq: general condition 3}.

\subsection{Q-UCB for Finite-horizon Tabular MDPs}
\begin{exampl}
     \caption{Q-UCB for finite-horizon tabular MDPs}
     \label{ex: qucb}
     \textbf{input}: $S$ (number of states), $A$ (number of actions), $H$, $T, \delta$. \\
     $Q_h(s,a)\leftarrow H, \quad  N_h(s,a)\leftarrow 0$ for all $h, s, a$. \\
     \For{$t=1, \ldots, T$}{
         \For{$h=1, \ldots, H$}{
             Choose $a^t_h \leftarrow \argmax_a Q_h(s^t_h, a)$. \\
             $\tau = N_h(s^t_h, a^t_h)\leftarrow N_h(s^t_h, a^t_h) +1$, \qquad $b_\tau\leftarrow c\sqrt{H^3\log(SAT/\delta)/\tau}$. \myComment{$c$ is a universal constant}\\
             $Q_h(s^t_h, a^t_h)\leftarrow (1-\alpha_\tau)Q_h(s^t_h, a^t_h) + \alpha_\tau \left[r_h^t(s_h^t, a_h^t) + V_{h+1}(s_{h+1}^t) + b_\tau\right]$. \myComment{$\alpha_\tau\triangleq \frac{H+1}{H+\tau}$} \\
             $V_h(s^t_h)\leftarrow \min\left\{H, \max_a Q_h(s^t_h, a)\right\}$. 
         }
     }
\end{exampl}
The finite-horizon tabular MDP problem fits in our framework with $\Pi$ being the set of deterministic polices on the MDP, and $f_t(\pi)$ being the expected reward of policy $\pi$ in episode $t$. Q-UCB (Hoeffding-style) is a model-free algorithm for finite-horizon tabular MDPs proposed by \cite{jin2018q}, whose pseudocode is in \pref{ex: qucb}. Let $H$ denote the horizon length, $s^t_h, a^t_h$ denote the state and actions visited at step $h$ of episode $t$, and $r^t_h$, $p^t_h$ denote the reward and transition functions at step $h$ of episode $t$. Without loss of generality, we assume that $s^t_1=s_1$ for all $t$ (i.e., the initial state is fixed). 

It has been shown in the proof of \citep[Theorem 1]{mao2020nearoptimal} that Q-UCB satisfy Assumption 1' with the following choices: 
\begin{align*}
    \dev(t)&=H\sum_{h=1}^H \max_{s,a}|r_h^{t}(s,a)-r_{h}^{t+1}(s,a)| + H^2\sum_{h=1}^H \max_{s,a}\|p_h^{t}(\cdot|s,a)-p_{h}^{t+1}(\cdot|s,a)\|_1, \\
    \tildef_t &= V^t_h(s_1), \tag{$V^t_h$ is the $V_h$ in \pref{ex: qucb} at the beginning of episode $t$}  \\
    \avgreg(t)&=\otil\left(\sqrt{\frac{H^5SA}{t}} + \frac{H^3SA}{t}\right). 
\end{align*}
The proof details are omitted here. 

\subsection{LSVI-UCB for Finite-horizon Linear MDPs}
\begin{exampl}
    \caption{LSVI-UCB for finite-horizon linear MDP}
    \label{ex: linear mdp ucb}
    \textbf{input}: $\calS$ (state space), $\calA$ (action space), $\phi(\cdot,\cdot): \calS\times \calA\rightarrow \mathbb{R}^d, H, T, \delta$. \\
    \For{$t=1, \ldots, T$}{
        \For{$h=H, \ldots, 1$}{
            $\Lambda_h \leftarrow \sum_{\tau=1}^{t-1}\phi(s^\tau_h, a^\tau_h)\phi(s^\tau_h, a^\tau_h)^\top + I$. \\
            $w_h \leftarrow \Lambda_h^{-1} \sum_{\tau=1}^{t-1}\phi(s^\tau_h, a^\tau_h)\left[r_h^\tau(x_h^\tau, a_h^\tau)+ \max_{a\in\calA} Q_{h+1}(x^\tau_{h+1},a)\right]$. \\
            $Q_h(\cdot, \cdot)\leftarrow \min\left\{w_h^\top \phi(\cdot,\cdot) + 2\beta \left(\phi(\cdot,\cdot)\Lambda_h^{-1}\phi(\cdot,\cdot)\right)^{1/2}, H\right\}$.  \\
            \ \myComment{$\beta = cdH\sqrt{\log(T/\delta)}$ for some universal constant $c$} \label{line: linear mdp time t}
        }
        \For{$h=1, \ldots, H$}{
            Take action $a^t_h\leftarrow \argmax_{a\in\calA} Q_h(s^t_h, a)$. 
        }
    }
\end{exampl}

See \citep{zhou2020nonstationary, touati2020efficient} for the non-stationary finite-horizon linear MDP setting. We assume that the reward function and the transition function at step $h$ of episode $t$ are $r^t_h(s,a)=\phi(s,a)^\top \theta^t_h$ and $p^t_h(s'|s,a)=\phi(s,a)^\top \mu^t_h(s')$ where $\phi(\cdot,\cdot)$ is the feature function that maps a state-action pair to a $d$-dimensional feature vector. 
The problem fits in our framework with $\Pi$ being the set of deterministic policies, and $f_t(\pi)$ being the expected reward of policy $\pi$ in episode $t$. The LSVI-UCB algorithm is an optimism-based algorithm proposed by \cite{jin2020provably}, whose pseudocode is shown in \pref{ex: linear mdp ucb}. We define $Q^t_h, w^t_h, \Lambda^t_h$ to be the $Q_h, w_h, \Lambda_h$ at \pref{line: linear mdp time t} of round $t$. Furthermore, define $V^t_h(s) = \max_{a\in\calA} Q^t_h(s,a)$. Again, without loss of generality, we assume $s^t_1=s_1$ (the initial state is fixed).

We define 
\begin{align*}
    \dev(t) &= dH\sqrt{\log(T/\delta)}\left(\sum_{h=1}^H \|\theta^t_{h}-\theta^{t+1}_{h}\|_2 + H\sum_{h=1}^H \|\mu^t_{h}-\mu^{t+1}_{h}\|_{F}\right), \\
    \tildef_t &= V^t_1(s_1), \\ 
    \avgreg(t) &= c\sqrt{\frac{d^3 H^4}{t}}\log (T/\delta) = \beta H\sqrt{\frac{d\log(T/\delta)}{t}}. \tag{$c$ and $\beta$ defined in \pref{ex: linear mdp ucb}} 
\end{align*}
Below, we verify that LSVI-UCB satisfies Assumption 1' with the $\Delta$, $\tildef$, and $\rho$ defined above. Assume that $\dev_{[1,t]}\leq \avgreg(t)$. By similar arguments as in the proof of \citep[Lemma 3]{zhou2020nonstationary}, we have 
\begin{align}
    &\left| \phi(s,a)^\top w^t_h - Q_{h}^\star(s,a) - \mathbb{P}^t_h(V^t_h - V^\star_h)(s,a)\right| 
    \leq \left(\beta + \sqrt{dt} B_{\theta, [1,t]} + H\sqrt{dt} B_{\mu,[1,t]} \right) \left\|\phi(s,a)\right\|_{\left(\Lambda^t_h\right)^{-1}}   \label{eq: error linear MDP bound}
\end{align}
where $B_{\theta,[1,t]} = \sum_{\tau=1}^{t-1}\sum_{h=1}^H \|\theta^\tau_{h}-\theta^{\tau+1}_{h}\|_2$ and $B_{\mu, [1,t]} = \sum_{\tau=1}^{t-1}\sum_{h=1}^H \|\mu^\tau_{h}-\mu^{\tau+1}_{h}\|_{F}$. By the definition of $\dev(t)$, the right-hand side of \pref{eq: error linear MDP bound} can be further upper bound by 
\begin{align}
    \left(\beta + \frac{1}{H}\sqrt{\frac{t}{d\log(T/\delta)}}\dev_{[1,t]}\right)\left\|\phi(s,a)\right\|_{\left(\Lambda^t_h\right)^{-1}} \leq 2\beta \left\|\phi(s,a)\right\|_{\left(\Lambda^t_h\right)^{-1}}, \label{eq: final bound linear MDP}
\end{align}
where the inequality is by the assumption that $\dev_{[1,t]}\leq \avgreg(t)$. Similar to the proof of \citep[Lemma 4]{zhou2020nonstationary}, we can then show that for any $t, h$, 
\begin{align*}
    Q_h^t(s,a) - Q_h^{\star}(s,a) 
    &=  \phi(s,a)^\top w^t_h - Q_h^\star(s,a)  + 2\beta\|\phi(s, a)\|_{\left(\Lambda^t_h\right)^{-1}} \\
    &\geq \max_{s'} \left(V^t_{h+1}(s') - V^\star_{h+1}(s')\right) \tag{by \pref{eq: error linear MDP bound} and \pref{eq: final bound linear MDP}}
\end{align*}
and further using induction to show that $V^t_1(s)\geq V^\star_1(s)$. Thus, $\tildef_t=V_1^t(s_1)\geq V_1^\star(s_1)$, which verifies \pref{eq: general condition 3}. One can also show that $\sum_{\tau=1}^t \left(\tildef_\tau - R_\tau\right) = \order(t\avgreg(t))$ using the standard analysis of LSVI-UCB (e.g., \citep[Theorem 3.1]{jin2020provably}, \citep[Theorem 5]{zhou2020nonstationary}). This verifies \pref{eq: general condition 4}.

\subsection{ILOVETOCONBANDITS for Contextual Bandits} 
\begin{exampl}[H]
     \caption{ILOVETOCONBANDITS for contextual bandits}
     \label{ex: ilove}
     \textbf{input}: $\Pi$ (policy set), $\calA$ (action set), $T$, $\delta$. \\
     \For{$t=1, \ldots, T$}{
         Calculate $Q_t\in\Delta_{\Pi}$ that satisfies the following constraints with some universal constant $c'>0$:  
         \begin{align*}
              \sum_{\pi}Q(\pi)\hatReg_{[1,t-1]}(\pi) &\leq 2c'A\mu_t \\
              \forall \pi\in\Pi, \qquad \frac{1}{t-1}\sum_{\tau=1}^{t-1} \frac{1}{Q^{\mu_t}(\pi(x_\tau)|x_{\tau})} &\leq 2A + \frac{\hatReg_{[1,t-1]}(\pi)}{c'\mu_t}
         \end{align*}
         where $\mu_t\triangleq \sqrt{\frac{\log(|\Pi|T/\delta)}{At}}$, $Q^{\mu}(a|x)\triangleq (1-A\mu)\sum_{\pi\in\Pi}Q(\pi)\one[\pi(x)=a] + \mu$, and 
         \begin{align*}
             \hatReg_{\calI}(\pi) \triangleq \frac{1}{|\calI|}\max_{\pi'}\sum_{\tau\in\calI}  \left(\hatr_\tau(\pi'(x_\tau)) - \hatr_\tau(\pi(x_\tau))\right), \qquad \hatr_\tau(a) \triangleq \frac{R_\tau\one[a_\tau=a]}{p_{\tau}(a)}.
         \end{align*}
         Let $p_t(\cdot) = Q^{\mu_t}(\cdot|x_t)$ and sample $a_t\sim p_t$. 
     }
     
\end{exampl}

In the contextual bandit problem, in each round, the learner first sees a context $x_t\in\calX$, and then chooses an action $a_t\in[A]$ based on it. The learner then receives the reward $r_t(a_t)\in\mathbb{R}$. We assume that $(x_t, r_t)$ is sampled from the distribution $\calD_t$. The goal of the learner is to be comparable to the best mapping $\pi: \calX\rightarrow [A]$ within a given set of mappings $\Pi$ (which are called \emph{policies}), i.e., the learner wants to minimize $\sum_t(r_t(\pi_t^*(x_t)) - r_t(a_t)$ where $\pi_t^*\triangleq \max_{\pi'\in\Pi} \E_{(x,r)\sim\calD_t}[r(\pi'(x))]$.  See \citep{agarwal2014taming} for more detailed description of the problem. This problem fits in our framework with the same $\Pi$ and $f_t(\pi) = \E_{(x,r)\sim \calD_t}[r(\pi(x))]$. 

The algorithm ILOVETOCONBANDITS (\pref{ex: ilove}) by \cite{agarwal2014taming} achieves the optimal regret bound in the i.i.d. case. 
The analysis for ILOVETOCONBANDITS is more involved. Fortunately, \cite{chen2019new} already has helpful lemmas for ILOVETOCONBANDITS in the non-stationary case, and we can simply reuse them. 
We show a more general result that \pref{assum:assump2} is satisfied no matter how large $\dev_{[1,t]}$ is. 

Let $\calR_{\calI}(\pi)=\frac{1}{|\calI|}\sum_{\tau\in\calI}\E_{(x,r)\sim \calD_\tau}\left[r(\pi(x))\right]$ be the expected of policy $\pi$ in the interval $\calI$, $\hatcalR_{\calI}(\pi)= \frac{1}{|\calI|}\sum_{\tau\in\calI} \hatr_\tau(\pi(x_\tau))$ be an unbiased estimator of $\calR_{\calI}(\pi)$, with $\hatr_\tau$ an unbiased estimator for the action reward constructed with inverse propensity weighting at time $\tau$. Let $\Reg_{\calI}(\pi)=\max_{\pi'}\calR_{\calI}(\pi') - \calR_{\calI}(\pi)$ and $\hatReg_{\calI}(\pi)=\max_{\pi'}\hatcalR_{\calI}(\pi') - \hatcalR_{\calI}(\pi)$. Below, we will show that ILOVETOCONBANDITS satisfies Assumption 1' with the following definitions: 
\begin{align*}
    \dev(t)&\triangleq \|\calD_t-\calD_{t+1}\|_{\text{TV}}=\int_{r}\int_x |\calD_t(x,r) - \calD_{t+1}(x,r)| \mathrm{d}x\mathrm{d}r, \\
    \tildef_t &\triangleq  \max_{\pi}\hatcalR_{[1,t-1]}(\pi) + \const_2 A\mu_{t-1} \tag{for some universal constant $c_2>0$} \\
    \avgreg(t)&\triangleq \sqrt{\frac{A\log(|\Pi|T/\delta)}{t}}.
\end{align*}

Note that $\dev(t)$ upper bounds $|\E_{(x,r)\sim \calD_{t}}[r(\pi(x))] - \E_{(x,r)\sim \calD_{t+1}}[r(\pi(x))]|$. 

Combining the proofs of Lemma 14 and Lemma 16 in \citep{chen2019new}, we get the following guarantee with probability at least $1-\delta$ for any policy $\pi$: 
\begin{align}
      \left|\hatcalR_{[1,t]}(\pi) - \calR_{[1,t]}(\pi)\right| \leq \const_1 \hatReg_{[1,t]}(\pi) + \const_2A\mu_t + \const_3 \dev_{[1,t]}, \label{eq: CB R deviaion}
\end{align}
where $\mu_t=\sqrt{\frac{\log(|\Pi|T/\delta)}{At}}$ and $\const_1, \const_2, \const_3$ are universal constants. To see how to get \pref{eq: CB R deviaion}, notice that Lemma 14 of  \citep{chen2019new} gives $\left|\hatcalR_{[1,t]}(\pi)-\calR_{[1,t]}(\pi)\right| \leq \order\left(\frac{\mu_t}{t}\sum_{\tau=1}^t U_\tau +  \frac{\log(|\Pi| T/\delta)}{t\mu_t}\right)$, and they further upper bound $U_\tau$ by $\order\left( \frac{\Reg_{[1,t]}}{\mu_t} + A + \frac{\dev_{[1,t]}}{\mu_t} \right)$ in the second-to-last line in their proof of Lemma 16. Combining them yields \pref{eq: CB R deviaion}. Notice that they have an additional $\log T$ factor which we do not suffer. 

Below, let $\barpi_t=\argmax_{\pi}\calR_{[1,t]}(\pi)$. Then we have
\begin{align}
     \max_{\pi}\hatcalR_{[1,t]}(\pi) 
     &\geq \hatcalR_{[1,t]}(\barpi_t)
     \geq \calR_{[1,t]}(\barpi_t) - \const_3\dev_{[1,t]} - \const_2 A\mu_t \nonumber
     \\
     &=\max_{\pi}\calR_{[1,t]}(\pi) - \const_3\dev_{[1,t]} - \const_2 A\mu_t, \label{eq: CB6}
\end{align}
where in the second inequality we use \pref{eq: CB R deviaion} with the fact that $\Reg_{[1,t]}(\barpi_t)=0$. 
Therefore, if we choose $\tildef_t = \max_{\pi}\hatcalR_{[1,t-1]}(\pi) + \const_2 A\mu_{t-1}$, then
\begin{align}
     \tildef_t 
     &\geq \max_\pi \calR_{[1,t-1]}(\pi) - \const_3\dev_{[1,t-1]}    \tag{using \pref{eq: CB6} and the definition of $\tildef_t$} \\
     &\geq \max_{\pi}\max_{\tau\in[1,t]} \calR_\tau(\pi) - (\const_3+1)\dev_{[1,t]}   \label{eq: ILOVE first condition}
\end{align}
which proves \pref{eq: general condition 3}.  
Next, we show \pref{eq: general condition 4}: 
\begin{align}
     \tildef_t - \E_t[R_t]
     &\leq \sum_\pi Q_t(\pi)\left(\tildef_t - \calR_t(\pi)\right) + \order(A\mu_t)    \tag{by the algorithm, which uses $\order(A\mu_t)$ probability to explore actions}\\
     &= \sum_\pi Q_t(\pi)\left(\max_{\pi'}\hatcalR_{[1,t-1]}(\pi') - \calR_t(\pi)\right) + \order(A\mu_t)   \nonumber \\
     &\leq \sum_\pi Q_t(\pi)\left(\max_{\pi'}\hatcalR_{[1,t-1]}(\pi') - \calR_{[1,t-1]}(\pi)\right) + \order\left(A\mu_t+ \dev_{[1,t]} \right)    \nonumber \\
     &= \sum_\pi Q_t(\pi)\left( \hatReg_{[1,t-1]}(\pi) + \hatcalR_{[1,t-1]}(\pi) - \calR_{[1,t-1]}(\pi)\right) + \order\left(A\mu_t+ \dev_{[1,t]} \right)   \nonumber \\
     &\leq \sum_\pi Q_t(\pi)\left( \hatReg_{[1,t-1]}(\pi) + \const_1\Reg_{[1,t-1]}(\pi)\right)  + \order\left(A\mu_t+\dev_{[1,t]}\right)   \tag{by \pref{eq: CB R deviaion}}   \nonumber \\
     &\leq (1+2\const_1)\sum_{\pi}Q_t(\pi)\hatReg_{[1,t-1]}(\pi) + \order\left(A\mu_t + \dev_{[1,t]}\right)   \label{eq: ILOVE condition 2}
\end{align}
where the last inequality is by Lemma 16 of \citep{chen2019new}, which bounds $\Reg_{[1,t-1]}(\pi)$ by $2\hatReg_{[1,t-1]}(\pi) + \order\left(A\mu_t + \dev_{[1,t]}\right)$. 
By the algorithm, $\sum_{\pi}Q_t(\pi)\hatReg_{[1,t-1]}(\pi)$ is of order $\order\left(A\mu_t\right)$. Therefore, the last expression can further be upper bounded by $\order\left(A\mu_t + \dev_{[1,t]}\right)$. 
Finally, with the above calculation and Azuma's inequality, we get 
\begin{align*}
     \sum_{\tau=1}^t \left(\tildef_\tau - R_\tau\right)\leq \sum_{\tau=1}^t \left(\tildef_\tau - \E_\tau[R_\tau]\right) + \sum_{\tau=1}^t \left(\E_\tau[R_\tau] - R_\tau\right) \leq \order\left(\sqrt{At\log(|\Pi|T/\delta)} + \dev_{[1,t]}\right)
\end{align*}
Since we choose $\avgreg(t)=\sqrt{\frac{A\log(|\Pi|T/\delta)}{t}}$, \pref{eq: general condition 4} is also satisfied.

\subsection{FALCON for Contextual Bandits} 

\begin{exampl}
    \caption{FALCON for realizable contextual bandits} 
    \label{ex: falcon}
    \textbf{input}: $\Phi$ (reward function class), $\calA$ (action sets), $T, \delta$. \\
    \For{$t=1, \ldots, T$}{
        Let $\gamma_t = \sqrt{At/\log(|\Phi|T/\delta)}$. \\
        Compute $\hatphi_t = \argmin_{\phi\in\Phi} \sum_{\tau=1}^{t-1}(\phi(x_\tau, a_\tau) - r_t(a_t))^2$ \\
        Observe context $x_t$. \\
        Let $\widehat{a}_t = \argmax_{a\in\calA} \hatphi(x_t,a)$. Define 
        \begin{align*}
            p_t(a)\triangleq 
            \begin{cases}
                \frac{1}{A+\gamma_t \left(\hatphi_t(x_t,\widehat{a}_t) - \hatphi_t(x_t, a)\right)}, \qquad &\text{for\ } a\neq \widehat{a}_t, \\
                1- \sum_{a'\neq \widehat{a}_t}  p_t(a'), &\text{for\ } a=\widehat{a}_t.
            \end{cases}
        \end{align*}
        Sample $a_t\sim p_t$ and observe reward $R_t$. 
    }
\end{exampl}

FALCON is an algorithm for stationary contextual bandits. It relies on the assumption that the expected reward of action $a$ under context $x$ is given by an unknown function $\phi^\star(x,a): \calX\times \calA\rightarrow [0,1]$. The learner is given the function class $\Phi$ that contains $\phi^\star$. For each $\phi\in\Phi$, one can derive a policy $\pi_\phi: \calX\rightarrow \calA$ such that $\pi_\phi(x)=\argmax_{a\in\calA} \phi(x,a)$. It is straightforward to see that the optimal policy is $\pi_{\phi^\star}$, and the learner's goal is to be competitive with it. This problem falls into our framework with $\Pi = \{\pi_\phi: \phi\in\Phi\}$ and $f_t(\pi) = \E_{x\in\calD_t}\left[\phi^*(x, \pi(x))\right]$ where $\calD_t$ is the distribution of context at time $t$. The algorithm FALCON is shown in \pref{ex: falcon}.

Below, we show that it also satisfies Assumption~1'. 

At time $t$, the context $x_t$ is sampled from $\calD_t$, and the reward is generated by $\E[r_t(x_t,a_t)]=\phi^\star_t(x_t, a_t)$. We slightly modify their algorithm so that at every round $t$, the algorithm call the regression oracle once and obtain $\hatphi_t = \argmin_{\phi\in\Phi} \sum_{\tau=1}^{t-1}(\phi(x_\tau, a_\tau)-R_\tau)^2$ (the original algorithm does this only when the time index doubles), and then construct a mapping from context to action distribution $p_t(\cdot|\cdot)$ as specified in their algorithm. 
Analogous to their definitions, we define 
\begin{align*}
     \calR_{[1,t-1]}(\pi) &= \frac{1}{t-1} \sum_{\tau=1}^{t-1} \E_{x\sim \calD_\tau}\left[ \phi^\star_\tau(x, \pi(x))\right], \\
     \hatcalR_{[1,t-1]}(\pi) &=  \frac{1}{t-1}\sum_{\tau=1}^{t-1}  \E_{x\sim \calD_\tau}\left[\hatphi_{t}(x, \pi(x))\right], \\
     \hatReg_{[1,t-1]}(\pi) &= \hatcalR_{[1,t-1]}(\pi_{\hatphi_t}) - \hatcalR_{[1,t-1]}(\pi), \\
     \Reg_{[1,t-1]}(\pi)  &= \max_{\phi\in\Phi} \calR_{[1,t-1]}(\pi_\phi)  - \calR_{[1,t-1]}(\pi),   \\
     V_t(p, \pi) &= \E_{x\sim \calD_t}\left[\frac{1}{p(\pi(x)|x)}\right], \\
     \calV_t(\pi) &= \max_{\tau\in[1,t]} V_\tau(p_\tau, \pi).
\end{align*}
We will show that FALCON satisfies Assumption~1' with the following definitions: 
\begin{align*}
    \dev(t)&=\sqrt{A}\max_{x,a} |\phi_t^\star(x,a) - \phi_{t+1}^\star(x,a)| + \int_x |\calD_t(x) - \calD_{t+1}(x)|\mathrm{d}x, \\ 
    \avgreg(t)&=\sqrt{\frac{A\log(|\Phi|T/\delta)}{t}}
\end{align*}

By the same calculation as in Lemma 7 of \citep{simchi2020bypassing}, for any $\pi$, 
\begin{align*}
     &(t-1)\left|\hatcalR_{[1,t-1]}(\pi) - \calR_{[1,t-1]}(\pi)\right|^2 \\
     &= \frac{1}{t-1}\left(\sum_{\tau=1}^{t-1} \E_{x\sim \calD_\tau}\left[\hatphi_t(x, \pi(x)) - \phi^\star_\tau(x,\pi(x))\right]\right)^2 \\
     &\leq \sum_{\tau=1}^{t-1} \left(\E_{x\sim \calD_\tau}\left[\hatphi_t(x, \pi(x)) - \phi^\star_\tau(x,\pi(x))\right]\right)^2 \\
     &\leq \sum_{\tau=1}^{t-1} \left(\E_{x\sim \calD_\tau}\left[\sqrt{\frac{1}{p_\tau(\pi(x)|x)}p_\tau(\pi(x)|x)\left(\hatphi_t(x, \pi(x)) - \phi^\star_\tau(x,\pi(x))\right)^2}\right]\right)^2\\
     &\leq \sum_{\tau=1}^{t-1} \left(\E_{x\sim \calD_\tau}\left[\sqrt{\frac{1}{p_\tau(\pi(x)|x)}\E_{a\sim p_\tau(\cdot|x)}\left(\hatphi_t(x, a) - \phi^\star_\tau(x,a)\right)^2}\right]\right)^2 \\
     &\leq \sum_{\tau=1}^{t-1} \E_{x\sim \calD_\tau}\left[\frac{1}{p_\tau(\pi(x)|x)}\right] \E_{x\sim \calD_\tau}\E_{a\sim p_\tau(\cdot|x)}\left[\left(\hatphi_t(x, a) - \phi^\star_\tau(x,a)\right)^2\right] \\
     &\leq \sum_{\tau=1}^{t-1} V_\tau(p_\tau, \pi) \E_{a\sim p_\tau(\cdot|x)}\left[\left(\hatphi_t(x, a) - \phi^\star_\tau(x,a)\right)^2\right] \\
     &\leq \calV_{t-1}(\pi) \sum_{\tau=1}^{t-1} \E_{a\sim p_\tau(\cdot|x)}\left[\left(\hatphi_t(x, a) - \phi^\star_\tau(x,a)\right)^2\right].
\end{align*}
Using \pref{lemma: falcon lemma 1} and \pref{lemma: FALCON bound max V} below, when $\hatdev_{[1,t]}\leq \order\left(\avgreg(t)\right)$, we have 
\begin{align*}
    \sum_{\tau=1}^{t-1} \E_{x\sim\calD_\tau, a\sim p_\tau(\cdot|x)}\left[\left(\hatphi_t(x,a) - \phi^\star_\tau(x,a)\right)^2\right]\leq \order\left( \log(T|\Phi|/\delta) \right)
\end{align*}
and $\calV_{t-1}(\pi) \leq \order(A) + \max_{\tau\in[1,t-2]}\gamma_{\tau}\hatReg_{[1,\tau]}(\pi) $, where $\gamma_t=\Theta\left(\sqrt{\frac{At}{\log(|\Phi|T/\delta)}}\right)$.  Note that they are actually of the same order as in the Lemma 7 of \citep{simchi2020bypassing} since the additional terms contributed by $\dev_{[1,t]}$ are dominated by other terms. Thus, the bound we get for $\left|\calR_{[1,t-1]} - \hatcalR_{[1,t-1]}\right|$ is of the same order as their Lemma 7, which is
\begin{align}
     \left|\calR_{[1,t-1]}(\pi) - \hatcalR_{[1,t-1]}(\pi)\right| 
     &\leq \order\left(\sqrt{\frac{\log(T|\Phi|/\delta)}{t}\left(A+\max_{\tau\in[1,t-2]}\gamma_{\tau} \hatReg_{[1,\tau]}(\pi)\right)}\right) \\
     &\leq \frac{1}{16}\max_{\tau\in[1,t-2]}\hatReg_{[1,\tau]}(\pi) + \order\left(\avgreg(t)\right)   \tag{by AM-GM}.  \\
     \label{eq: R deviation}
\end{align}
 Then one can follow the derivation in their Lemma 8 using \pref{eq: R deviation}, and get 
 \begin{align*}
     \Reg_{[1,t]}(\pi) - \hatReg_{[1,t]}(\pi) 
     &\leq \frac{1}{8}\max_{\tau\in[1,t-1]}\hatReg_{[1,\tau]}(\pi) + \order(\avgreg(t)), \\
     \hatReg_{[1,t]}(\pi) - \Reg_{[1,t]}(\pi)
     &\leq \frac{1}{8}\max_{\tau\in[1,t-1]}\hatReg_{[1,\tau]}(\pi) + \order(\avgreg(t)).
 \end{align*}
Using these two inequalities, together with $\left| \Reg_{[1,\tau]}(\pi) - \Reg_{[1,t]}(\pi) \right| \leq \order(\dev_{[1,t]}) = \order\left( \avgreg(t) \right)$, 
we can also prove 
\begin{align}
     \Reg_{[1,t]}(\pi)\leq 2\hatReg_{[1,t]}(\pi) + \order(\avgreg(t)), \qquad  \hatReg_{[1,t]}(\pi)\leq 2\Reg_{[1,t]}(\pi) + \order(\avgreg(t))  \label{eq: FALCON closeness}
\end{align}
by induction as their Lemma 8. One can see that all bounds we obtain are of the same order as in the stationary case shown in \citep{simchi2020bypassing}, thanks to the condition $\dev_{[1,t]}=\order(\avgreg(t))$. 

Then following their Lemmas 9 and 10, we obtain regret bound $\max_\phi t\calR_{[1,t]}(\pi_\phi) - \sum_{\tau=1}^t  R_\tau = \order\left(\sqrt{At\log(T|\Phi|/\delta)}\right)$. 

Similar to the calculation in \pref{eq: ILOVE first condition}, by picking $\tildef_t =  \hatcalR_{[1,t-1]}(\pi_{\hatphi_t}) + c\sqrt{\frac{A\log(T|\Phi|/\delta)}{t}}$ with large enough $c$, we have 
\begin{align*}
    \tildef_t &\geq \max_\phi \calR_{[1,t-1]}(\pi_{\phi}) + c\sqrt{\frac{A\log(T|\Phi|/\delta)}{t}} -  \order(\avgreg(t))   \tag{by \pref{eq: R deviation}}\\
    &\geq \calR_{[1,t-1]}(\pi_{\phi_1^\star}) 
    \geq \calR_1(\pi_{\phi_1^\star}) - \order(\dev_{[1,t]}) \geq \min_{\tau\in[1,t]}\max_{\phi} \calR_\tau(\pi_\phi) -  \order(\dev_{[1,t]}), 
\end{align*}
which verifies \pref{eq: general condition 3}. To upper bound $\sum_{\tau=1}^t \left(\tildef_\tau - R_\tau\right)$, we follow a similar calculation as \pref{eq: ILOVE condition 2}, and use the condition $\dev_{[1,t]}=\order(\avgreg(t))$. This verifies \pref{eq: general condition 4}. 

\begin{lemma}
     \label{lemma: falcon lemma 1}
     If $\dev_{[1,t]}\leq \order\left(\avgreg(t)\right)$, then 
     \begin{align*}
          \sum_{\tau=1}^{t-1} \E_{x\sim\calD_\tau, a\sim p_\tau(\cdot|x)}\left[\left(\hatphi_t(x,a) - \phi^\star_\tau(x,a)\right)^2\right] = \order\left(\log(T|\Phi|/\delta)\right). 
     \end{align*}
\end{lemma}
\begin{proof}
    First, we consider a speicific $\phi$. Define $Y_\tau = \left( \phi(x_\tau, a_\tau) - R_\tau \right)^2 - \left(\phi^\star_{\tau}(x_\tau, a_\tau) - R_\tau\right)^2$. Then we have $\E[Y_\tau]=\E\left[(\phi(x_\tau, a_\tau) - \phi_\tau^\star(x_\tau, a_\tau))^2\right]$ and $\E\left[ Y_\tau^2 \right]\leq 4\E\left[(\phi(x_\tau, a_\tau)-\phi_\tau^\star(x_\tau, a_\tau))^2\right]=4\E[Y_\tau]$. By Freedman's inequality, 
    \begin{align*}
          \sum_{\tau=1}^{t-1} Y_\tau 
          &\geq \sum_{\tau=1}^{t-1} \E[Y_\tau] - \const_1\sqrt{\sum_{\tau=1}^{t-1}\E\left[Y_\tau^2\right] \log(T/\delta) } - \const_2 \log(T|\Phi|/\delta) \\
          &\geq \sum_{\tau=1}^{t-1} \E[Y_\tau] - 2\const_1\sqrt{\sum_{\tau=1}^{t-1}\E\left[Y_\tau\right] \log(T/\delta) } - \const_2 \log(T|\Phi|/\delta). 
    \end{align*}
    The above implies (by solving for $\sum_{\tau=1}^{t-1} \E[Y_\tau]$)
    \begin{align}
         \sum_{\tau=1}^{t-1} \E[Y_\tau] \leq 2\sum_{\tau=1}^{t-1}Y_\tau + 4(\const_1^2+\const_2)\log(T|\Phi|/\delta). \label{eq: CB freedman 1}
    \end{align}
    For the other direction, we also have 
    \begin{align}
         \sum_{\tau=1}^{t-1} Y_\tau 
          &\leq 2\sum_{\tau=1}^{t-1} \E[Y_\tau] + \left(\frac{\const_1^2}{4} + \const_2\right) \log(T|\Phi|/\delta).   \label{eq: CB freedman 2}
    \end{align}
    Then we can bound 
    \begin{align*}
         &\sum_{\tau=1}^{t-1} \E_{x\sim\calD_\tau, a\sim p_\tau(\cdot|x)}\left[\left(\hatphi_t(x,a) - \phi^\star_\tau(x,a)\right)^2\right]    \tag{using \pref{eq: CB freedman 1}}\\
         &\leq 2\sum_{\tau=1}^{t-1} \left(\hatphi_{t}(x_\tau, a_\tau) - R_\tau\right)^2 -  2\sum_{\tau=1}^{t-1} \left(\phi_{\tau}^\star(x_\tau, a_\tau) - R_\tau\right)^2 + 4(c_1^2+c_2)\log(T|\Phi|/\delta) \\
         &\leq 2\sum_{\tau=1}^{t-1} \left(\phi_1^\star(x_\tau, a_\tau) - R_\tau\right)^2 -  2\sum_{\tau=1}^{t-1} \left(\phi_{\tau}^\star(x_\tau, a_\tau) - R_\tau\right)^2 + 4(c_1^2+c_2)\log(T|\Phi|/\delta)   \tag{by the optimality of $\hatphi_t$} \\
         &\leq 4\sum_{\tau=1}^{t-1} \E\left[\left(\phi_1^\star(x_\tau, a_\tau) - \phi_\tau^\star(x_\tau, a_\tau)\right)^2\right] + \const_3 \log(T|\Phi|/\delta)   \tag{using \pref{eq: CB freedman 2}}\\
         &\leq \frac{4(t-1)}{A}\dev_{[1,t]}^2 + \const_3\log(T|\Phi|/\delta).    \tag{by the definition of $\dev_{[\cdot,\cdot]}$}
    \end{align*}
    By the condition on $\dev_{[1,t]}$, we have $\frac{4(t-1)}{A}\dev_{[1,t]}^2=\order\left(\log(T|\Phi|/\delta)\right)$, which proves the lemma. 
\end{proof}

\begin{lemma} \label{lemma: FALCON bound max V}
     If $\dev_{[1,t]}\leq \order\left(\avgreg(t)\right)$, then  
     \begin{align*}
           \calV_{t}(\pi) \leq \order(A) + \max_{\tau\in[1,t-1]}\gamma_{\tau}\hatReg_{[1,\tau]}(\pi) 
     \end{align*}
     where $\gamma_t=\Theta\left(\sqrt{\frac{At}{\log(|\Phi|T/\delta)}}\right)$. 
\end{lemma}

\begin{proof}
     Similar to Lemma 6 of FALCON, for $\tau\in[1,t]$, 
     \begin{align*}
          V_\tau(p_\tau, \pi) 
          &\leq A + \gamma_{\tau-1} \E_{x\sim \calD_\tau}\left[ \hatphi_\tau(x, \pi_{\hatphi_\tau}(x)) - \hatphi_\tau(x,\pi(x)) \right] \\
          &\leq A + \frac{\gamma_{\tau-1}}{\tau-1}\sum_{s=1}^{\tau-1}\E_{x\sim \calD_s}\left[ \hatphi_\tau(x, \pi_{\hatphi_\tau}(x)) - \hatphi_\tau(x,\pi(x)) \right] + \gamma_{\tau-1} \dev_{[1,t]} \\
          &\leq A + \gamma_{\tau-1}\hatReg_{[1,\tau-1]}(\pi) +  \gamma_{\tau-1} \dev_{[1,t]}. 
     \end{align*}
     By the condition $\hatdev_{[1,t]}\leq \order\left(\avgreg(t)\right)$, that last term $\gamma_{\tau-1} \dev_{[1,t]}$ is of order $\order\left( A \right)$. By the definition of $\calV_{t}(\pi)$, this finishes the proof. 
\end{proof}

\end{document}